\documentclass[11pt,a4paper]{article}
\usepackage{amsthm}
\usepackage[round,authoryear,sort]{natbib}
\bibpunct{(}{)}{;}{a}{,}{,} 

\usepackage[bookmarks,colorlinks]{hyperref}
\usepackage{xcolor}
\definecolor{darkblue}{rgb}{0,0.22,0.66}
\definecolor{darkcyan}{RGB}{0,139,139}
\definecolor{darkgray}{HTML}{666666}
\hypersetup{colorlinks=true,
citecolor=darkcyan,
linkcolor=darkblue,
urlcolor=darkgray
}
\usepackage[a4paper]{geometry}		
\usepackage{amsmath,amsfonts,amssymb}
\usepackage[scr=rsfs]{mathalpha} 
\usepackage{mathtools}
\usepackage{dsfont} 
\usepackage{bm}

\usepackage{xurl,xcolor,enumerate}
\usepackage{tabularx,multirow,array,booktabs}
\usepackage{colortbl}
\usepackage{graphicx}
\graphicspath{{figures/}}%
\usepackage{float}
\usepackage{subcaption} 

\usepackage{caption}
\usepackage{amsthm}
\theoremstyle{plain}
\newtheorem{theorem}{Theorem}[section]%
\newtheorem{corollary}[theorem]{Corollary} %
\newtheorem{lemma}[theorem]{Lemma}
\newtheorem*{lemma*}{Lemma}

\newtheorem{proposition}[theorem]{Proposition}
\theoremstyle{definition}

\theoremstyle{remark}

\newtheorem{remark}[theorem]{Remark}
\newtheorem*{remark*}{Remark}
\usepackage{etoolbox,xparse,dsfont,bm,amssymb}
\newcommand{\myMFabc}[4]{\expandafter#1\csname#3#4\endcsname{{#2{#4}}}}
\newcommand{\myMFcmd}[4]{\expandafter#1\csname#3#4\endcsname{{#2{\csname#4\endcsname}}}}

\newcommand{\MFabc}[3][\newcommand]{
    \def\doOld##1##2{\forcsvlist{\myMFabc{#1}{##1}{##2}}{#3}}
    \providecommand{\do}{do}
    \RenewDocumentCommand \do { >{\SplitList{,}} m } { \doOld##1 }
    \docsvlist{#2}
}
\newcommand{\MFcmd}[3][\newcommand]{
    \def\doOld##1##2{\forcsvlist{\myMFcmd{#1}{##1}{##2}}{#3}}
    \providecommand{\do}{do}
    \RenewDocumentCommand \do { >{\SplitList{,}} m } { \doOld##1 }
    \docsvlist{#2}
}

\newcommand{\bmzero}{{\bm{0}}}

\let\one\bbone
\MFabc{ {\mathfrak,frak}, }{T,u,U,o,O,E,m}
\MFabc{ {\mathbb, }, }{A,N,Z,R,C,Q}

\newcommand{\hatbm}[1]{\widehat{\bm{#1}}}
\newcommand{\tildebm}[1]{\widetilde{\bm{#1}}}
\newcommand{\bmcal}[1]{\bm{\mathcal{#1}}}
\newcommand{\caltilde}[1]{\mathcal{\widetilde{#1}}}
\newcommand{\calhat}[1]{\mathcal{\widehat{#1}}} 
\newcommand{\scrtilde}[1]{\widetilde{\mathscr{#1}\mspace{1mu}\mspace{-1mu}}}
\newcommand{\scrhat}[1]{\mathscr{\widehat{#1}\mspace{1mu}\mspace{-1mu}}}
\newcommand{\bmcalhat}[1]{\bm{\mathcal{\widehat{#1}}}}
\newcommand{\bmcaltilde}[1]{\bm{\mathcal{\widetilde{#1}}}}
\MFabc{ {\mathcal,cal}, {\mathscr,scr}, {\mathbb,bb}, {\bmcal,bmcal}, {\bmcal,calbm}, {\caltilde,caltilde}, {\caltilde,tildecal}, {\calhat,calhat}, {\calhat,hatcal}, {\scrtilde,scrtilde}, {\scrtilde,tildescr}, {\scrhat,scrhat}, {\scrhat,hatscr}, {\bmcalhat,bmcalhat}, {\bmcalhat,bmhatcal}, 
{\bmcalhat,calbmhat}, {\bmcalhat,calhatbm}, {\bmcalhat,hatbmcal}, {\bmcalhat,hatcalbm}, {\bmcaltilde,bmcaltilde}, {\bmcaltilde,bmtildecal}, {\bmcaltilde,calbmtilde}, {\bmcaltilde,caltildebm}, {\bmcaltilde,tildebmcal}, {\bmcaltilde,tildecalbm}}{A,B,C,D,E,F,G,H,I,J,K,L,M,N,O,P,Q,R,S,T,U,V,W,X,Y,Z}

\MFabc{{\bm,bm}, {\mathsf,sf}, {\widehat,hat}, {\widetilde,tilde}, {\hatbm,hatbm}, {\hatbm,bmhat}, {\tildebm,tildebm}, {\tildebm,bmtilde}}{a,b,c,d,e,f,g,h,i,j,k,l,m,n,o,p,q,r,s,t,u,v,w,x,y,z,A,B,C,D,E,F,G,H,I,J,K,L,M,N,O,P,Q,R,S,T,U,V,W,X,Y,Z}

\let\eps\varepsilon
\MFcmd{{\bm,bm}, {\widehat,hat}, {\widetilde,tilde}, {\tildebm, tildebm}, {\tildebm, bmtilde}, {\hatbm, hatbm}, {\hatbm, bmhat}}{alpha,beta,gamma,delta,epsilon,zeta,eta,theta,iota,kappa,lambda,mu,nu,xi,omicron,pi,rho,sigma,tau,upsilon,phi,chi,psi,omega,Alpha,Beta,Gamma,Delta,Epsilon,Zeta,Eta,Theta,Iota,Kappa,Lambda,Mu,Nu,Xi,Omicron,Pi,Rho,Sigma,Tau,Upsilon,Phi,Chi,Psi,Omega,varrho,varphi,vartheta,varepsilon,varsigma,ell,eps}

\newcommand{\actdef}[1]{\expandafter\def\csname#1\endcsname{{\ensuremath{\mathtt{#1}}}}}
\forcsvlist{\actdef}{ReLU, LReLU, LeakyReLU, ELU, GELU, SiLU, Softplus, dGELU, dSiLU, dSoftplus, Tanh, Sigmoid, Arctan, Softsign, SRS, dSRS, Swish, dSwish, Mish, dMish, SELU, CELU, dSELU, Sin,SinLU, SinTU, PSinTU, sine, cosine, Sine, Cosine, EUAF}

\newlength{\myLength}

\newcommand{\mn}[7][]{\ensuremath{	{\hspace{0.6pt}\mathcal{M}\hspace{-0.0pt}\mathcal{N}\hspace{-0.8pt}}_{#2}
		#1\{#3,\hspace{2.2pt} #4,\hspace{2.2pt} #5;\hspace{3.0297pt} {#6}\hspace{-1.0298pt}\to\hspace{-0.98pt}{#7}#1\}
}}

\newcommand{\nn}[6][]{\ensuremath{	{\hspace{0.6pt}\mathcal{N}\hspace{-1.28pt}\mathcal{N}\hspace{-0.725pt}}_{#2}
		#1\{#3,\hspace{1.987pt} #4;\hspace{3.0297pt} {#5}\hspace{-1.0298pt}\to\hspace{-0.98pt}{#6}#1\}
}}

\renewcommand{\nn}[6][]{\ensuremath{	{\hspace{0.6pt}\mathcal{F}\hspace{-0.828pt}\mathcal{N}\hspace{-0.725pt}}_{#2}
		#1\{#3,\hspace{1.987pt} #4;\hspace{3.0297pt} {#5}\hspace{-1.0298pt}\to\hspace{-0.98pt}{#6}#1\}
}}

\let\ts\intercal
\def\ts{\tn{$\textsf{T}$}}
\def\ts{\textsf{\textup{T}}}

\def\cpwl{\textsf{\textup{CPwL}}}
\let\cpl\cpwl

\newcommand{\dprime}{{\prime\prime}}

\newcommand{\mycase}[2]{\par \vspace{0.25cm}\noindent\textbf{\hspace{8pt}Case }$#1\colon$ #2\par \vspace{0.18cm} \par}

\usepackage{tikz}

\newenvironment{keywords}{\par \noindent\textbf{Keywords}:}{\par}

\usepackage[left,mathlines]{lineno}
\usepackage{refcount}

\definecolor{mylinenumbercolor}{HTML}{BEBEBE}

\makeatletter
\newcommand*\patchAmsMathEnvironmentForLineno[1]{%
	\expandafter\let\csname old#1\expandafter\endcsname\csname #1\endcsname
	\expandafter\let\csname oldend#1\expandafter\endcsname\csname end#1\endcsname
	\renewenvironment{#1}%
	{\linenomath\csname old#1\endcsname}%
	{\csname oldend#1\endcsname\endlinenomath}}%
\newcommand*\patchBothAmsMathEnvironmentsForLineno[1]{%
	\patchAmsMathEnvironmentForLineno{#1}%
	\patchAmsMathEnvironmentForLineno{#1*}}%
\patchBothAmsMathEnvironmentsForLineno{equation}%
\patchBothAmsMathEnvironmentsForLineno{align}%
\patchBothAmsMathEnvironmentsForLineno{flalign}%
\patchBothAmsMathEnvironmentsForLineno{alignat}%
\patchBothAmsMathEnvironmentsForLineno{gather}%
\patchBothAmsMathEnvironmentsForLineno{multline}%
\makeatother

\let\tilde\widetilde

\let\epsilon\varepsilon
\let\eps\varepsilon

\let\tn\textnormal

\let\cdots\customcdots
\let\myforall\forall
\def\forall{{\myforall\, }}
\let\myexists\exists
\def\exists{{\myexists\, }}

\long\def\blue#1{{\color{blue}#1}}

\newenvironment{colorenv}[1][blue]
{%
  \begingroup
  \color{#1}%
  \ignorespaces
}
{%
  \endgroup
  \ignorespacesafterend
}

\definecolor{mygray}{RGB}{230,230,230}
\definecolor{myorange}{HTML}{ff7f0e}

\let\cite\citep
\usepackage{doi}
\def\doitext{DOI: }
\makeatletter
\long\def\@makefntext#1{\@setpar{\@@par\@tempdima \hsize 
		\advance\@tempdima-15pt\parshape \@ne 15pt \@tempdima}\par
	\parindent 2em\noindent \hbox to \z@{\hss{\textsuperscript{\@thefnmark}} \hfil}#1}
\newlength\aftertitskip     \newlength\beforetitskip
\newlength\interauthorskip  \newlength\aftermaketitskip
\def\maketitle{\par
	\begingroup
	\def\thefootnote{\color{black}\fnsymbol{footnote}}
	\def\@makefnmark{\hbox to 0pt{$^{\@thefnmark}$\hss}}
	\@maketitle \@thanks
	\endgroup
	\setcounter{footnote}{0}
	\let\maketitle\relax \let\@maketitle\relax
	\gdef\@thanks{}\gdef\@author{}\gdef\@title{}\let\thanks\relax}

\def\@startauthor{\noindent \normalsize\bf}
\def\@endauthor{}
\def\@starteditor{\noindent \small {\bf Editor:~}}
\def\@endeditor{\normalsize}
\def\@maketitle{\vbox{\hsize\textwidth
		\linewidth\hsize \vskip \beforetitskip
		{\begin{center} \Large\bf \@title \par \end{center}} \vskip \aftertitskip
		{\def\and{\unskip\enspace{\rm and}\enspace}%
			\def\addr{\small\it}%
            \def\email{\hfill\small\ttfamily}%
			\def\name{\normalsize\bf}%
			\def\AND{\@endauthor\rm\hss \vskip \interauthorskip \@startauthor}
			\@startauthor \@author \@endauthor}
		\vskip \aftermaketitskip
}}
\makeatother
\colorlet{blue}{black}

\let\ldots\cdots
\let\dots\cdots

\title{Fourier Multi-Component and Multi-Layer Neural Networks: Unlocking High-Frequency Potential}

\author{\name   Shijun Zhang\thanks{\hspace{-2pt}Corresponding author.}
\email  \href{mailto:shijun.zhang@polyu.edu.hk}{shijun.zhang@polyu.edu.hk}\\
\addr  Department of Applied Mathematics\\
Hong Kong Polytechnic University
\AND
\name  Hongkai Zhao
\email 
\href{mailto:zhao@math.duke.edu}{zhao@math.duke.edu}\\
\addr  Department of Mathematics\\
Duke University
\AND  \name  Yimin Zhong
\email \href{mailto:yimin.zhong@auburn.edu}{yimin.zhong@auburn.edu}
\\ 
\addr  Department of Mathematics and Statistics\\ Auburn University
\AND  \name  Haomin Zhou
\email  	\href{mailto:hmzhou@math.gatech.edu}{hmzhou@math.gatech.edu} 
\\
\addr School of Mathematics\\
 Georgia Institute of Technology
    }
\begin{document}
\maketitle

\begin{colorenv}
    \begin{abstract}
The architecture of a neural network and the choice of its activation function are both fundamental to its performance. Equally important is ensuring that these two elements are well matched, as their alignment is key to effective representation and learning. In this paper, we introduce the Fourier Multi-Component and Multi-Layer Neural Network (FMMNN), a model that combines \sine-type activations with the multi-component and multi-layer structure of MMNNs. In an FMMNN, each component is represented as a trainable linear combination of fixed random \sine-type basis functions, while multi-layer composition generates more complex and adaptive high-frequency features. We establish that FMMNNs retain exponential expressive power for function approximation even under a low-rank architectural structure. We also analyze the optimization landscape of FMMNNs and find it to be substantially more favorable than that of standard fully connected neural networks, especially for high-frequency targets. In addition, we propose a scaled random initialization method for the first-layer weights in FMMNNs, which accelerates training and improves final performance when sufficient samples are available. Extensive numerical experiments support our theoretical insights, showing that FMMNNs achieve strong accuracy and favorable convergence behavior on oscillatory function-approximation benchmarks.
\end{abstract}
\end{colorenv}

\begin{keywords}
low-rank structures,  
  high-frequency approximation, 
  scaled initialization,
Fourier analysis,
random-basis approximation
\end{keywords}


\section{Introduction}
\label{sec:intro}

\begin{colorenv}[blue]
The effectiveness of a neural network is jointly determined by its architecture
and activation function, especially for high-frequency approximation. Standard
activations such as \ReLU, \GELU, \texttt{sigmoid}, and \texttt{tanh} are stable
and widely used, but standard fully connected neural networks, or multi-layer
perceptrons, often exhibit a low-frequency bias during training. This has
motivated frequency-aware approaches such as periodic activations,
Fourier-feature or positional-encoding methods, space--frequency localized
representations, variable-periodic activations, and multi-scale decompositions;
see Section~\ref{sec:related:work} for further discussion.

This paper follows this frequency-aware viewpoint but focuses on the interaction
between sine-type activations and the multi-component/multi-layer structure of
MMNNs. A sine activation alone does not guarantee efficient learning in a
standard FCNN. In contrast, MMNNs~\cite{ZZZZ-24-MMNN} decompose complex
functions into trainable components built from fixed random bases and then
recombine them through layers. This structure reduces the number of trainable
parameters and often leads to a better-conditioned learning problem. In an
FMMNN, each component is a trainable linear combination of fixed random
sine-type basis functions, and can therefore be viewed as a Fourier-type
random-basis expansion. Multi-layer composition of such components can then
generate more complex and adaptive higher-frequency features. Thus, FMMNN is not merely a
\sine{}-activated network; it effectively combines activation, architecture, and
initialization in a structured way, and can accurately approximate well-sampled target
functions and their derivatives.

For nonsmooth targets, however, pure Fourier-type approximation may be less
effective and can suffer from Gibbs-type behavior. Motivated by this limitation
and by the expressive benefit of singular activations such as \ReLU, we introduce
the Sine Truncated Unit
\[
    \SinTU_s \coloneqq \sin\circ \calT_s,
    \qquad
    \calT_s(x)\coloneqq \max\{x,s\}.
\]
The parameter \(s\) controls the balance between oscillation and
singularity: as \(s\) decreases, \(\SinTU_s\) approaches the pure \sine{}
activation when $s$ is below most preactivation values; for larger \(s\), the truncation introduces a \ReLU-like singular
feature. This gives an additional degree of freedom for matching the activation
to the regularity of the target function. 
Making the truncation level \(s\) learnable is a natural direction for future work.
We further introduce a principled first-layer scaling strategy inspired by the
two-layer analysis in~\cite{ZZZZ-23}. Scaling the initial slope weights inside
the activation function introduces higher-frequency modes at initialization and
can accelerate training when sufficient data are available. For FMMNNs, applying
this initial scaling only to the first layer improves the capture of fine sample
features while avoiding the instability that may arise from scaling deeper layers;
see Section~\ref{sec:scaling:init}.
\end{colorenv}

We demonstrate that integrating the MMNN structure with \sine{} (or \SinTU{s}) creates a surprisingly effective synergy, particularly for efficiently capturing high-frequency components.
Our main contributions are summarized below. 

\begin{itemize}
\item First, we establish that using \texttt{sine} or \SinTU{s} as activation functions within the MMNNs framework offers significant mathematical potential in terms of approximation capability.
In particular, given a $1$-Lipschitz function $f:[0,1]^d\to \mathbb{R}$ and a \SinTU{} function $\varrho$, for any $p \in [1,\infty)$, there exists   $\phi$ realized by an $\varrho$-activated MMNN of width $2d(4N-1)$, rank $3d$, and depth $L+2$, such that
\[
\| \phi - f\|_{L^p([0,1]^d)} \leq  2\sqrt{d} \cdot N^{-L}.
\]
For the generalized version (applicable to generic continuous functions) and the \sine-related version, see Theorems~\ref{thm:main1} and \ref{thm:main2}.
It is worth noting that when we focus on the approximation results using \ReLU{} and sine as activation functions, our Theorem~\ref{thm:main2} provides a clear improvement over previous studies — namely, \cite{shijun:floor:relu}, which investigates \ReLU-Floor, and \cite{yarotsky:2019:06}, which considers \ReLU-sine with fixed width.
Specifically, our theorem achieves an approximation rate of $O(N^{-L})$, while both \cite{shijun:floor:relu} and \cite{yarotsky:2019:06} attain only $O(N^{-\sqrt{L}})$ (with \cite{yarotsky:2019:06} restricted to the regime of width $O(d)$). In addition, our construction uses only $O(NL)$ parameters when $d$ is treated as a fixed constant, significantly reducing redundancy compared with the $O(N^2L)$ required in these earlier works.


\item Next, we analyze the landscape of the cost function with respect to network parameters, which provides insight into the training complexity across different network architectures and activation functions. Notably, the MMNN structure results in a significantly more favorable optimization landscape compared to FCNNs, as illustrated in Figure~\ref{fig:landscape:intro:v} (see Section~\ref{sec:landscape} for further details).

\begin{figure}[ht]
            \centering
            \begin{subfigure}[b]{0.242046\textwidth}
                    \centering            
                    \includegraphics[width=0.9997\textwidth]{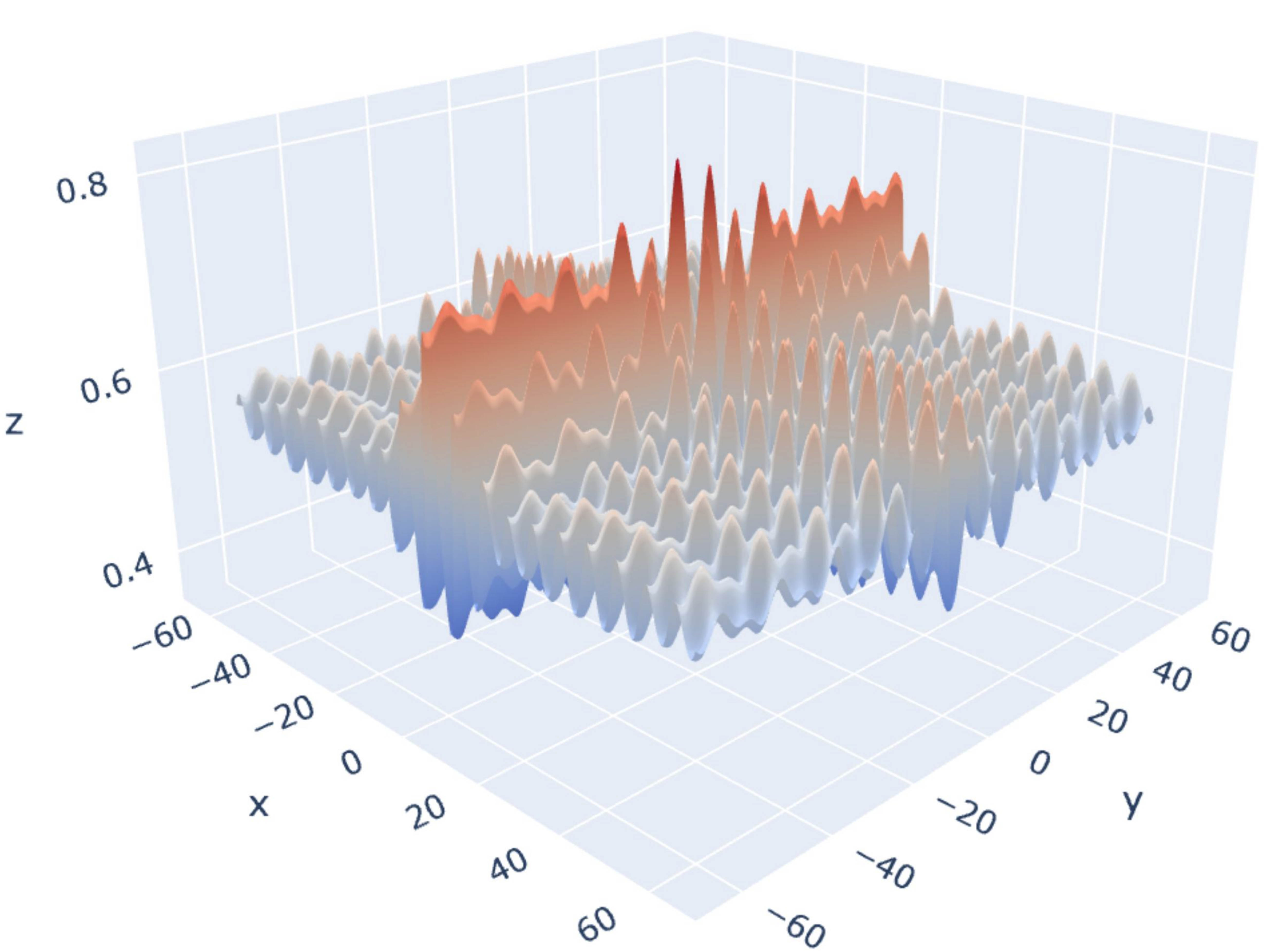}
                    \subcaption{\Sine{} FCNN.}
                \end{subfigure}
                \hfill
            \begin{subfigure}[b]{0.242046\textwidth}
                    \centering            \includegraphics[width=0.9997\textwidth]{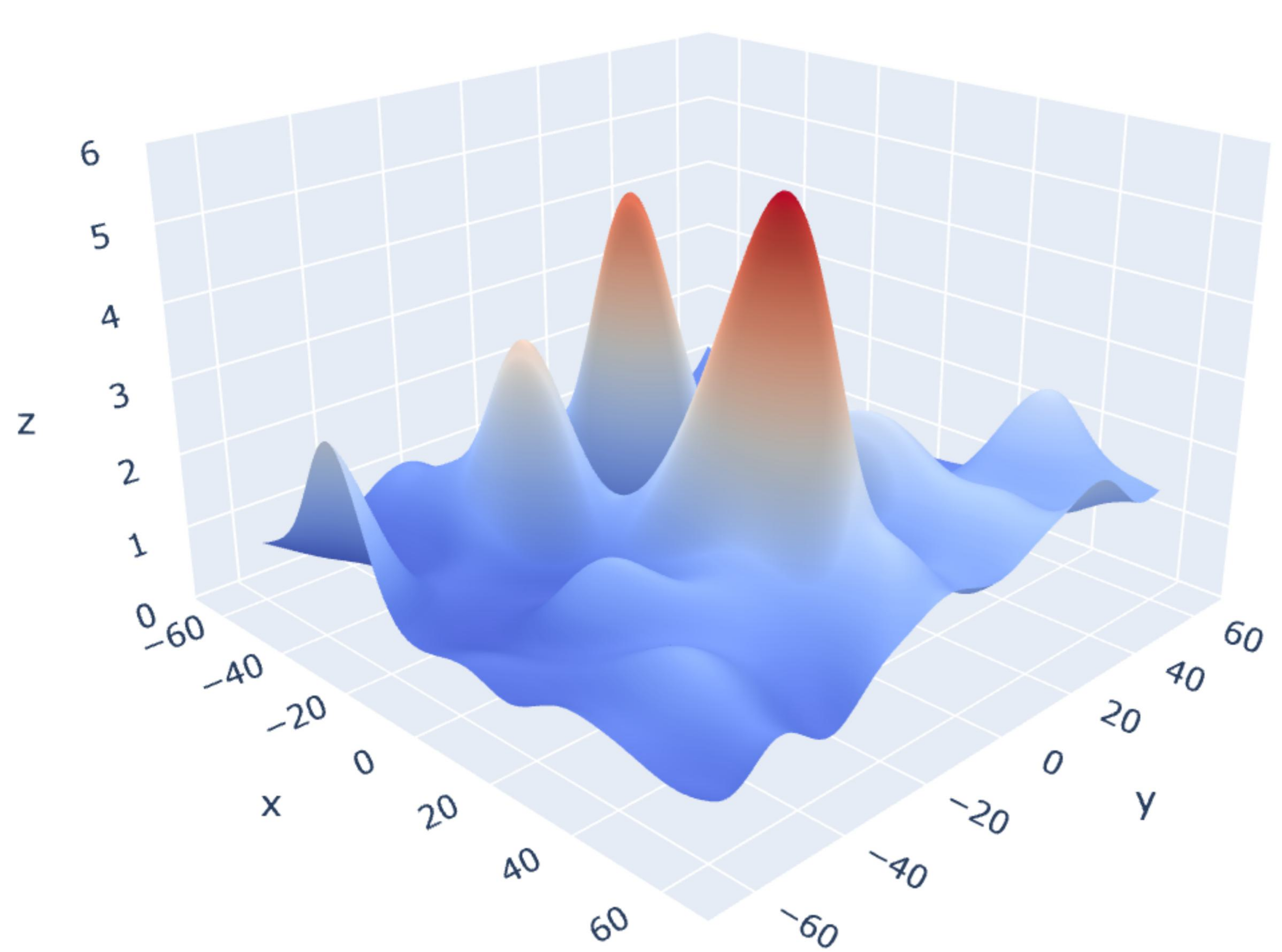}
                    \subcaption{\Sine{} MMNN.}
                \end{subfigure}
\hfill
                \begin{subfigure}[b]{0.24207300245\textwidth}
                    \centering            
                    \includegraphics[width=0.999\textwidth]{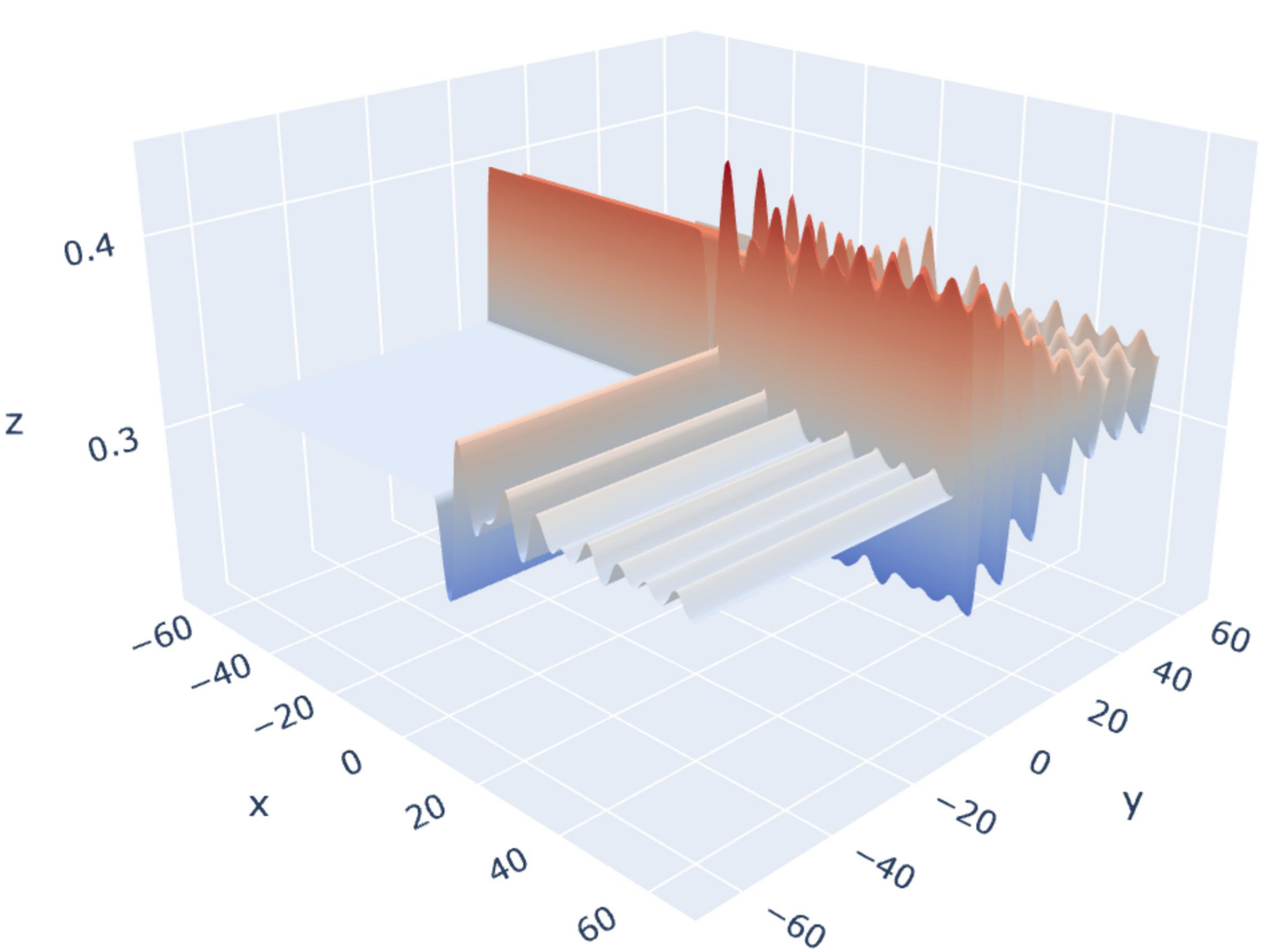}
                    \subcaption{$\SinTU_0$ FCNN.}
                \end{subfigure}
                \hfill
            \begin{subfigure}[b]{0.242\textwidth}
                    \centering            \includegraphics[width=0.999\textwidth]{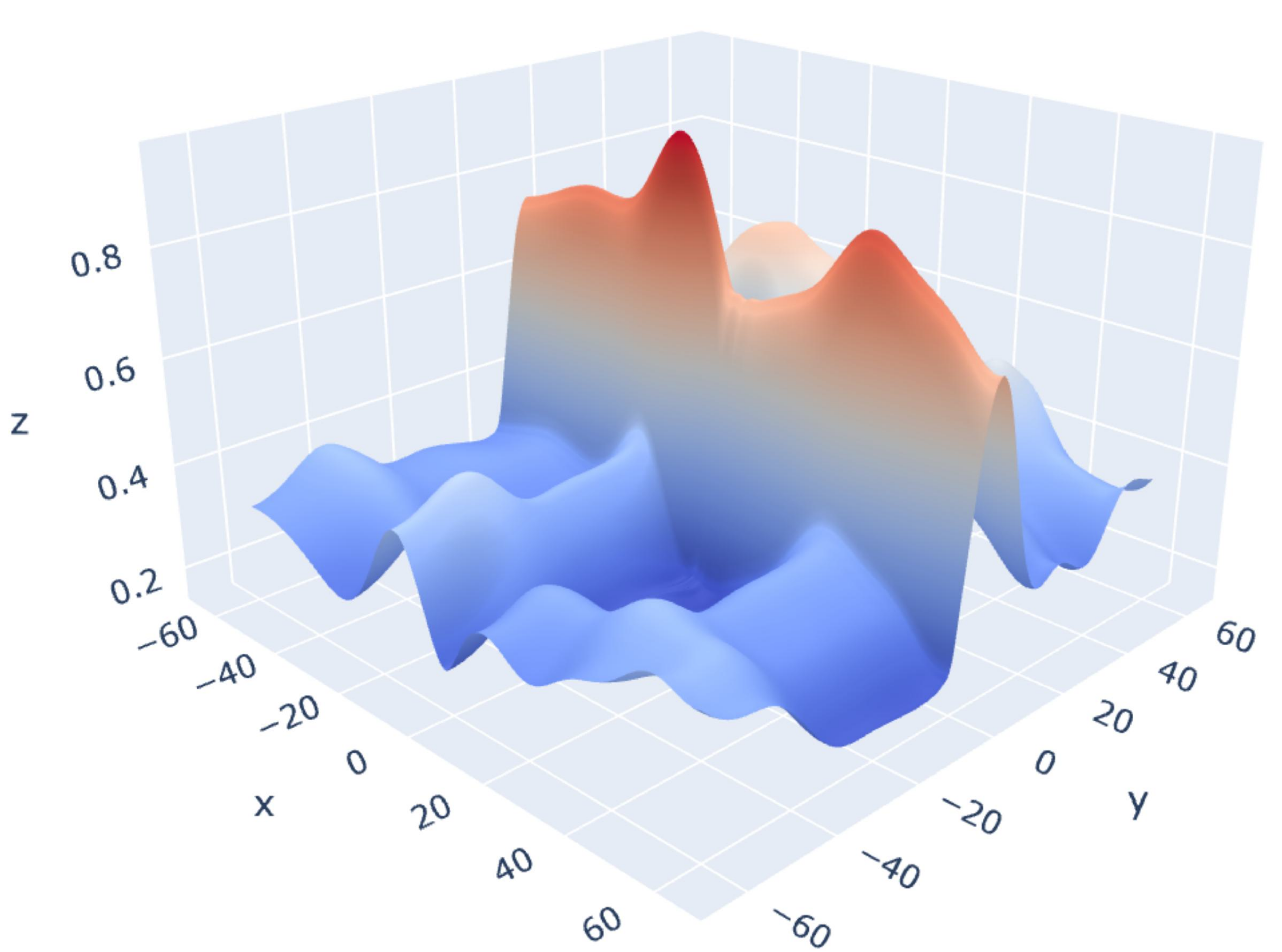}
                    \subcaption{$\SinTU_0$ MMNN.}
                \end{subfigure}             
                \caption{Comparison of the cost function landscapes in terms of two parameters.
                }
\label{fig:landscape:intro:v}
\end{figure}

\item 
Then, extensive numerical experiments demonstrate that Fourier MMNNs (FMMNNs), which use \texttt{sine} or \SinTU{s} as activation functions, consistently outperform other models in both accuracy and efficiency, as shown in Table~\ref{tab:error:comparison:MMNNs:vs:FCNNs:intro}. For \( f_1 \), \sine{}-activated MMNNs achieve the best result, aligning with expectations since \( f_1 \in 
 C^\infty(\mathbb{R}) \) (see Figure~\ref{fig:f1:f2:intro}). The accurate approximation of derivatives is particularly noteworthy, given the complexity of \( f_1 \) and the fact that
the training process relies
solely on function values, without incorporating derivative information.
Additionally, as we can see from Figure~\ref{fig:f1:errors:vs:epoch}, the \sine{} activation function not only accelerates the convergence of MMNNs but also enhances their overall performance. On the other hand, while $\SinTU_{-\pi}$ (with singularity) significantly reduces training error, it results in mediocre test performance. This outcome is expected, as the target function $f_1$ is in $C^\infty(\R)$.
For \( f_2 \in C^0(\mathbb{R}) \setminus  C^1(\mathbb{R}) \) (see Figure~\ref{fig:f1:f2:intro}), which contains numerous singularities, \( \SinTU_{-\pi} \) achieves the best accuracy, demonstrating its effectiveness in capturing these singular features.
Notably, even in this inherently challenging case, \sine{}-activated MMNNs still achieve results comparable to the best. 
For \( f_3 \in C^0(\mathbb{R}) \setminus  C^1(\mathbb{R}) \) (see Figure~\ref{fig:f1:f2:intro}), FCNNs are highly sensitive to training hyperparameters and often fail with small mini-batches. In contrast, FMMNNs remain stable and perform well across different settings. Even with large mini-batches, training FCNNs is time-consuming, yet they still underperform compared to \sine-activated MMNNs.
For more details on the experiments, including additional tests in two and three dimensions, refer to Section~\ref{sec:experiments}.

\item 
Finally, inspired by scaled parameter initialization in our previous work \cite{ZZZZ-23}, we propose a scaled random initialization to the weights of the first layer in FMMNNs, which can significantly accelerate the learning process and improve final performance, particularly when ample training samples are available (see Section~\ref{sec:scaling:init}).

\end{itemize}

\begin{figure}[ht]
            \centering
  \begin{subfigure}[b]{0.95\textwidth}
                    \centering            
                    \includegraphics[width=0.999\textwidth]{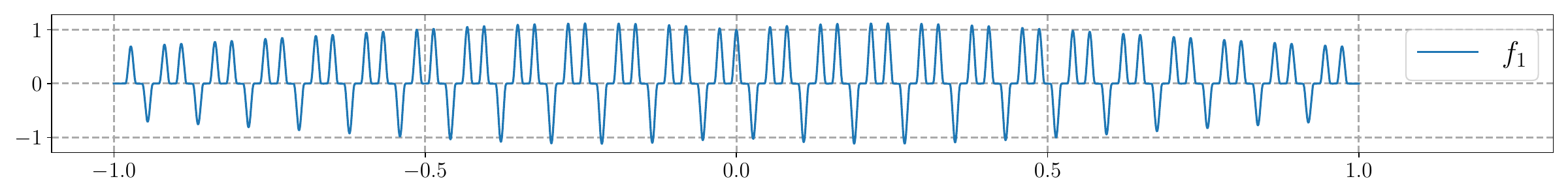}
                \end{subfigure}
            \begin{subfigure}[b]{0.95\textwidth}
                    \centering            
                    \includegraphics[width=0.999\textwidth]{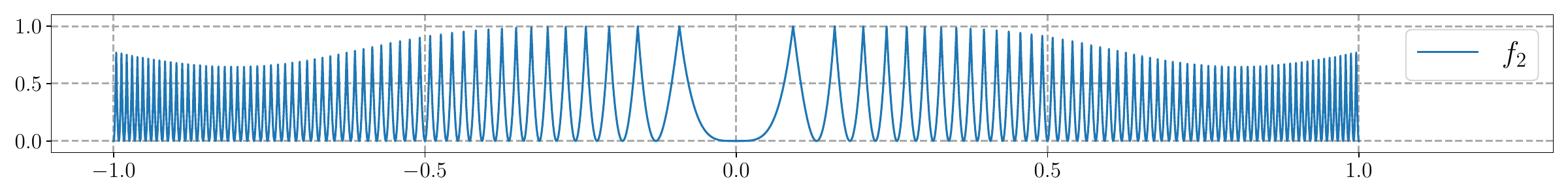}
                \end{subfigure}
            \begin{subfigure}[b]{0.95\textwidth}
                    \centering            
                    \includegraphics[width=0.999\textwidth]{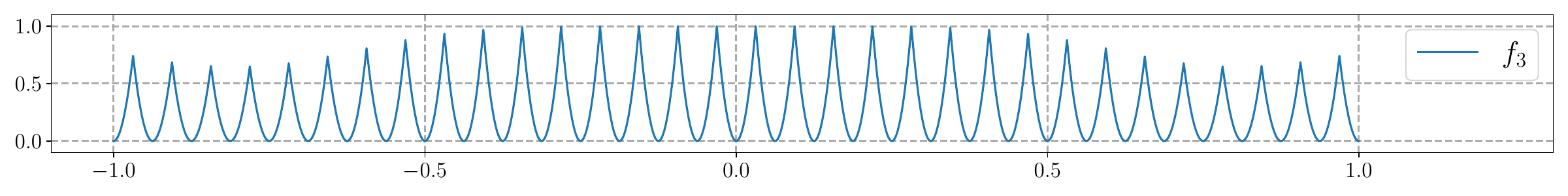}
                \end{subfigure}
\caption{Illustrations of $f_1\in C^\infty(\R)$ and $f_2,f_3\in C^0(\R)\setminus  
 C^1(\R)$, defined in \eqref{eq:def:f1:Cinfty:MMNN:vs:FCNN}, \eqref{eq:def:f2:C0:MMNN:vs:FCNN}, and \eqref{eq:def:f3:C0:MMNN:vs:FCNN}.}
    \label{fig:f1:f2:intro}   
\end{figure}

\begin{figure}[ht]
            \centering
  \begin{subfigure}[b]{0.312295\textwidth}
                    \centering            
                    \includegraphics[width=0.999\textwidth]{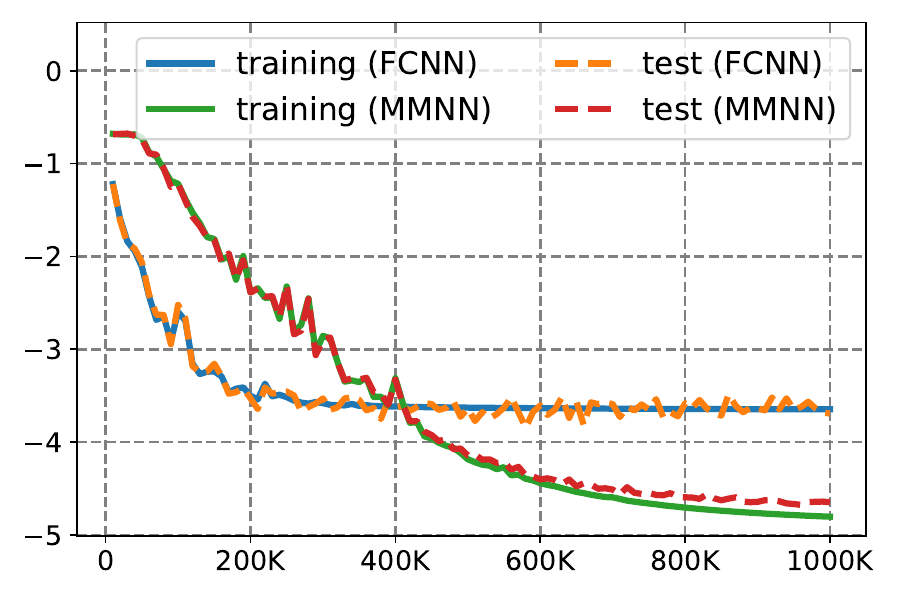}
                    \subcaption{\ReLU.}
                \end{subfigure}
            \begin{subfigure}[b]{0.312295\textwidth}
                    \centering            
                    \includegraphics[width=0.999\textwidth]{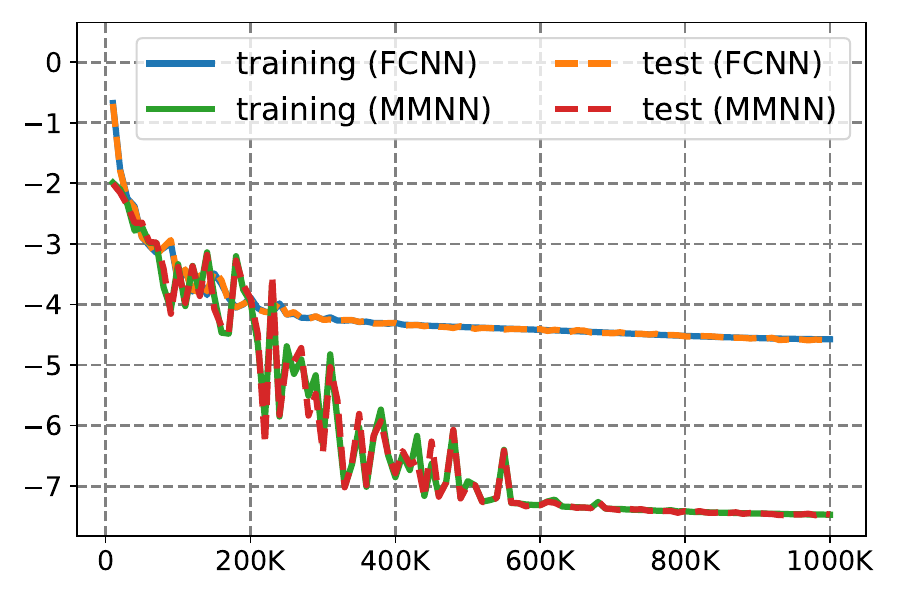}
                    \subcaption{\Sine.}
                \end{subfigure}
            \begin{subfigure}[b]{0.312295\textwidth}
                    \centering            
                    \includegraphics[width=0.999\textwidth]{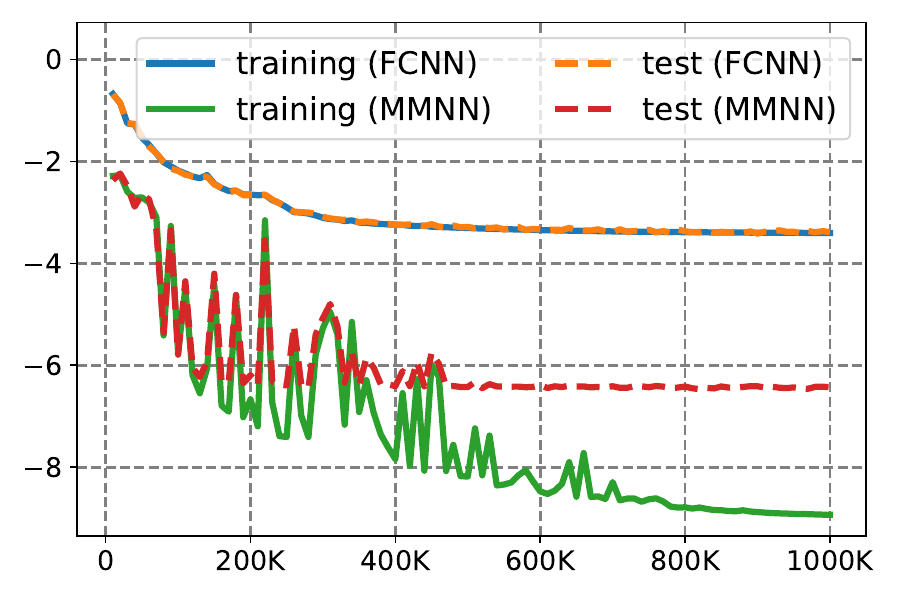}
                    \subcaption{$\SinTU_{-\pi}$.}
                \end{subfigure}
\caption{Illustrations of training and test errors (base-10 logarithm) versus epoch for $f_1$.}
    \label{fig:f1:errors:vs:epoch}   
\end{figure}

\begin{table}[ht]
	\centering  
 \setlength{\tabcolsep}{0.868em} 
 \renewcommand{\arraystretch}{1.15}
\caption{Comparison of test errors. Training is conducted in double precision.
}
\label{tab:error:comparison:MMNNs:vs:FCNNs:intro}
	\resizebox{0.8280\textwidth}{!}{ 
		\begin{tabular}{ccccccc@{\hspace{10pt}}c} 
			\toprule
            & 
            &\multicolumn{4}{c}{
            $\bm{6}$ hidden layers and $\bm{7.3\times 10^{4}}$ trainable parameters } & & \\
            \cmidrule(lr){3-6}
        & 
        &\multicolumn{2}{c}{MMNN} 
        &\multicolumn{2}{c}{FCNN} & 
        \multicolumn{2}{c}{\#training-samples}
        \\
            \cmidrule(lr){3-4}
            \cmidrule(lr){5-6}
            \cmidrule(lr){7-8}
      target & 
      {activation}
      &
      MSE 
    &  MaxE 
    &    MSE 
    &  MaxE   &
      mini-batch &  all
    \\
			\midrule
$f_1$ & $\mathtt{ReLU}$ &  $ 2.24 \times 10^{-5} $  &  $ 4.06 \times 10^{-2} $  &  $ 2.31 \times 10^{-4} $  &  $ 1.93 \times 10^{-1} $  &  3000  & 3000  
 \\ 

\rowcolor{mygray}$f_1$ & $\mathtt{tanh}$ &  $ 4.91 \times 10^{-6} $  &  $ 1.24 \times 10^{-2} $  &  $ 2.67 \times 10^{-3} $  &  $ 3.69 \times 10^{-1} $  &  3000  & 3000  
 \\ 

 $f_1$ & $\mathtt{sine}$ &  $ \bm{3.43 \times 10^{-8}} $  &  $ \bm{8.37 \times 10^{-4}} $  &  $ 2.62 \times 10^{-5} $  &  $ 2.35 \times 10^{-2} $  &  3000  & 3000  
 \\ 

\rowcolor{mygray}$f_1$ & $\mathtt{SinTU}_{-\pi}$ &  $ 3.65 \times 10^{-7} $  &  $ 4.90 \times 10^{-3} $  &  $ 4.14 \times 10^{-4} $  &  $ 1.44 \times 10^{-1} $  &  3000  & 3000  
 \\ 
 
 \midrule

$f^\prime_1$ & $\mathtt{tanh}$ &  $ 1.24 \times 10^{-4} $  &  $ 5.67 \times 10^{-2} $  &  $ 1.10 \times 10^{-2} $  &  $ 6.31 \times 10^{-1} $  &    &   
 \\ 

\rowcolor{mygray}$f^\prime_1$ & $\mathtt{sine}$ &  $ 1.27 \times 10^{-6} $  &  $ 5.12 \times 10^{-3} $  &  $ 5.83 \times 10^{-4} $  &  $ 1.06 \times 10^{-1} $  &    &   
 \\ 
\midrule
 $f^\dprime_1$ & $\mathtt{tanh}$ &  $ 9.92 \times 10^{-4} $  &  $ 1.59 \times 10^{-1} $  &  $ 3.02 \times 10^{-2} $  &  $ 8.62 \times 10^{-1} $  &    &   
 \\ 

\rowcolor{mygray}$f^\dprime_1$ & $\mathtt{sine}$ &  $ 7.82 \times 10^{-6} $  &  $ 1.37 \times 10^{-2} $  &  $ 4.45 \times 10^{-3} $  &  $ 2.71 \times 10^{-1} $  &    &   
 \\ 
 
 \midrule
 
$f_2$ & $\mathtt{ReLU}$ &  $ 1.53 \times 10^{-3} $  &  $ 4.37 \times 10^{-1} $  &  $ 2.13 \times 10^{-2} $  &  $ 6.35 \times 10^{-1} $  &  3000 & 60000  
 \\ 

\rowcolor{mygray}$f_2$ & $\mathtt{tanh}$ &  $ 1.42 \times 10^{-4} $  &  $ 1.51 \times 10^{-1} $  &  $ 1.21 \times 10^{-2} $  &  $ 5.56 \times 10^{-1} $  &  3000 & 60000  
 \\ 

 $f_2$ & $\mathtt{sine}$ &  $ 4.84 \times 10^{-6} $  &  $ 2.88 \times 10^{-2} $  &  $ 9.07 \times 10^{-5} $  &  $ 9.22 \times 10^{-2} $  &  3000 & 60000  
 \\ 

\rowcolor{mygray}$f_2$ & $\mathtt{SinTU}_{-\pi}$ &  $ \bm{1.28 \times 10^{-6}} $  &  $ \bm{2.31 \times 10^{-2}} $  &  $ 6.14 \times 10^{-3} $  &  $ 5.31 \times 10^{-1} $  &  3000 & 60000  
 \\

 \midrule  

 $f_3$ & $\mathtt{sine}$ &  $ 7.68 \times 10^{-8} $  &  $ 6.06 \times 10^{-3} $  &  $ 3.04 \times 10^{-2} $  &  $ 6.99 \times 10^{-1} $  & 500 & 18000  
 \\ 

\rowcolor{mygray} $f_3$ & $\mathtt{sine}$ &  $ 1.27 \times 10^{-7} $  &  $ 6.58 \times 10^{-3} $  &  $ 6.69 \times 10^{-2} $  &  $ 6.68 \times 10^{-1} $  & 1000 & 18000  
 \\ 

  $f_3$ & $\mathtt{sine}$ &  $ 8.75 \times 10^{-8} $  &  $ 6.28 \times 10^{-3} $  &  $ 2.43 \times 10^{-4} $  &  $ 1.52 \times 10^{-1} $  & 1500 & 18000  
 \\ 

\rowcolor{mygray} $f_3$ & $\mathtt{sine}$ &  $ \bm{6.41 \times 10^{-8}} $  &  $ \bm{5.33 \times 10^{-3}} $  &  $ 5.86 \times 10^{-6} $  &  $ 3.38 \times 10^{-2} $  & 2000 & 18000  
 \\ 
 \bottomrule
		\end{tabular} 
	}
\end{table}

The paper is structured as follows. 
In Section~\ref{sec:sine:potential}, 
we begin with a detailed analysis of the MMNN architecture and its approximation capabilities using \texttt{sine} and \SinTU{s} as  activation functions. Several theorems and corollaries are established, with rigorous proofs deferred to the appendix.
We also discuss practical considerations, including the cost function landscape and the benefits of fixing weights within activation functions, and conclude with a review of related work.
Section~\ref{sec:experiments} presents extensive numerical experiments that support our theoretical findings. 
Finally, Section~\ref{sec:conclusion} concludes the paper with a brief discussion.

\section{The Potential of the \Sine{} Activation Function}
\label{sec:sine:potential}

In this section, we begin with a detailed analysis of the MMNN architecture in Section~\ref{sec:MMNN:structure}, followed by an examination of its mathematical approximation capabilities using \texttt{sine} and \SinTU{s} activation functions in Section~\ref{sec:approx:power}.
In Section~\ref{sec:landscape}, we address practical considerations such as the cost function landscape and the interaction between \texttt{sine} activation functions and MMNNs, emphasizing the advantages of fixing weights within activation functions. We conclude with a discussion of related work in Section~\ref{sec:related:work}.

\subsection{Structure of MMNNs}
\label{sec:MMNN:structure}

Before presenting the main results, we first introduce the architecture of MMNNs. 
An MMNN is a multi-layer composition of functions $\bmh_i$, formally defined as $\bmh: \R^{d_0}\to \R^{d_m}$ with
\begin{equation}
    \label{eq:h:MMNN}
    \bmh = \bmh_m \circ \bmh_{m-1} \circ \cdots \circ \bmh_1,
\end{equation}
where each layer $\bmh_i: \R^{d_{i-1}} \to \R^{d_i}$ represents a multi-component shallow network with width $n_i$ and $d_i$ components, given by
\begin{equation*}
    \bmh_i(\bmx) = \bmA_i \, \sigma(\bmW_i \bmx + \bmb_i) + \bmc_i,
\end{equation*}
where $\bmW_i \in \R^{n_i \times d_{i-1}}$, $\bmb_i\in \R^{n_i}$, $\bmA_i \in \R^{d_i \times n_i}$, and $\bmc_i\in\R^{d_i}$. Here, $\sigma(\bmW_i{[j,:]} \cdot \bmx + \bmb_i[j])$ for $ j = 1, 2, \dots, n_i$ act as randomly parameterized basis functions in $\R^{d_{i-1}}$. Each component $\bmA_i[k,:]\sigma(\bmW_i\bmx+\bmb_i) +\bmc_i[k]$, for $ k=1,2, \ldots, d_i$, is a linear combination of these basis functions. 

In each layer, the number of components $d_{i}$, referred to as rank, is significantly smaller than the number of hidden neurons $n_i$, known as the layer width. The utilization of a diverse set of random basis functions, enabled by $n_i\gg d_{i-1}$, along with their well-conditioned nature due to random parametrization, facilitates easy training of $\bmA_i$ and $\bmc_i$ to approximate smooth functions in $\R^{d_{i-1}}$. By integrating multiple components per layer and composing multiple layers, this balance between rank and width, combined with the flexible component structure employing random bases, enhances the effectiveness of MMNNs in both representation and learning.
The \textbf{width} of an MMNN is defined as $\max\{n_i : i = 1, 2, \dots, m-1\}$, the \textbf{rank} as $\max\{d_i : i = 1, 2, \dots, m-1\}$, and the \textbf{depth} as $m$. For convenience, we use the compact notation $(N, R, L)$ to denote a network of width $N$, rank $R$, and depth $L$. In most of our experiments, we assume equal layer width and rank, i.e., $n_i=N$ and $d_i=R$.

In summary, MMNNs consider each component as a fundamental unit, where it consists of a linear combination of randomly parameterized neurons (basis functions). This contrasts with FCNNs, which treat individual neurons as the primary units. Components within each layer are combined and further composed across layers to effectively approximate target functions. The MMNN structure is enriched by introducing rank as an additional dimension alongside width and depth, offering greater flexibility in network architecture.
Furthermore, the training paradigm for MMNNs diverges significantly from that of FCNNs. 
Within each MMNN layer, represented by $\bmA \, \sigma(\bmW \bmx + \bmb) + \bmc$, the set
\[
\big\{ \sigma(\bmW{[j,:]} \cdot \bmx + \bmb[j]) : j = 1, 2, \dots, n \big\}
\]
is interpreted as a shared random basis for all components. Consequently, during training, only the parameters $\bmA$ and $\bmc$ are updated, while $\bmW$ and $\bmb$ remain fixed after random initialization.

To mitigate the vanishing gradient issue in deep MMNNs, techniques inspired by ResNets \cite{7780459} can be employed to improve training efficiency. Building on this concept, we introduce the ResMMNN, which modifies the structure of \eqref{eq:h:MMNN} as 
\begin{equation*}
    \bmh = \bmh_m \circ (\bmI + \bmh_{m-1}) \circ \cdots \circ (\bmI + \bmh_3) \circ (\bmI + \bmh_2) \circ \bmh_1,
\end{equation*}
where $\bmI$ denotes the identity mapping. This definition of ResMMNN can be further generalized by applying identity mappings selectively to specific layers. Such variations are still referred to as ResMMNNs. See Figure~\ref{fig:MMNN:eg} for an illustration of a ResMMNN with size $(6,2,3)$. 
Furthermore, additional layer operations, such as Batch Normalization \cite{pmlr-v37-ioffe15} and Dropout \cite{JMLR:v15:srivastava14a}, can also be applied to specific layers of MMNNs to enhance training stability, accelerate convergence, and improve generalization, among other benefits.

\begin{figure}[ht]
	\centering
	\includegraphics[width=0.95\linewidth]{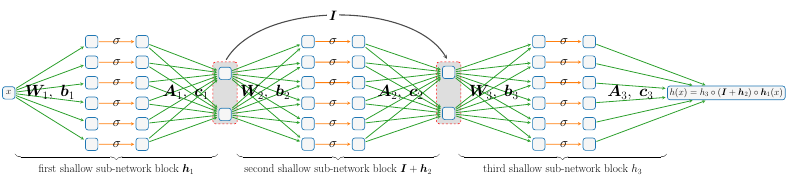}
\caption{An illustration of a ResMMNN of size $(6,2,3)$. During training, only the parameters $\bmA_i$'s and $\bmc_i$'s are updated, while $\bmW_i$'s and $\bmb_i$'s are randomly initialized and remain fixed.}
	\label{fig:MMNN:eg}
\end{figure}

\subsection{Approximation Capability}
\label{sec:approx:power}

We first introduce some notations before presenting our main results on the exponential approximation capabilities of MMNNs using \texttt{sine} or \SinTU{s} as activation functions.
We denote $\mn{\varrho}{N}{R}{L}{\R^d}{\R^n}$ as the set of vector-valued functions $\bmphi:\mathbb{R}^d\to\mathbb{R}^n$ that can be represented
by $\varrho$-activated 
MMNNs
of width $\le N\in \N^+$, rank $\le R\in\N^+$, and depth $\le L\in\N^+$. 
Additionally, in this notation, if $\varrho$ is replaced by $(\varrho_1,\dots,\varrho_k)$, it indicates that each neuron can be activated by any of $\varrho_i$'s.
Let
$\omega_f(\cdot)$ be the modulus of continuity of $f\in C([0,1]^d)$ defined via
\begin{equation*}
    \omega_f(t)\coloneqq \sup \big\{|f(\bmx)-f(\bmy)|: \|\bmx-\bmy\|_2\le t,\,\;\bmx,\bmy\in [0,1]^d\big\}\quad \tn{for any $t\ge 0$.}
\end{equation*}
 Let $\calS$ denote  the set of \SinTU{s}, i.e., 
\begin{equation*}
    \calS\coloneqq \{\SinTU_s:s\in\R\},\quad \tn{where}\quad \SinTU_s\coloneqq \sin\circ \calT_s\quad \tn{and}\quad \calT_s(x) \coloneqq \max\{x,\, s\}=
\begin{cases} 
x & \text{if } x \ge s, \\ 
s & \text{if } x < s.
\end{cases}
\end{equation*}
With the above notations, we present the following theorem, which demonstrates that MMNNs activated by \SinTU{s} possess exponential approximation power. 

\begin{theorem}
\label{thm:main1}
Given $f \in C([0,1]^d)$ and $\varrho\in\calS$, for any $N,L \in \mathbb{N}^+$ and $p \in [1,\infty)$, there exists  
$\phi\in \mn[\big]{\varrho}{2d(4N-1)}{3d}{L+2}{\mathbb{R}^d}{\mathbb{R}}$
such that
\[
\| \phi - f\|_{L^p([0,1]^d)} \leq  2\sqrt{d}\cdot\omega_f\big(N^{-L}\big).
\]
\end{theorem}


The preceding theorem establishes the approximation capabilities of MMNNs activated by \SinTU{s}, which are truncated variations of the \sine{} function. Next, we explore the case where the pure \sine{} function is used as the activation function. However, due to its lack of singularity, the \sine{} function poses challenges in spatial localization, making it difficult to construct mathematical frameworks for spatial decomposition based on continuity.
To overcome this mathematical challenge, we introduce \ReLU{} as an additional activation function. This modification enables a more effective spatial decomposition, leading to the following theorem.

\begin{theorem}
\label{thm:main2}
Given $f \in C([0,1]^d)$, for any $N,L \in \mathbb{N}^+$ and $p \in [1,\infty)$, there exists  
$\phi\in \mn[\big]{(\sine,\,\ReLU)}{d(4N-1)}{3d}{L+2}{\mathbb{R}^d}{\mathbb{R}}$
such that
\[
\| \phi - f\|_{L^p([0,1]^d)} \leq  2\sqrt{d}\cdot\omega_f\big(N^{-L}\big).
\]
\end{theorem}
\begin{remark}
    The theorem does not impose a specific arrangement of $\sine$ and \ReLU{}. However, our proof demonstrates that applying \ReLU{} to all but the last two hidden layers, where $\sine$ is used, is sufficient. This result is theoretical; in practice, additional $\sine$ activation functions may be required, as discussed later.
\end{remark}

The proofs of Theorems~\ref{thm:main1} and~\ref{thm:main2} will be presented in Section~\ref{sec:proof:thm:main}. 
We adopt a notation for FCNNs analogous to \(\mn{\varrho}{N}{R}{L}{\R^d}{\R^n}\) used for MMNNs. Specifically, let \(\nn{\varrho}{N}{L}{\R^d}{\R^n}\) denote the set of vector-valued functions \(\bmphi: \mathbb{R}^d \to \mathbb{R}^n\) that can be realized by \(\varrho\)-activated FCNNs with width at most \(N \in \mathbb{N}^+\) and depth at most \(L \in \mathbb{N}^+\).
Similarly, if \(\varrho\) is replaced by \((\varrho_1,\dots,\varrho_k)\), it indicates that each neuron can be activated by any of the \(\varrho_i\)'s.
As a direct consequence of Theorems~\ref{thm:main1} and \ref{thm:main2}, we establish the following two corollaries for FCNNs.

\begin{corollary}
\label{coro:main1}
Given $f \in C([0,1]^d)$ and $\varrho\in\calS$, for any $N,L \in \mathbb{N}^+$ and $p \in [1,\infty)$, there exists  
$\phi\in \nn[\big]{\varrho}{2d(4N-1)}{L+2}{\mathbb{R}^d}{\mathbb{R}}$
such that
\[
\| \phi - f\|_{L^p([0,1]^d)} \leq  2\sqrt{d}\cdot\omega_f\big(N^{-L}\big).
\]
\end{corollary}

\begin{corollary}
\label{coro:main2}
Given $f \in C([0,1]^d)$, for any $N,L \in \mathbb{N}^+$ and $p \in [1,\infty)$, there exists 
$\phi\in \nn[\big]{(\sine,\,\ReLU)}{d(4N-1)}{L+2}{\mathbb{R}^d}{\mathbb{R}}$
such that
\[
\| \phi - f\|_{L^p([0,1]^d)} \leq  2\sqrt{d}\cdot\omega_f\big(N^{-L}\big).
\]
\end{corollary}


\begin{remark}
    It is worth highlighting the substantial difference in the total number of parameters between an MMNN and an FCNN. For an MMNN of width \( N \), rank \( R \), and depth \( L \), the parameter count is \( O(NRL) \), whereas for an FCNN of width \( N \) and depth \( L \), it is \( O(N^2L) \). Notably, in an MMNN, the rank \( R \) (the number of components in each layer) is significantly smaller than the network width \( N \) (the number of random hidden neurons per layer), which guarantees that the set of $N$ random basis functions is diverse enough to approximate smooth functions in the input space of dimension $R$ from the previous layer. Additionally, as previously discussed, only approximately half of the total parameters in an MMNN are trained.
\end{remark}

\begin{remark}
By applying techniques from \cite{shijun:smooth:functions} (specifically Theorem~2.1), the above results could be extended to the \( L^\infty \)-norm, although the constants involved would be considerably larger. The extension involves more technical complexities and is of little importance to the main themes of this paper, so we do not pursue it here.
\end{remark}


The proofs of Theorems~\ref{thm:main1} and \ref{thm:main2} (see Section~\ref{sec:proof:thm:main}) rely on two key components. 
The first component involves constructing a subnetwork that partitions a $d$-dimensional unit hypercube into uniform subcubes of small size, with only a minor discrepancy due to the continuity of the activation function. Within each subcube, the function is approximated by a constant function.
The second component is the existence of a subnetwork that maps the index of each subcube to the function value at a representative point within the subcube (e.g., its center). In designing this subnetwork, it suffices to ensure accuracy at a finite set of points rather than over an entire interval. This highlights the power of composition, which simplifies the construction.
As we shall see later, the periodicity and irrationality of the \texttt{sine} function play a crucial role in efficiently addressing the second component. Specifically, for any $\epsilon > 0$ and any $M \in \mathbb{N}^+$, given $f_n \in [-1,1]$, there exist suitable values of $v$ and $w$ such that
\begin{equation}
\label{eq:sin:sin::-:fn}
    \left| \sin \big( v\cdot \sin (n w) \big) - f_n \right| < \epsilon  \quad \text{for } n = 1, 2, \dots, M.
\end{equation}

Many mathematical approximation results show theoretical representation power; however, in practice, a more important issue is whether one has an efficient training strategy to achieve a good computational performance. For most mathematical neural network approximation results (constructive or non-constructive), the network parameters depend on the target function nonlinearly and globally. On the other hand, most training processes are gradient descent based (first order) methods which are very local and sensitive to ill-conditioning of the cost function in terms of a very large number of parameters. This typically leads to a gap between the theoretical results and practical performance. In our case, although two \sine{} functions (or \SinTU{s}) are theoretically sufficient for value fitting (e.g., Equation~\eqref{eq:sin:sin::-:fn}), finding the two appropriate numbers, $v$ and $w$, is impractical in general. Consequently, a larger network with multiple components and layers is essential for effective optimization.

Mathematical and numerical investigations in later sections demonstrate that using \texttt{sine} or \SinTU{s} as activation functions in MMNNs with well-balanced structures significantly improves the network’s capability and learning efficiency. This is consistent with the key feature of MMNNs that each component, which is a one-hidden-layer network, only needs to approximate a smooth function and can be trained effectively while Fourier series can approximate smooth functions efficiently.

\subsection{Optimization Landscapes}
\label{sec:landscape}

In Section~\ref{sec:approx:power}, we demonstrate that MMNNs activated by \texttt{sine} or \SinTU{s} possess strong approximation capabilities. However, having good approximation power only reflects theoretical potential and does not necessarily guarantee effective learning in practice. 
Next, we discuss the practical learning difficulty of MMNNs activated by \texttt{sine} or \SinTU{s}. We focus on the most intuitive aspect: the landscape of the cost function with respect to the network parameters, which serves as an indicator of the training complexity in practice. This analysis is conducted across various network architectures and activation functions.











We first consider three basic cases where the target function \( f \) takes the following forms: 
\[
\sin(w^* x + b^*), \quad \sum_{i=1}^2 \sin(w_i^* x), \quad \text{and} \quad \sin\big(w_2^* \sin(w_1^* x)\big),
\]
respectively. The corresponding cost functions are given by
\begin{equation*}
    \mathcal{L}_1(w_1, w_2) = \int_{-\pi}^{\pi} \Big(\sin(w_1x + w_2) - \sin(w_1^* x + w_2^*)\Big)^2 \, dx,
\end{equation*}
\begin{equation*}
    \mathcal{L}_2(w_1, w_2) = \int_{-\pi}^{\pi} \Bigg(\sum_{i=1}^2 \sin(w_i x) - \sum_{i=1}^2 \sin(w_i^* x)\Bigg)^2 \, dx,
\end{equation*}
and
\begin{equation*}
    \mathcal{L}_3(w_1, w_2) = \int_{-\pi}^{\pi} \Big(\sin\big(w_2 \sin(w_1 x)\big) - \sin\big(w_2^* \sin(w_1^* x)\big)\Big)^2 \, dx.
\end{equation*}
\begin{figure}[ht]
            \centering

            \,\hfill
            \begin{subfigure}[b]{0.302245\textwidth}
                    \centering            
                    \includegraphics[width=0.999\textwidth]{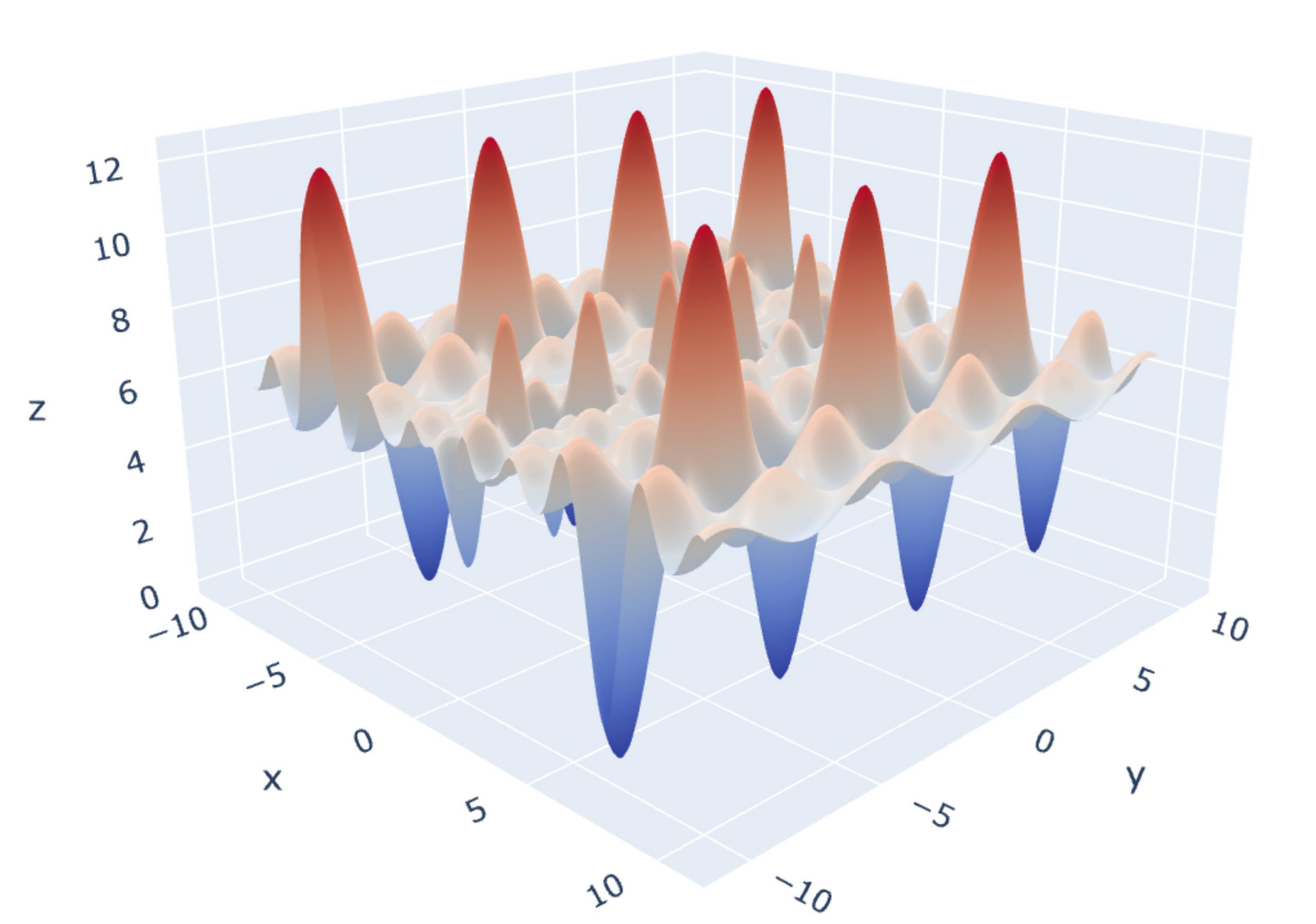}
                    \subcaption{$\calL_1(w_1,w_2)$.}
                \end{subfigure}
                \hfill
            \begin{subfigure}[b]{0.302245\textwidth}
                    \centering            \includegraphics[width=0.999\textwidth]{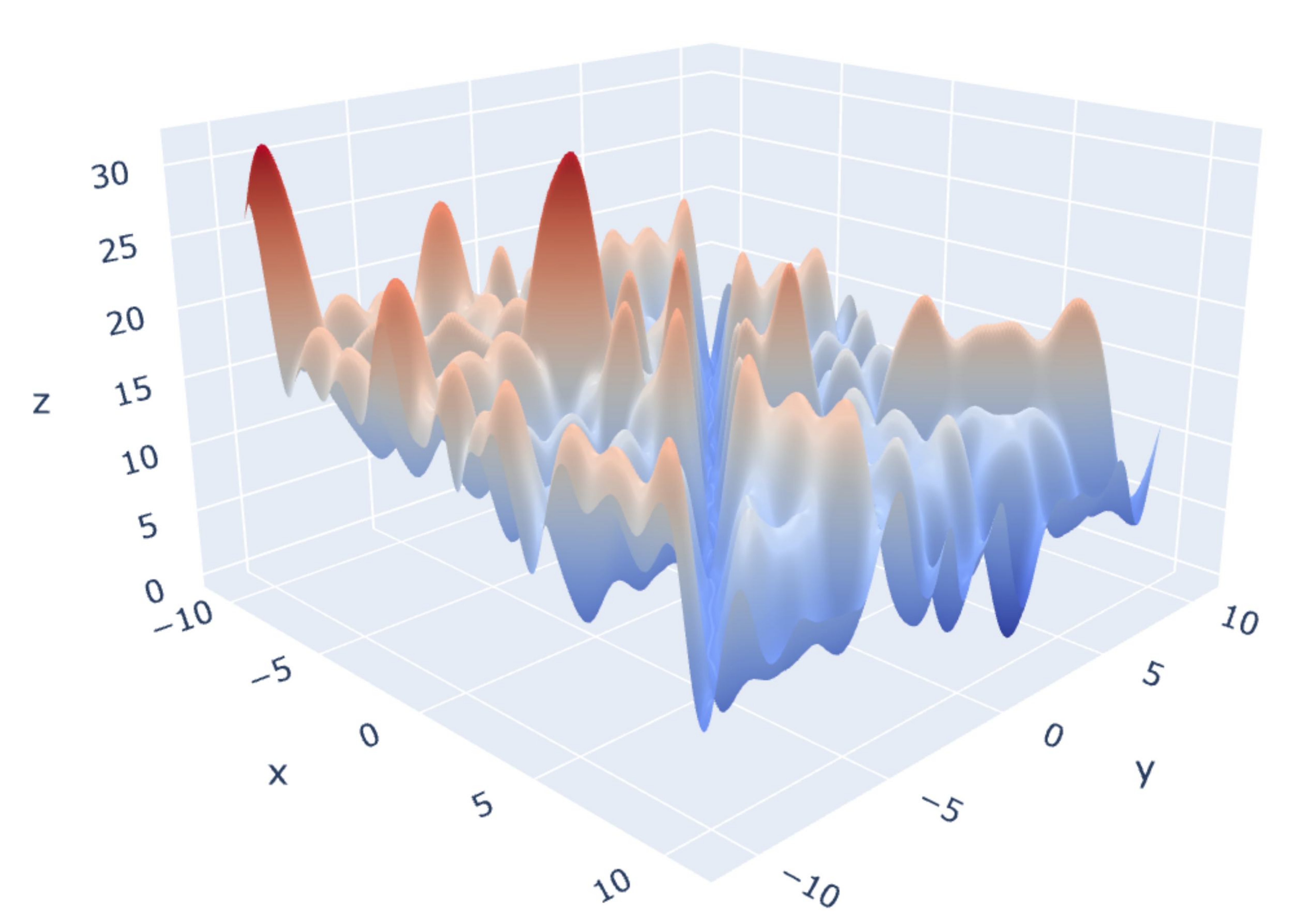}
                    \subcaption{$\calL_2(w_1,w_2)$.}
                \end{subfigure}
            \hfill
            \begin{subfigure}[b]{0.302245\textwidth}
                    \centering            \includegraphics[width=0.999\textwidth]{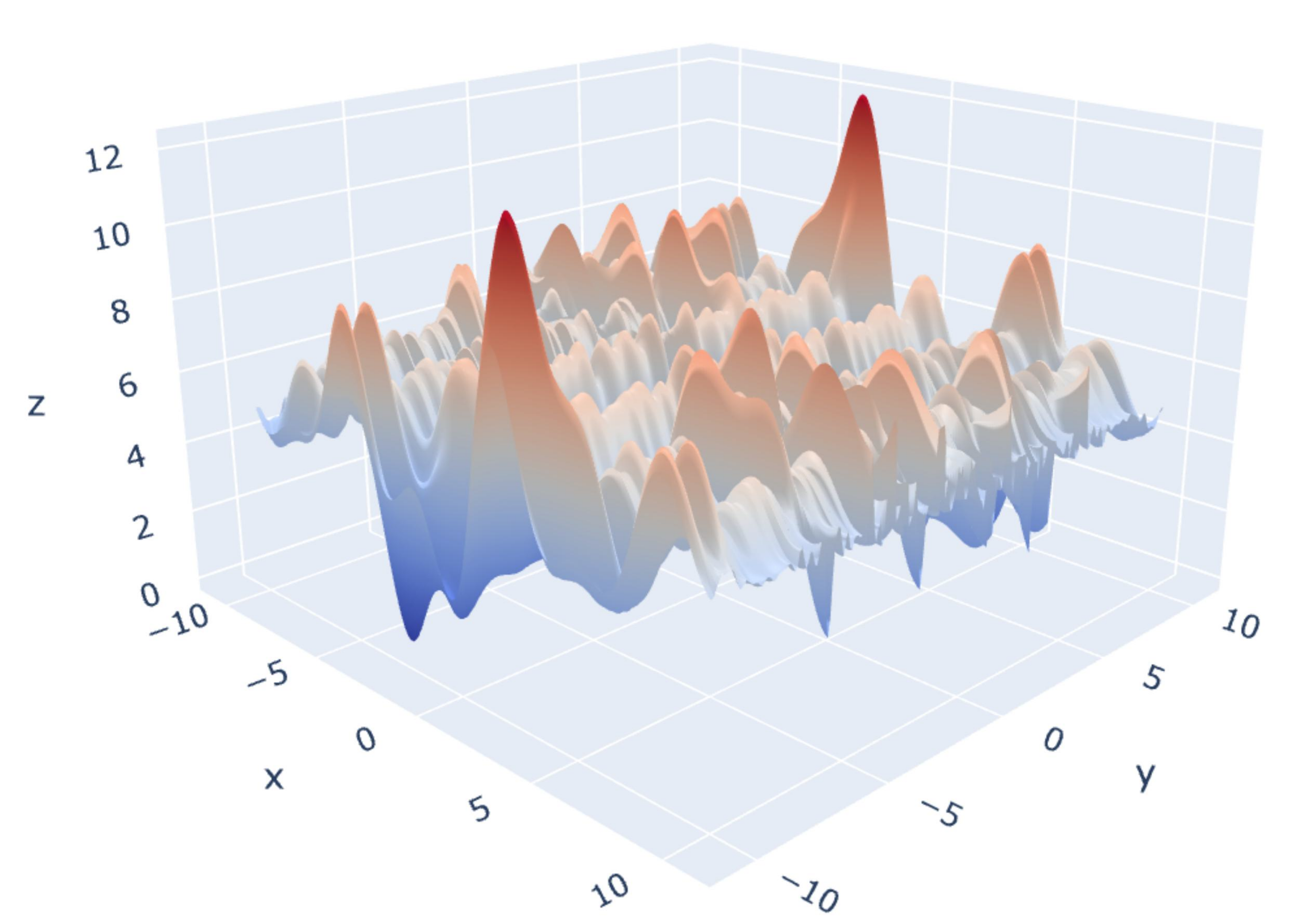}
                    \subcaption{$\calL_3(w_1,w_2)$.}
                \end{subfigure}\hfill
                \,                
    \caption{
    Landscape visualizations of $\mathcal{L}_i$ for $i = 1,2,3$, where $x$, $y$, and $z$ represent $w_1$, $w_2$, and $\mathcal{L}_i(w_1,w_2)$, respectively.
}
\label{fig:landscape:three:base:cases}
\end{figure}

\noindent
Here, we use the integration range \((-\pi, \pi)\) instead of \((-1, 1)\) because the \sine{} function has a period of \(2\pi\), which simplifies the calculations and makes the test more straightforward.
The landscapes of these three cost functions are illustrated in Figure~\ref{fig:landscape:three:base:cases}. As observed, the landscapes are quite complex, featuring numerous local minima. This indicates that using the \sine{} function as an activation function poses significant challenges for effective learning in practice.
We will later see that this issue is particularly severe for FCNNs. However, the structure of MMNNs simplifies the landscape, making them more conducive to effective learning.

Next, we examine general network architectures, specifically FMMNNs and FCNNs.
To ensure a fair comparison, the FCNN is designed to have a comparable number of learnable parameters to the MMNN.
The FCNN is defined as
\begin{equation}\label{eq:FCNN:eg}
    \phi(x) = \mathcal{A}_2 \circ \varrho \circ {\mathcal{A}}_1 \circ \varrho \circ {\mathcal{A}}_0(x),
\end{equation}
where \({\mathcal{A}}_0: \mathbb{R} \to \mathbb{R}^{64}\), \({\mathcal{A}}_1: \mathbb{R}^{64} \to \mathbb{R}^{64}\),   \(\mathcal{A}_2: \mathbb{R}^{64} \to \mathbb{R}\) are affine linear maps, and $\varrho$ is either the \sine{} or $\SinTU_0$ activation function.
The MMNN is defined as
\begin{equation}\label{eq:MMNN:eg}
    \phi(x) = \tilde{\mathcal{A}}_3 \circ \varrho \circ {\tilde{\mathcal{A}}}_2 \circ {\tilde{\mathcal{A}}}_1 \circ \varrho \circ {\tilde{\mathcal{A}}}_0(x),
\end{equation}
where \({\tilde{\mathcal{A}}}_0: \mathbb{R} \to \mathbb{R}^{128}\), \({\tilde{\mathcal{A}}}_1: \mathbb{R}^{128} \to \mathbb{R}^{32}\), \({\tilde{\mathcal{A}}}_2: \mathbb{R}^{32} \to \mathbb{R}^{128}\),  \(\tilde{\mathcal{A}}_3: \mathbb{R}^{128} \to \mathbb{R}\) are affine linear maps.
The cost function is given by
\begin{equation*}
    \calL(w_1,w_2)=\int_{-\pi}^{\pi} \Big(\phi(x)- f(x)\Big)^2 \, dx,\quad \tn{where}\quad 
    f(x)=\frac{1}{1+100x^2}.
\end{equation*}

As shown in (a), (b), (d), and (e) of Figure~\ref{fig:landscape:FCNN:vs:MMNN}, the learning landscape of MMNNs is considerably simpler than that of FCNNs. Likewise, (b), (c), (e), and (f) of Figure~\ref{fig:landscape:FCNN:vs:MMNN} clearly illustrate that our learning strategy, which involves fixing parameters in $\tildecalA_2$ while optimizing those in $\tildecalA_1$ rather than the reverse, is well-justified and reasonable. We would like to point out that
different initializations can produce varying results; however, the overall landscape complexity remains largely consistent. The figures shown in Figure~\ref{fig:landscape:FCNN:vs:MMNN} represent relatively complex cases among several initializations with identical settings.
 We have experimented with deeper networks and other target functions, and the results are generally similar to the two-hidden-layer cases presented.
 For deeper MMNNs, if two parameters are selected from the trainable parameters (i.e., \( \bmA_i \)'s and \( \bmc_i \)'s, see Section~\ref{sec:MMNN:structure}), the landscape always remains simple, reflecting the effectiveness and rationality of our training strategy.



\begin{figure}[ht]
            \centering
            %
            \,\hfill\begin{subfigure}[b]{0.27300245\textwidth}
                    \centering            
                    \includegraphics[width=0.999\textwidth]{figures/LandScapeFCNN1Act1.pdf}
                    \subcaption{FCNN ($\calA_1$ in \eqref{eq:FCNN:eg}).}
                \end{subfigure}
                \hfill
            \begin{subfigure}[b]{0.27300245\textwidth}
                    \centering            \includegraphics[width=0.999\textwidth]{figures/LandScapeMMNN1Act1.pdf}
                    \subcaption{MMNN ($\tildecalA_1$ in \eqref{eq:MMNN:eg}).}
                \end{subfigure}
                \hfill
            \begin{subfigure}[b]{0.27300245\textwidth}
                    \centering            \includegraphics[width=0.999\textwidth]{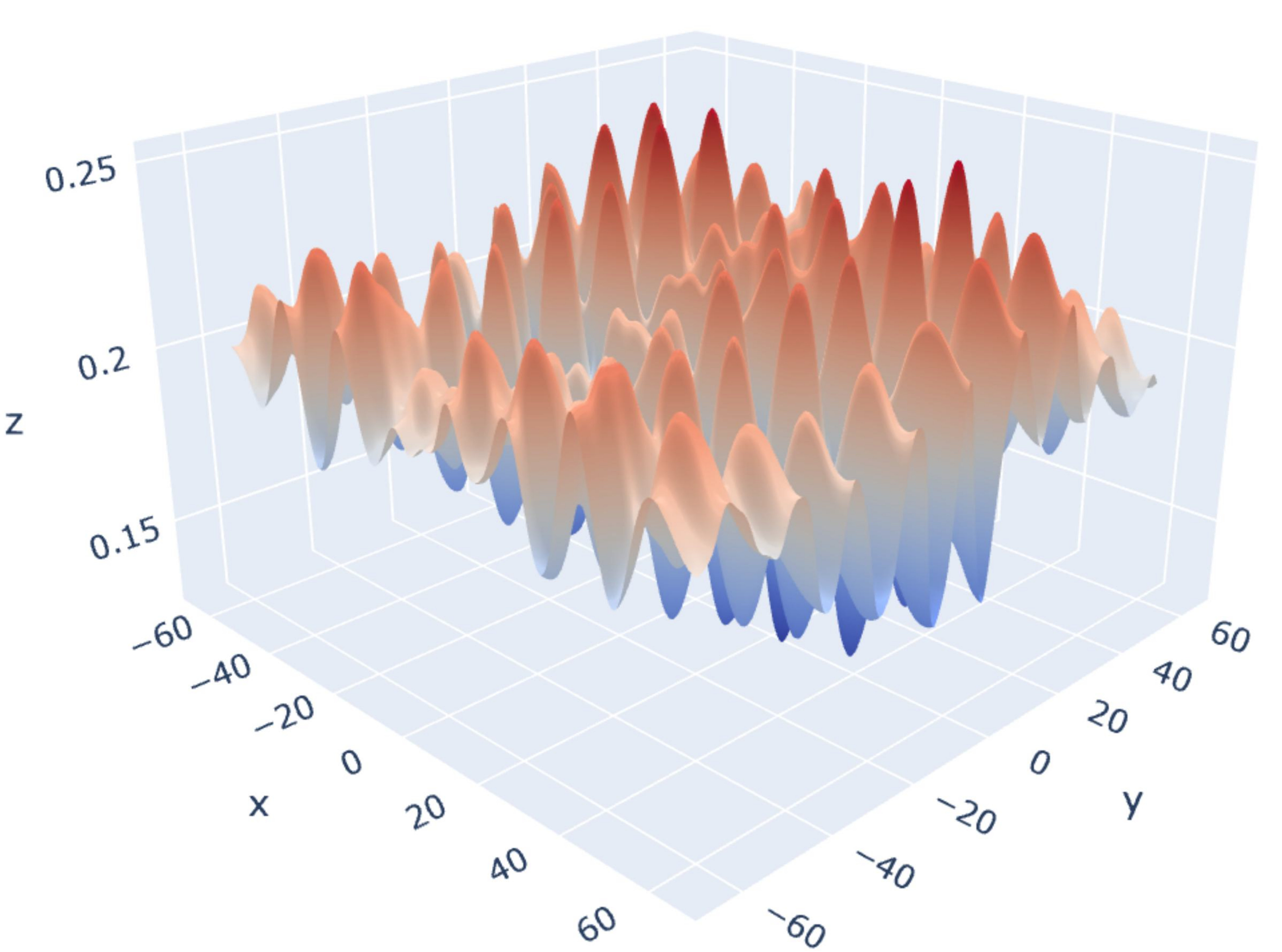}
                    \subcaption{MMNN ($\tildecalA_2$ in \eqref{eq:MMNN:eg}).}
                \end{subfigure}\hfill
                \,
                \\
                \,\hfill\begin{subfigure}[b]{0.27300245\textwidth}
                    \centering            
                    \includegraphics[width=0.999\textwidth]{figures/LandScapeFCNN1Act2.pdf}
                    \subcaption{FCNN ($\calA_1$ in \eqref{eq:FCNN:eg}).}
                \end{subfigure}
                \hfill
            \begin{subfigure}[b]{0.27300245\textwidth}
                    \centering            \includegraphics[width=0.999\textwidth]{figures/LandScapeMMNN1Act2.pdf}
                    \subcaption{MMNN ($\tildecalA_1$ in \eqref{eq:MMNN:eg}).}
                \end{subfigure}
                \hfill
            \begin{subfigure}[b]{0.27300245\textwidth}
                    \centering            \includegraphics[width=0.999\textwidth]{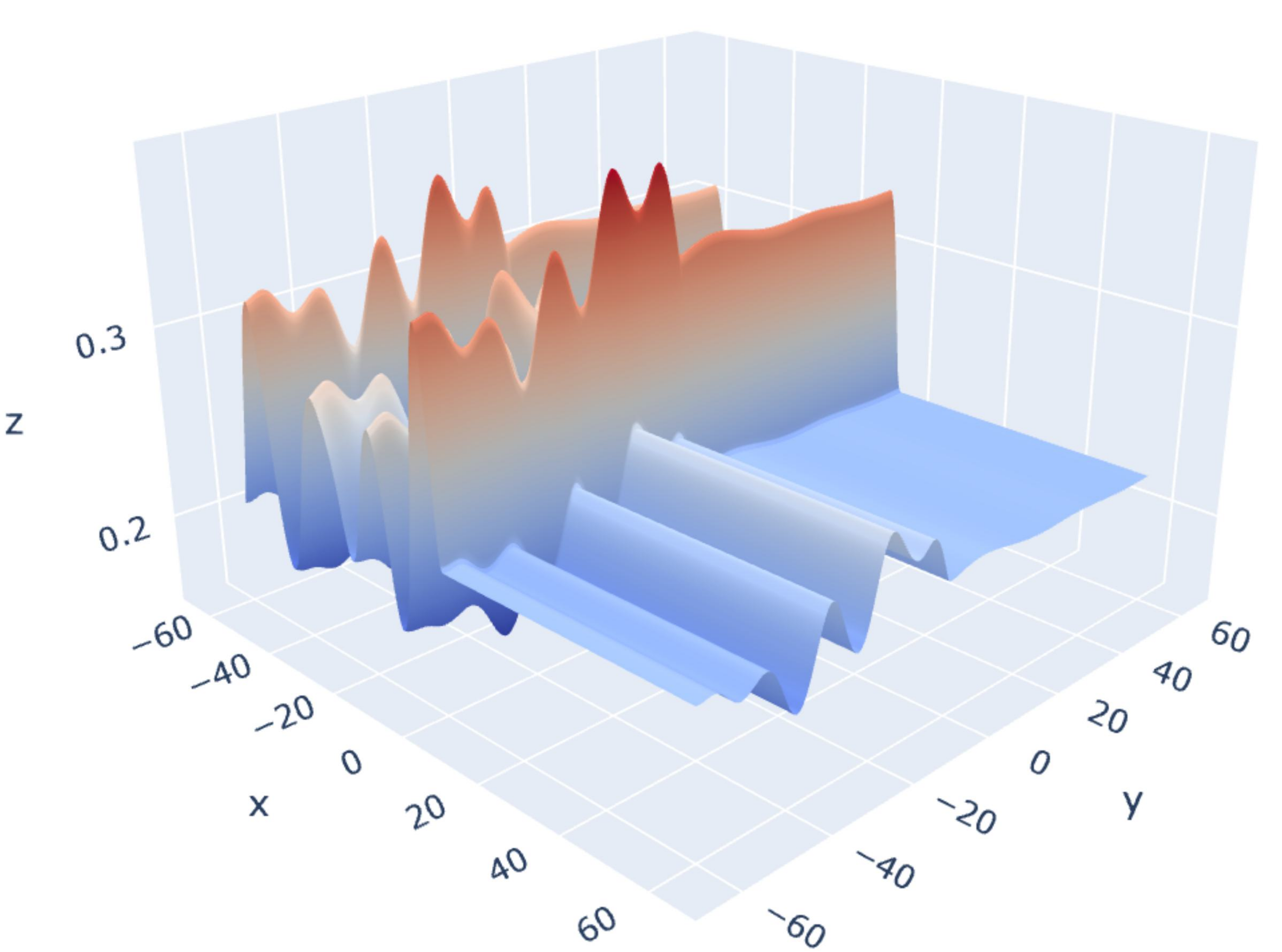}
                    \subcaption{MMNN ($\tildecalA_2$ in \eqref{eq:MMNN:eg}).}
                \end{subfigure}\hfill
                \,
                
                \caption{Comparison of the cost function landscapes for FCNNs and MMNNs. All parameters are initialized using PyTorch's default linear initialization. Here, $z$ represents the cost function, while $x$ and $y$ denote two parameters  from the weights of
\({\mathcal{A}}_1\), \({\tilde{\mathcal{A}}}_1\), and \({\tilde{\mathcal{A}}}_2\) in 
 \eqref{eq:FCNN:eg} and \eqref{eq:MMNN:eg}.
The top and bottom rows correspond to the \sine{} and $\SinTU_0$ activation functions, respectively.
}

    \label{fig:landscape:FCNN:vs:MMNN}
    
\end{figure}

\subsection{Related Work}
\label{sec:related:work}

Extensive research has explored the approximation capabilities of neural networks across various architectures. Early works established the universal approximation theorem for single-hidden-layer networks~\cite{Cybenko1989ApproximationBS,HORNIK1991251,HORNIK1989359}, proving their ability to approximate specific functions arbitrarily well, though without quantifying error in relation to network size. Subsequent studies~\citep{yarotsky18a,yarotsky2017,B_lcskei_2019,ZHOU2019,10.3389/fams.2018.00014,2019arXiv190501208G,2019arXiv190207896G,MO,shijun:NonlineArpprox,shijun:Characterized:by:Numer:Neurons,shijun:smooth:functions,shijun:arbitrary:error:with:fixed:size,shijun:thesis,shijun:intrinsic:parameters,yarotsky:2019:06, fang2024addressing,shijun:2023:beyond:ReLU:to:diverse:actfun} analyzed approximation errors, relating them to network width, depth, or parameter count, and addressed the spectral bias in neural network approximations. 
Here, we specifically highlight two papers \cite{shijun:arbitrary:error:with:fixed:size, yarotsky:2019:06} that are closely related to our theoretical results.
 The results in \cite{yarotsky:2019:06} imply that \ReLU/\sine-activated fully connected neural networks (FCNNs) with width $O(d)$ and depth $O(L)$ can approximate a 1-Lipschitz function $f:[0,1]^d\to\mathbb{R}$ within an error of $O(2^{-\sqrt{L}})$. Compared to this, our work achieves several key improvements: 
1) Our results incorporate width $N$, extending beyond fixed-width networks.  
2) We improve the approximation error rate to $O(N^{-L})$, much better than $O(2^{-\sqrt{L}})$ when $N$ is fixed.  
3) We introduce \SinTU{}, a novel hybrid activation function, where a single \SinTU{} function can replace two activation functions (\ReLU,\,\sine) in the approximation results.  
4) Our results specifically apply to MMNNs, which introduce an additional dimension called rank beyond width and depth.  
In \cite{shijun:arbitrary:error:with:fixed:size}, the authors propose a simple activation function, \EUAF, and demonstrate that a fixed-size \EUAF-activated FCNN can approximate any continuous function $f\in C([0,1]^d)$ to an arbitrarily small error by adjusting only finitely many parameters. 
\EUAF{} emphasizes theoretical approximation but tends to perform less effectively in practice, often appearing somewhat artificial. 
In contrast, our FMMNN is naturally structured. Although its theoretical approximation power is comparatively weaker, an exponential approximation rate is typically sufficient in practical applications. Our extensive experiments further confirm its effectiveness, demonstrating surprisingly strong empirical performance.

\begin{colorenv}[blue]
High-frequency function approximation has been studied from several
complementary perspectives. One important line of work concerns the frequency
principle or spectral bias of neural networks, which observes that standard
fully connected neural networks (FCNNs), also known as multi-layer perceptrons
(MLPs), tend to learn low-frequency components earlier than high-frequency
components during training
\cite{pmlr-v97-rahaman19a,xu2019frequency,NEURIPS2020_55053683}. This provides
one explanation for the difficulty of fitting highly oscillatory targets with
standard activations such as \ReLU, \GELU, \texttt{sigmoid}, and \texttt{tanh}
\cite{10.5555/3104322.3104425,2016arXiv160608415H}.

Several frequency-aware designs have been proposed to alleviate this difficulty.
Periodic or sinusoidal activations introduce an explicit oscillatory inductive
bias. For instance, SIREN-type networks use sine activations together with a
carefully hand-tuned initialization and have shown strong performance in implicit
neural representations and derivative fitting
\cite{NEURIPS2020_53c04118}. Snake-type activations add a periodic correction
to a linear term and are designed to improve the learning of periodic structures
\cite{NEURIPS2020_11604531}. Related sinusoidal mechanisms have also been
studied in different neural-network models
\cite{novello2024tamingfrequencyfactorysinusoidal,2025arXiv250200869M,
fathony2021multiplicative,doi:10.1137/21M144431X,WANG2021113938}. Another
widely used approach is to map the input coordinate into a trigonometric feature
space before applying an MLP, as in Fourier-feature or positional-encoding
networks. Such methods make high-frequency basis functions directly available
to the network and can substantially improve coordinate-based regression
\cite{NEURIPS2020_55053683,WANG2021113938,10.1145/3503250}.

Other methods incorporate frequency information in more structured ways.
WIRE/Gabor-type networks combine oscillation with spatial localization through
wavelet-like or Gabor-type activations, which is useful for representing local
oscillatory signals \cite{10204916}. Recent variable-periodic activations, such
as FINER and FINER++, further tune the supported frequency range through
activation or bias design \cite{liu2024finer,zhu2024finerpp}. Multi-scale
methods, including PhaseDNN and MscaleDNN, address high-frequency learning by
phase shifting, scaling, or decomposing the target into multiple frequency
regimes
\cite{doi:10.1137/19M1310050,liu2020mscalednn,cai2019mscalednn}. In scientific
computing and dynamical systems, oscillatory and multi-scale structures also
appear naturally; recent neural-computational studies on nonlinear chaotic and
fractional-order systems further illustrate this broader motivation
\cite{hassan2024lorenzluchen,raja2025fitzhughnagumo,raja2025hindmarshrose}.

Our approach is related to the above works in its emphasis on frequency-aware
representation, but differs in the role played by the architecture. FMMNNs do not just replace a standard activation with a Fourier-type activation, nor do they merely add Fourier features to the input. Instead, they combine sine-type randomized
components with the structured multi-component and multi-layer architecture of
MMNNs which can decompose a complicated function into smooth components. The key point is that each component can be represented and easily trained as a linear combination of fixed
random sine-type basis functions, while the multi-layer composition and
recombination mechanism produces more complex high-frequency features. Thus,
Fourier-type representation is used as a component-level building block rather
than as a stand-alone activation replacement in a conventional FCNN.

This distinction is important because using \sine{} as the activation function
in an FCNN does not automatically lead to robust performance. In our numerical
experiments, sine-activated FCNNs work well in some cases but are unstable or
less accurate in others, suggesting that the benefit of \sine{} depends strongly
on the NN architecture and optimization landscape. By contrast, the
MMNN structure and the \sine{} activation, or its variant \SinTU{s}, form a more
effective combination: the component-wise random-basis structure leads to a
simpler optimization problem, while the sine-type basis functions provide  effective building blocks in the frequency domain. In addition to the
numerical evidence, our approximation results and landscape observations provide
theoretical and empirical support for this architecture--activation matching
principle.
\end{colorenv}

\section{Numerical Experiments}
\label{sec:experiments}

In this section, we conduct extensive experiments to validate our analysis and demonstrate the effectiveness of MMNNs. Across all tests, we ensure: 1) Sufficient data sampling is used to capture fine details of the target function; 2) The Adam optimizer \cite{DBLP:journals/corr/KingmaB14} is used for training; 
3) Parameter initialization follows PyTorch's default settings;
4)  $\bmW$ and $\bmb$ (parameters inside activation functions) are fixed, while only $\bmA$ and $\bmc$ (parameters outside activation functions) are trained (see Section~\ref{sec:MMNN:structure} for details);
\begin{colorenv}[blue]
    5) We use the empirical mean squared error (MSE) as the training loss and report both MSE and the maximum absolute error (MaxE)  for  evaluation:
\[
    \mathrm{MSE}
    :=
    \frac{1}{|\mathcal{D}|}
    \sum_{\bmx_i\in \mathcal{D}}
    \bigl|\phi(\bmx_i)-f(\bmx_i)\bigr|^2,
    \qquad
    \mathrm{MaxE}
    :=
    \max_{\bmx_i\in \mathcal{D}}
    \bigl|\phi(\bmx_i)-f(\bmx_i)\bigr|,
\]
where \(\phi\) denotes the neural network approximation, \(f\) denotes the target function, and \(\mathcal{D}\)  denotes the training set or test sample set.
\end{colorenv}

\begin{colorenv}[blue]
Most comparisons in this section are controlled comparisons: the compared models
use the same target function, training and test samples, optimizer, learning-rate
schedule, precision, mini-batch size, and training budget, while only the factor
under comparison is changed, such as architecture, activation function, or
initialization strategy. For MMNN/FMMNN models, the internal random parameters
\(\bmW\) and \(\bmb\) are fixed after initialization and only \(\bmA\) and \(\bmc\) are
trained; for FCNN/MLP baselines, all affine parameters are trainable. We
therefore report trainable and total parameter counts when this distinction is
important.
\end{colorenv}

Section~\ref{sec:FMMNN_vs_FCNN} presents a comparison between MMNNs and FCNNs across a variety of activation functions, including \ReLU, \GELU, \texttt{tanh}, \texttt{sine}, \cosine, and \SinTU{s}. Section~\ref{sec:sine_vs_others} then focuses specifically on MMNNs, evaluating the performance of the \texttt{sine} activation relative to other activation functions. Finally, Section~\ref{sec:scaling:init} uses various numerical experiments to demonstrate that an initial scaling strategy can significantly speed up learning and improve performance when sufficient training data are available.







\subsection{MMNNs Versus FCNNs}
\label{sec:FMMNN_vs_FCNN}



We will thoroughly compare the performance of MMNNs and FCNNs using various activation functions.
\begin{colorenv}[blue]
SIREN-type models are sine-activated MLPs with carefully designed frequency
scaling and initialization. Therefore, the sine-activated FCNN baseline in this
section is related to, but is not intended to be a fully tuned SIREN baseline.
Here we use the same default initialization protocol as the other FCNN baselines
to isolate the effect of activation and architecture. The role of frequency
scaling is studied separately in Section~\ref{sec:scaling:init}.
\end{colorenv}
The overall message from these tests is that: 1) MMNNs perform better than FCNNs when the same activation function is used, and 2) FMMNNs consistently achieve the strongest or near-strongest results in these tests.
In our tests, we select rather complex target functions -- highly oscillatory with or without nonsmoothness, as both network types perform well on simple functions, making their differences less apparent. Additionally, our target functions should not be generated using \sine, \cosine, or their compositions and combinations, since we use \sine{} as the activation function.
Based on these considerations, we first choose an oscillatory target function $f_1\in C^\infty(\R)$ defined as
\begin{equation}\label{eq:def:f1:Cinfty:MMNN:vs:FCNN}
    f_1(x) = \frac{1}{1 + 2x^2}\sum_{i=-n}^{n}   (-1)^{(i\bmod 3)}\cdot\frac{|i| + n}{ n}  \cdot g\left( (2n+1)\left(x-\frac{i}{n+1} \right)\right),\\
\end{equation}
where \( n = 36 \), and \( (i \bmod 3)\in \{0,1,2\} \) represents the remainder when an integer \( i\in\Z \) is divided by 3.
Here, $g$ serves as a basis function, defined as
\begin{equation*}
   g(x) = \frac{g_0(x+1) g_0(1-x)}{g_0^2(1)},\quad \tn{where}\quad g_0(x) =
\begin{cases}
 \exp\left(-\frac{1}{x^2}\right) &\tn{if}\   x > 0, \\
 0 & \tn{if}\  x \leq 0.
\end{cases}
\end{equation*}
It is easy to verify that \( g_0 \in C^\infty(\mathbb{R}) \), and hence \( f_1 \in C^\infty(\mathbb{R}) \). See Figure~\ref{fig:fig:f1:g:f2:and:zoomin} for illustrations of \( f_1 \) and \( g \).
Next, we choose two nonsmooth oscillatory functions $f_2, f_3 \in C^0(\R)\setminus  C^1(\R)$ given by 
\begin{equation}
\label{eq:def:f2:C0:MMNN:vs:FCNN}
    f_2(x) = 
\frac{1 + 6x^8}{1 + 8x^6} \cdot 
\left( 120x^2 - 2 \left\lfloor \frac{120x^2 + 1}{2} \right\rfloor \right)^2
\end{equation}
and 
\begin{equation}
\label{eq:def:f3:C0:MMNN:vs:FCNN}
    f_3(x) = 
\frac{1 + 6x^8}{1 + 8x^6} \cdot 
\left( 32x  - 2 \left\lfloor \frac{32 x  + 1}{2} \right\rfloor \right)^2,
\end{equation}
where $\lfloor \cdot \rfloor$ denotes the floor function.
See Figure~\ref{fig:fig:f1:g:f2:and:zoomin} for visual depictions of \( f_2 \) and \( f_3 \).
For \( f_3 \), we use different mini-batch sizes to demonstrate that MMNNs are less sensitive to the hyper-parameters of training and therefore more stable compared to FCNNs.

\begin{figure}[ht]
            \centering
  
            \begin{subfigure}[b]{0.72\textwidth}
                    \centering            
                    \includegraphics[width=0.999\textwidth]{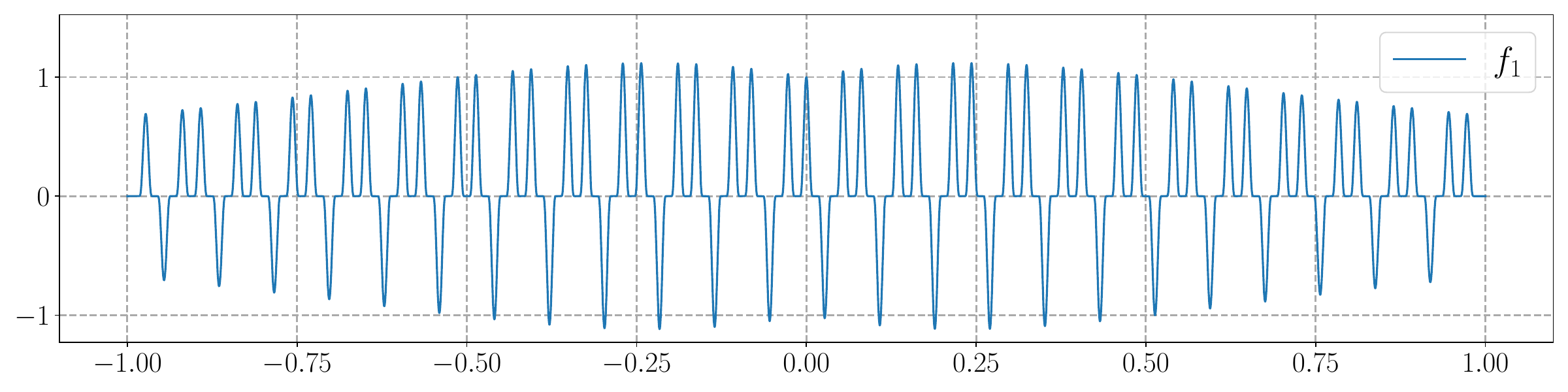}
                \end{subfigure}
            \begin{subfigure}[b]{0.27\textwidth}
                    \centering            \includegraphics[width=0.999\textwidth]{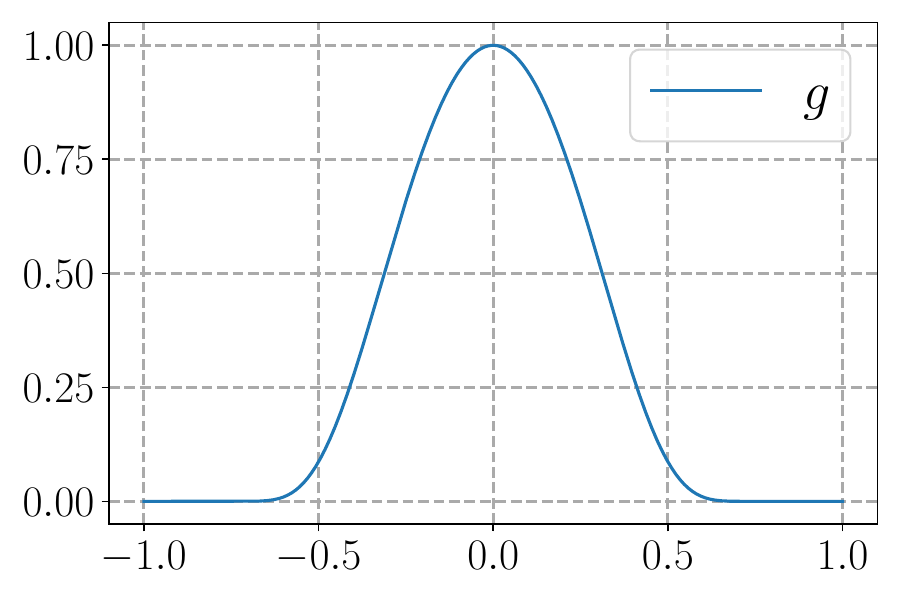}
                \end{subfigure}
      \begin{subfigure}[b]{0.999\textwidth}
                    \centering            
                    \includegraphics[width=0.999\textwidth]{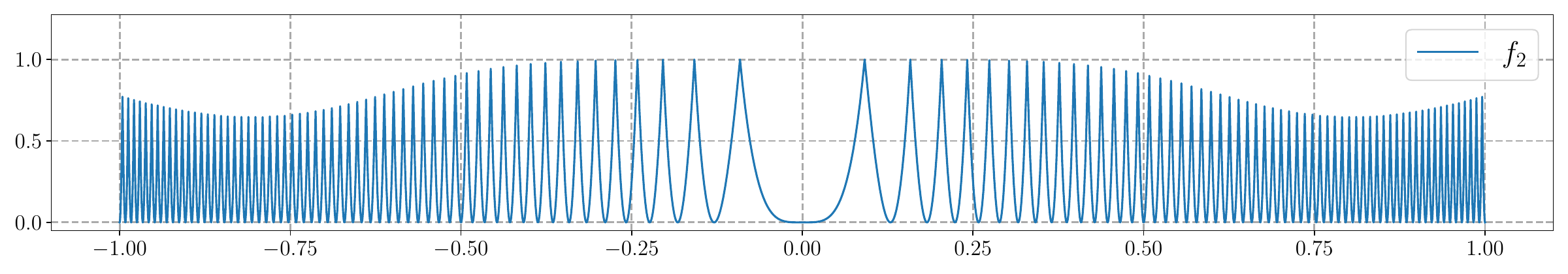}
                \end{subfigure}
                      \begin{subfigure}[b]{0.999\textwidth}
                    \centering            
                    \includegraphics[width=0.999\textwidth]{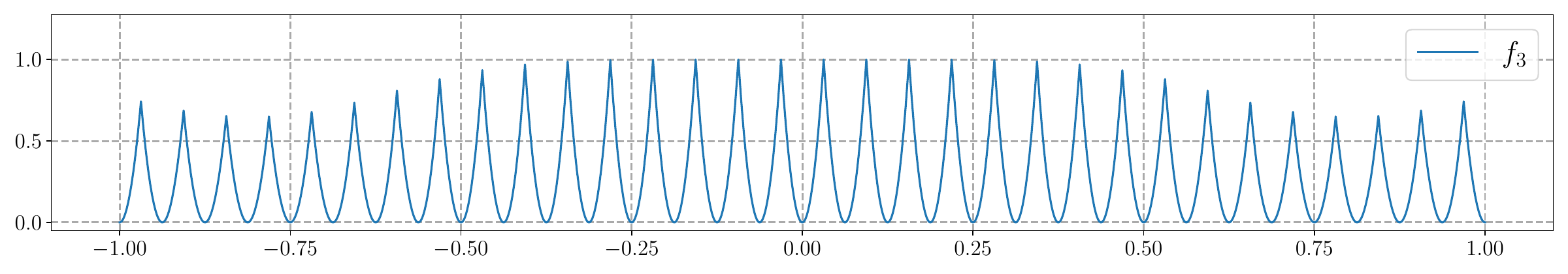}
                \end{subfigure}
        
\caption{Illustrations of $f_1$, $g$, $f_2$, and $f_3$.}
    \label{fig:fig:f1:g:f2:and:zoomin}   
\end{figure}

We employ various network architectures to approximate these functions and evaluate their performance. Notably, these tests are conducted using double precision rather than the default single precision to ensure precise comparisons.
For the test corresponding to \( f_1 \), we use a total of 3000 uniformly sampled points from \( [-1,1] \) for training. The mini-batch size is set to 3000, meaning all samples are trained simultaneously. The learning rate is defined as \( 10^{-3} \times 0.9^{\lfloor k/10000 \rfloor} \), where \( k = 1,2,\dots,1000000 \) denotes the epoch number. 
For the test error evaluation, we use another set of 3,000 samples drawn from the uniform distribution \( \mathcal{U}(-1,1) \).
We emphasize that the mini-batch method is not used for \( f_1 \) because our tests indicate that while mini-batching preserves the approximation of the original function, it leads to poor derivative approximation by only including function values in the cost function. As we can see from Table~\ref{tab:error:comparison:MMNNs:vs:FCNNs}, \sine-activated MMNNs perform best as the target function $f_1$ is smooth. We point out that the errors for derivatives in Table~\ref{tab:error:comparison:MMNNs:vs:FCNNs}  are relative errors, as absolute errors for derivatives can be misleading (the \( L^\infty \)-norm of \( f_1^\dprime \) exceeds 70,000). 
The accurate approximation of derivatives is surprising, given the complexity of the target function and the fact that the cost function is formulated solely based on the function values, meaning that no information about \( f_1^\prime\) or \( f_1^\dprime\) is incorporated into the optimization process.

\begin{table}[ht]
	\centering  
 \setlength{\tabcolsep}{0.68em} 
 \renewcommand{\arraystretch}{1.15}
 \caption{Comparison of test errors. Training is conducted in double precision, with the cost function relying only on function values and excluding derivatives. Relative errors are reported for derivatives, as the \( L^\infty \)-norm of \( f_1^\dprime \) exceeds 70000, making absolute errors misleading.
}
\label{tab:error:comparison:MMNNs:vs:FCNNs}
	\resizebox{0.985\textwidth}{!}{ 
		\begin{tabular}{ccccccccc@{\hspace{10pt}}c} 
			\toprule
             
            \multicolumn{2}{c}
            {\#parameters (trained / all)} &\multicolumn{2}{c}{35235 / \textbf{72993}} &\multicolumn{2}{c}{\textbf{72981} / 151281} &\multicolumn{2}{c}{\textbf{72961} / \textbf{72961}} && \\
            \cmidrule(lr){3-4}
            \cmidrule(lr){5-6}
            \cmidrule(lr){7-8}
        \rowcolor{mygray}   
        & 
        &\multicolumn{2}{c}{MMNN of size (434,16,6)} &\multicolumn{2}{c}{MMNN of size (900,16,6)} &\multicolumn{2}{c}{FCNN of size (120,--,6)} 
        &\multicolumn{2}{c}{\#training-samples}
        \\
            \cmidrule(lr){3-4}
            \cmidrule(lr){5-6}
            \cmidrule(lr){7-8}
            \cmidrule(lr){9-10}
      target function &   {activation}
      &          MSE 
    &  MaxE &    MSE 
    &  MaxE &    MSE 
    &  MaxE &
      mini-batch & all
    \\
    
			\midrule

$f_1$ & $\mathtt{ReLU}$ &  $ 4.47 \times 10^{-5} $  &  $ 6.34 \times 10^{-2} $  &  $ 2.24 \times 10^{-5} $  &  $ 4.06 \times 10^{-2} $  &  $ 2.31 \times 10^{-4} $  &  $ 1.93 \times 10^{-1} $  &  3000 & 3000  
 \\ 

\rowcolor{mygray}$f_1$ & $\mathtt{GELU}$ &  $ 5.54 \times 10^{-5} $  &  $ 5.94 \times 10^{-2} $  &  $ 1.45 \times 10^{-4} $  &  $ 1.00 \times 10^{-1} $  &  $ 1.63 \times 10^{-3} $  &  $ 2.91 \times 10^{-1} $  &  3000 & 3000  
 \\ 

 $f_1$ & $\mathtt{tanh}$ &  $ 4.12 \times 10^{-5} $  &  $ 2.90 \times 10^{-2} $  &  $ 4.91 \times 10^{-6} $  &  $ 1.24 \times 10^{-2} $  &  $ 2.67 \times 10^{-3} $  &  $ 3.69 \times 10^{-1} $  &  3000 & 3000  
 \\ 

\rowcolor{mygray}$f_1$ & $\mathtt{sine}$ &  $ 7.30 \times 10^{-7} $  &  $ 3.69 \times 10^{-3} $  &  $ \bm{3.43 \times 10^{-8}} $  &  $ \bm{8.37 \times 10^{-4}} $  &  $ 2.62 \times 10^{-5} $  &  $ 2.35 \times 10^{-2} $  &  3000 & 3000  
 \\ 

 $f_1$ & $\mathtt{cosine}$ &  $ 6.89 \times 10^{-6} $  &  $ 9.54 \times 10^{-3} $  &  $ 1.35 \times 10^{-7} $  &  $ 1.85 \times 10^{-3} $  &  $ 2.80 \times 10^{-6} $  &  $ 7.76 \times 10^{-3} $  &  3000 & 3000  
 \\ 

\rowcolor{mygray}$f_1$ & $\mathtt{SinTU}_{0}$ &  $ 8.11 \times 10^{-5} $  &  $ 8.65 \times 10^{-2} $  &  $ 4.40 \times 10^{-6} $  &  $ 1.81 \times 10^{-2} $  &  $ 2.15 \times 10^{-4} $  &  $ 1.52 \times 10^{-1} $  &  3000 & 3000  
 \\ 

 $f_1$ & $\mathtt{SinTU}_{-\pi}$ &  $ 1.25 \times 10^{-5} $  &  $ 2.76 \times 10^{-2} $  &  $ 3.65 \times 10^{-7} $  &  $ 4.90 \times 10^{-3} $  &  $ 4.14 \times 10^{-4} $  &  $ 1.44 \times 10^{-1} $  &  3000 & 3000  
 \\ 

\rowcolor{mygray}$f_1$ & $\mathtt{SinTU}_{-2\pi}$ &  $ 4.19 \times 10^{-6} $  &  $ 1.30 \times 10^{-2} $  &  $ 3.97 \times 10^{-7} $  &  $ 4.93 \times 10^{-3} $  &  $ 1.67 \times 10^{-5} $  &  $ 2.63 \times 10^{-2} $  &  3000 & 3000  
 \\

\midrule

  $f^\prime_1$ & $\mathtt{GELU}$ &  $ 1.24 \times 10^{-3} $  &  $ 2.61 \times 10^{-1} $  &  $ 2.34 \times 10^{-3} $  &  $ 2.97 \times 10^{-1} $  &  $ 7.74 \times 10^{-3} $  &  $ 5.16 \times 10^{-1} $  &    &   
 \\ 

\rowcolor{mygray}$f^\prime_1$ & $\mathtt{tanh}$ &  $ 7.50 \times 10^{-4} $  &  $ 1.11 \times 10^{-1} $  &  $ 1.24 \times 10^{-4} $  &  $ 5.67 \times 10^{-2} $  &  $ 1.10 \times 10^{-2} $  &  $ 6.31 \times 10^{-1} $  &    &   
 \\ 

 $f^\prime_1$ & $\mathtt{sine}$ &  $ 2.64 \times 10^{-5} $  &  $ 2.12 \times 10^{-2} $  &  $\bm{ 1.27 \times 10^{-6} }$  &  $ \bm{5.12 \times 10^{-3}} $  &  $ 5.83 \times 10^{-4} $  &  $ 1.06 \times 10^{-1} $  &    &   
 \\ 

\rowcolor{mygray}$f^\prime_1$ & $\mathtt{cosine}$ &  $ 1.77 \times 10^{-4} $  &  $ 4.76 \times 10^{-2} $  &  $ 5.35 \times 10^{-6} $  &  $ 1.15 \times 10^{-2} $  &  $ 8.91 \times 10^{-5} $  &  $ 3.84 \times 10^{-2} $  &    &   
 \\ 
\midrule

 $f^\dprime_1$ & $\mathtt{GELU}$ &  $ 1.03 \times 10^{-2} $  &  $ 1.22 \times 10^{0} $  &  $ 1.55 \times 10^{-2} $  &  $ 9.39 \times 10^{-1} $  &  $ 3.10 \times 10^{-2} $  &  $ 1.39 \times 10^{0} $  &    &   
 \\ 

\rowcolor{mygray}$f^\dprime_1$ & $\mathtt{tanh}$ &  $ 5.40 \times 10^{-3} $  &  $ 2.79 \times 10^{-1} $  &  $ 9.92 \times 10^{-4} $  &  $ 1.59 \times 10^{-1} $  &  $ 3.02 \times 10^{-2} $  &  $ 8.62 \times 10^{-1} $  &    &   
 \\ 

 $f^\dprime_1$ & $\mathtt{sine}$ &  $ 2.32 \times 10^{-4} $  &  $ 7.40 \times 10^{-2} $  &  $ \bm{7.82 \times 10^{-6} }$  &  $ \bm{1.37 \times 10^{-2} }$  &  $ 4.45 \times 10^{-3} $  &  $ 2.71 \times 10^{-1} $  &    &   
 \\ 

\rowcolor{mygray}$f^\dprime_1$ & $\mathtt{cosine}$ &  $ 1.48 \times 10^{-3} $  &  $ 1.30 \times 10^{-1} $  &  $ 4.12 \times 10^{-5} $  &  $ 5.98 \times 10^{-2} $  &  $ 7.87 \times 10^{-4} $  &  $ 1.15 \times 10^{-1} $  &    &   
 \\

 \midrule

$f_2$ & $\mathtt{ReLU}$ &  $ 4.15 \times 10^{-3} $  &  $ 4.87 \times 10^{-1} $  &  $ 1.53 \times 10^{-3} $  &  $ 4.37 \times 10^{-1} $  &  $ 2.13 \times 10^{-2} $  &  $ 6.35 \times 10^{-1} $  &  3000 & 60000  
 \\ 

\rowcolor{mygray}$f_2$ & $\mathtt{GELU}$ &  $ 4.15 \times 10^{-3} $  &  $ 4.72 \times 10^{-1} $  &  $ 1.24 \times 10^{-3} $  &  $ 4.15 \times 10^{-1} $  &  $ 1.33 \times 10^{-2} $  &  $ 5.20 \times 10^{-1} $  &  3000 & 60000  
 \\ 

 $f_2$ & $\mathtt{tanh}$ &  $ 4.19 \times 10^{-3} $  &  $ 5.16 \times 10^{-1} $  &  $ 1.42 \times 10^{-4} $  &  $ 1.51 \times 10^{-1} $  &  $ 1.21 \times 10^{-2} $  &  $ 5.56 \times 10^{-1} $  &  3000 & 60000  
 \\ 

\rowcolor{mygray}$f_2$ & $\mathtt{sine}$ &  $ 4.30 \times 10^{-5} $  &  $ 7.36 \times 10^{-2} $  &  $ 4.84 \times 10^{-6} $  &  $ 2.88 \times 10^{-2} $  &  $ 9.07 \times 10^{-5} $  &  $ 9.22 \times 10^{-2} $  &  3000 & 60000  
 \\ 

 $f_2$ & $\mathtt{cosine}$ &  $ 3.17 \times 10^{-4} $  &  $ 1.28 \times 10^{-1} $  &  $ 5.65 \times 10^{-6} $  &  $ 3.15 \times 10^{-2} $  &  $ 5.68 \times 10^{-5} $  &  $ 7.86 \times 10^{-2} $  &  3000 & 60000  
 \\ 

\rowcolor{mygray}$f_2$ & $\mathtt{SinTU}_{0}$ &  $ 2.10 \times 10^{-3} $  &  $ 4.61 \times 10^{-1} $  &  $ 2.51 \times 10^{-6} $  &  $ 2.61 \times 10^{-2} $  &  $ 2.00 \times 10^{-2} $  &  $ 6.19 \times 10^{-1} $  &  3000 & 60000  
 \\ 

 $f_2$ & $\mathtt{SinTU}_{-\pi}$ &  $ 3.61 \times 10^{-5} $  &  $ 7.42 \times 10^{-2} $  &  $ \bm{1.28 \times 10^{-6} }$  &  $\bm{ 2.31 \times 10^{-2} }$  &  $ 6.14 \times 10^{-3} $  &  $ 5.31 \times 10^{-1} $  &  3000 & 60000  
 \\ 

\rowcolor{mygray}$f_2$ & $\mathtt{SinTU}_{-2\pi}$ &  $ 3.46 \times 10^{-5} $  &  $ 7.03 \times 10^{-2} $  &  $ 5.04 \times 10^{-6} $  &  $ 3.05 \times 10^{-2} $  &  $ 3.05 \times 10^{-3} $  &  $ 3.47 \times 10^{-1} $  &  3000 & 60000  
 \\

\midrule

$f_3$ & $\mathtt{sine}$ &  $ 1.52 \times 10^{-7} $  &  $ 7.86 \times 10^{-3} $  &  $ 7.68 \times 10^{-8} $  &  $ 6.06 \times 10^{-3} $  &  $ 3.04 \times 10^{-2} $  &  $ 6.99 \times 10^{-1} $  & 500 & 18000  
 \\ 

\rowcolor{mygray} $f_3$ & $\mathtt{sine}$ &  $ 1.11 \times 10^{-6} $  &  $ 2.23 \times 10^{-2} $  &  $ 1.27 \times 10^{-7} $  &  $ 6.58 \times 10^{-3} $  &  $ 6.69 \times 10^{-2} $  &  $ 6.68 \times 10^{-1} $  & 1000 & 18000  
 \\ 

  $f_3$ & $\mathtt{sine}$ &  $ 3.52 \times 10^{-7} $  &  $ 1.07 \times 10^{-2} $  &  $ 8.75 \times 10^{-8} $  &  $ 6.28 \times 10^{-3} $  &  $ 2.43 \times 10^{-4} $  &  $ 1.52 \times 10^{-1} $  & 1500 & 18000  
 \\ 

\rowcolor{mygray} $f_3$ & $\mathtt{sine}$ &  $ 5.11 \times 10^{-7} $  &  $ 1.10 \times 10^{-2} $  &  $ \bm{6.41 \times 10^{-8}} $  &  $\bm{ 5.33 \times 10^{-3} }$  &  $ 5.86 \times 10^{-6} $  &  $ 3.38 \times 10^{-2} $  & 2000 & 18000  
 \\ 
 \bottomrule
		\end{tabular} 
	}
\end{table}

\begin{figure}[ht]
            \centering
            \begin{subfigure}[b]{0.49\textwidth}
                    \centering            
                    \includegraphics[width=0.999\textwidth]{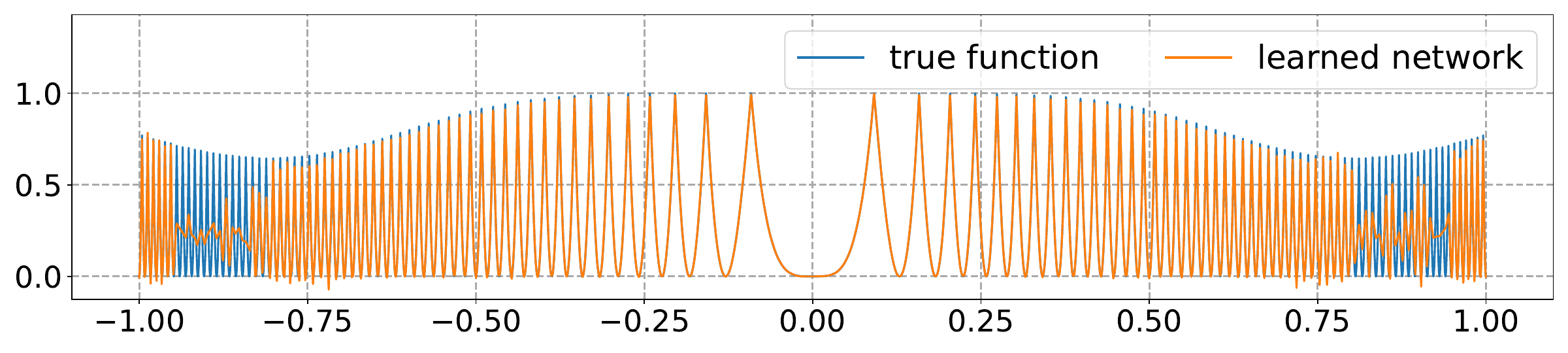}
                    \subcaption{\ReLU{} MMNN of size $(434,16,6)$. 
                    }
                \end{subfigure}
                \hfill
                \begin{subfigure}[b]{0.49\textwidth}
                    \centering   \includegraphics[width=0.999\textwidth]{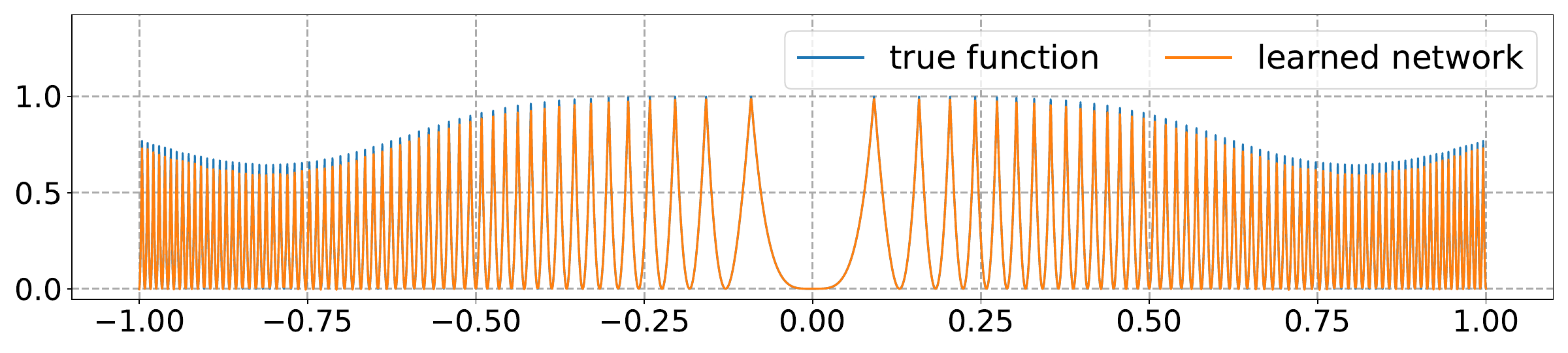}
                    \subcaption{\Sine{} MMNN of size $(434,16,6)$. 
                    }
                \end{subfigure}
                \\
            \begin{subfigure}[b]{0.490\textwidth}
                    \centering            \includegraphics[width=0.999\textwidth]{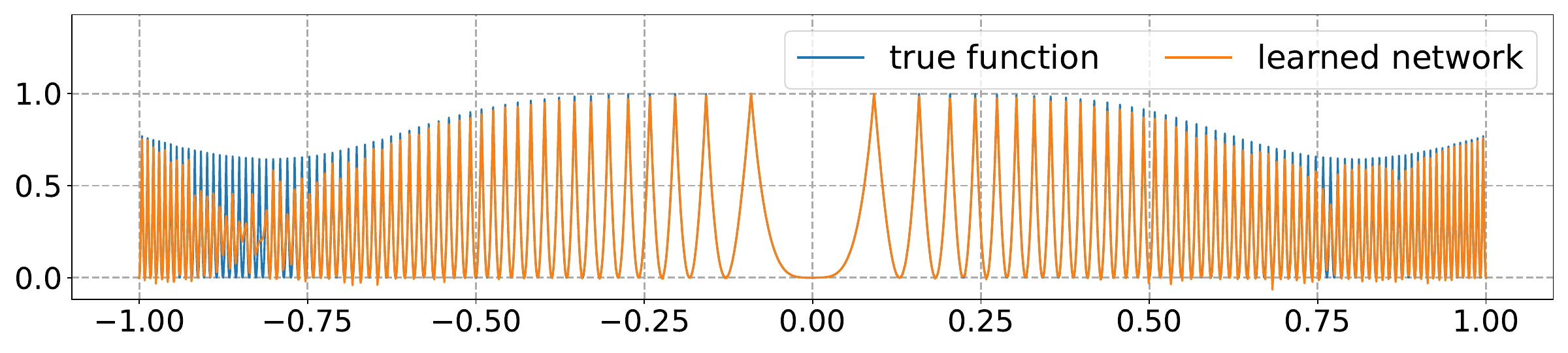}
                    \subcaption{\ReLU{} MMNN of size $(900,16,6)$.
                    }
                \end{subfigure}
                \hfill
                            \begin{subfigure}[b]{0.490\textwidth}
                    \centering            \includegraphics[width=0.999\textwidth]{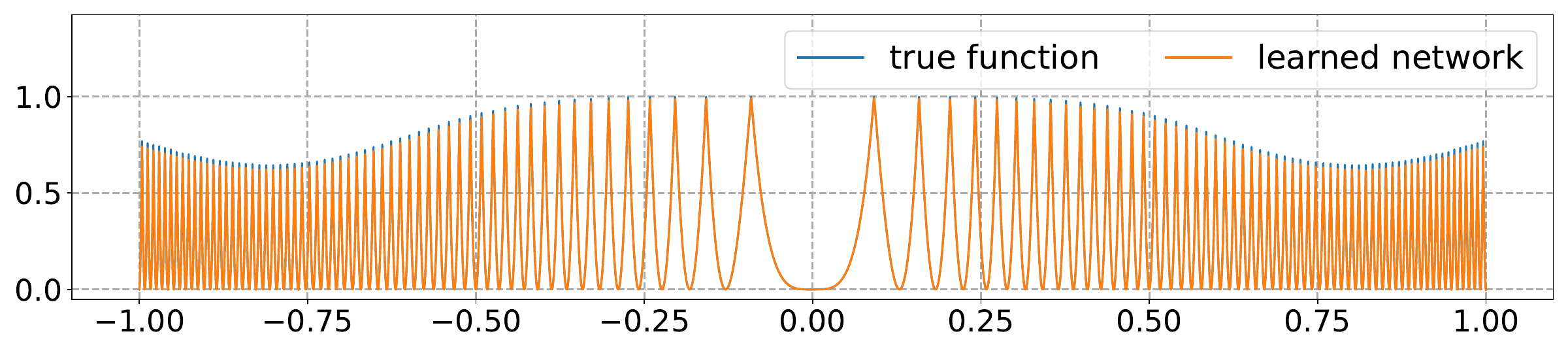}
                    \subcaption{\Sine{} MMNN of size $(900,16,6)$.
                    }
                \end{subfigure}
                \\
            \begin{subfigure}[b]{0.490\textwidth}
                    \centering            \includegraphics[width=0.999\textwidth]{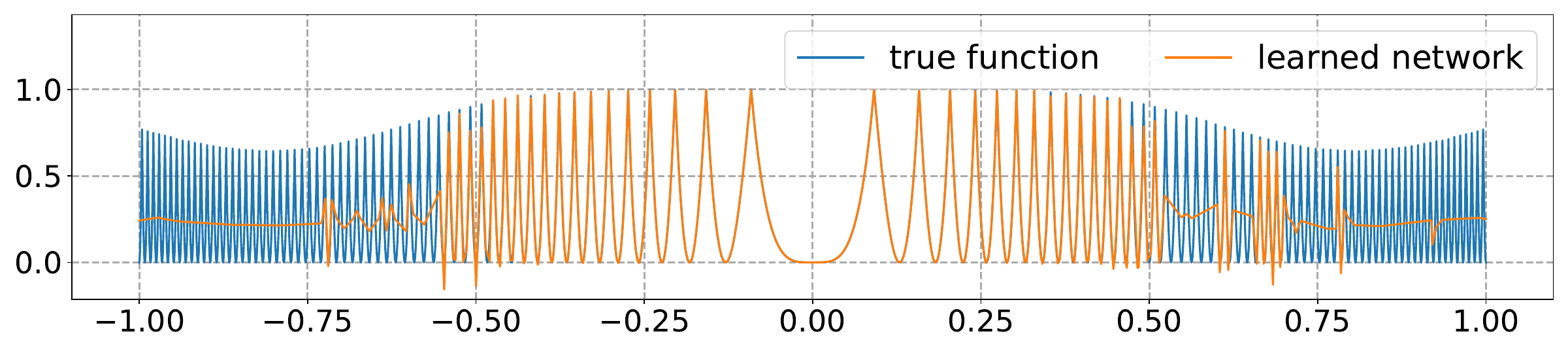}
                    \subcaption{\ReLU{} FCNN of size $(120,-,6)$. 
                    }
                \end{subfigure}
                \hfill
                \begin{subfigure}[b]{0.490\textwidth}
                    \centering            \includegraphics[width=0.999\textwidth]{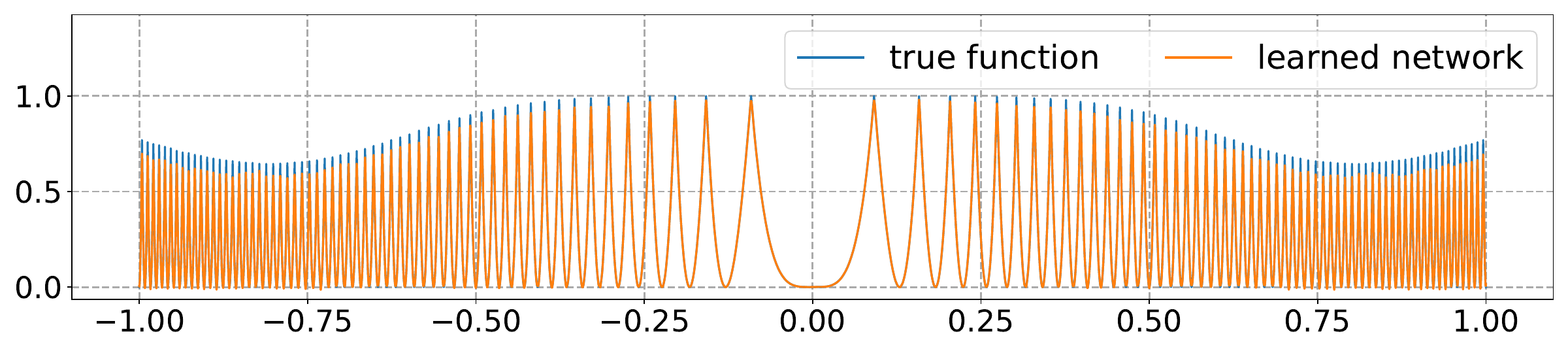}
                    \subcaption{\Sine{} FCNN of size $(120,-,6)$. 
                    }
                \end{subfigure}
                \caption{Illustrations of the true function $f_2$ and learned networks.
                }
    \label{fig:MMNNvsFCNN:123}   
\end{figure}

For the test corresponding to \( f_2 \), we use a total of 60000 uniformly sampled points from \( [-1,1] \) for training. The mini-batch size is set to 3000, and the learning rate is defined as \( 10^{-3} \times 0.9^{\lfloor k/500 \rfloor} \), where \( k = 1,2,\dots,50000 \) denotes the epoch number. 
To ensure accurate computation of the test error, we select another set of 60000 test samples from the uniform distribution \( \mathcal{U}(-1,1) \). 
As illustrated in Figure~\ref{fig:MMNNvsFCNN:123}, the MMNN architecture exhibits superior efficiency relative to the FCNN. Additionally, the \sine{} activation function proves more effective than \ReLU{} for approximating the complex target function \( f_2 \). It is worth noting that our training process involved a sufficiently large number of iterations, effectively eliminating the possibility of inadequate training as a contributing factor. Moreover, expanding the size of the MMNN would further substantially improve its performance.
As shown in Table~\ref{tab:error:comparison:MMNNs:vs:FCNNs}, MMNNs generally outperform FCNNs, regardless of the activation function. This advantage may stem from the simpler optimization landscape of MMNNs, which enables more efficient training.
Furthermore, Table~\ref{tab:error:comparison:MMNNs:vs:FCNNs} highlights that \SinTU{s} achieve the best performance, which is expected since \( f_2 \) contains many singularities that \SinTU{s} are well-suited to handle. Notably, even in this inherently unfavorable setting for \sine-based models, \sine{}-activated MMNNs still perform well. Thus, when the properties of the target function are uncertain in practical applications, trying the \sine{} activation function first is a reasonable strategy.
When the mini-batch size is relatively large, the number of training epochs tends to be high, making the process time-consuming. More importantly, even when \sine-activated FCNNs are given sufficiently large mini-batches and a sufficient number of training epochs, their final performance still falls short of \sine-activated MMNNs.

For the test corresponding to \( f_3 \), we use a total of 18000 uniformly sampled points from \( [-1,1] \) for training. The mini-batch size is set to \( n_{\tn{mbs}} \in \{500, 1000, 1500, 2000\} \), and the learning rate is defined as \( 10^{-3} \times 0.9^{\lfloor 2k/n_{\tn{mbs}} \rfloor} \), where \( k = 1,2,\dots,50n_{\tn{mbs}} \) denotes the epoch number. 
To ensure accurate computation of the test error, we select another set of 18000 test samples from the uniform distribution \( \mathcal{U}(-1,1) \).
As shown in Table~\ref{tab:error:comparison:MMNNs:vs:FCNNs}, FCNNs are relatively sensitive to the hyper-parameters of training. If the mini-batch size is too small, training may fail. However, MMNNs are more stable and succeed under various settings.

\subsection{MMNNs: \texttt{Sine} Versus Other Activation Functions}
\label{sec:sine_vs_others}

In this section, we compare the performance of MMNNs using different activation functions to demonstrate that FMMNNs consistently produce the best results.  The three target functions used in the tests are $f_4:[-1,1]\to\R$, $f_5:[-1,1]^2\to\R$, and $f_6:[-1,1]^3\to\R$, which are given by
\[
f_4(x) =
0.6\sin(200\pi x)+0.8\cos(160\pi x^2)
+ \frac{1 + 8x^8}{1 + 10x^4} \cdot \left| 180x - 2 \left\lfloor \frac{180x + 1}{2} \right\rfloor \right|,
\]
\[f_5(x_1, x_2) = \sum_{i=1}^2 \sum_{j=1}^2 a_{ij} \sin(b_i x_i +  c_{ij} x_i x_j)\cdot  \big|\cos( b_j x_j +  d_{ij} x_i^2)\big|,\]
and
$$
f_6(x_1, x_2,x_3) = \sum_{i=1}^3 \sum_{j=1}^3 \tildea_{ij} \sin(\tildeb_i x_i + \tildec_{ij} x_i x_j) \cdot\big|\cos(\tildeb_j x_j + \tilded_{ij} x_i^2)\big|,
$$
where  
$$
(a_{ij}) = \begin{bmatrix*} 0.3 & 0.2 \\ 0.2 & 0.3 \end{bmatrix*}, 
\
(b_i) = \begin{bmatrix*} 12\pi \\  8\pi \end{bmatrix*}, 
\
(c_{ij}) = \begin{bmatrix*} 4\pi & 18\pi \\ 16\pi & 10\pi \end{bmatrix*}, 
\
(d_{ij}) = \begin{bmatrix*} 14\pi & 12\pi \\ 18\pi & 10\pi \end{bmatrix*},
$$
$$
(\tildea_{ij}) = \begin{bmatrix*} 
0.3 & 0.1 & 0.4 \\
0.2 & 0.3 & 0.1 \\
0.2 & 0.1 & 0.3
\end{bmatrix*}, 
\
(\tildeb_i) = \begin{bmatrix*} 
\pi \\ 4\pi \\ 3\pi
\end{bmatrix*}, 
\
(\tildec_{ij}) = \begin{bmatrix*} 
2\pi & \pi & 3\pi \\
2\pi & 3\pi & 2\pi \\
3\pi & \pi & \pi
\end{bmatrix*}, 
\ \tn{and} \
(\tilded_{ij}) = \begin{bmatrix*} 
2\pi & 3\pi &  \pi \\
 \pi & 3\pi & 2\pi \\
 \pi & 2\pi & 3\pi
\end{bmatrix*}.
$$
Note that all three functions are only continuous but not differentiable. Illustrations of these three functions are shown in Figure~\ref{fig:f:123}.

\begin{figure}[ht]
            \centering
            \,\hfill
            \begin{subfigure}[b]{0.3227300245\textwidth}
                    \centering            
                    \includegraphics[width=0.8999\textwidth]{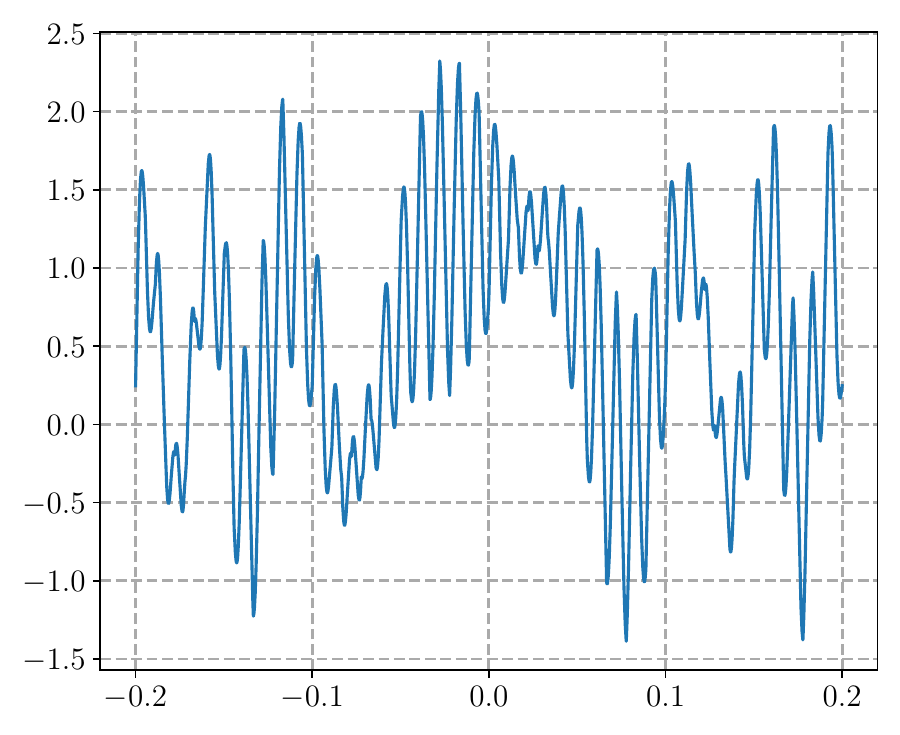}
                    \subcaption{$f_4$ limited on $[-0.2,0.2]$.}
                \end{subfigure}
                \hfill
            \begin{subfigure}[b]{0.3227300245\textwidth}
                    \centering            \includegraphics[width=0.8999\textwidth]{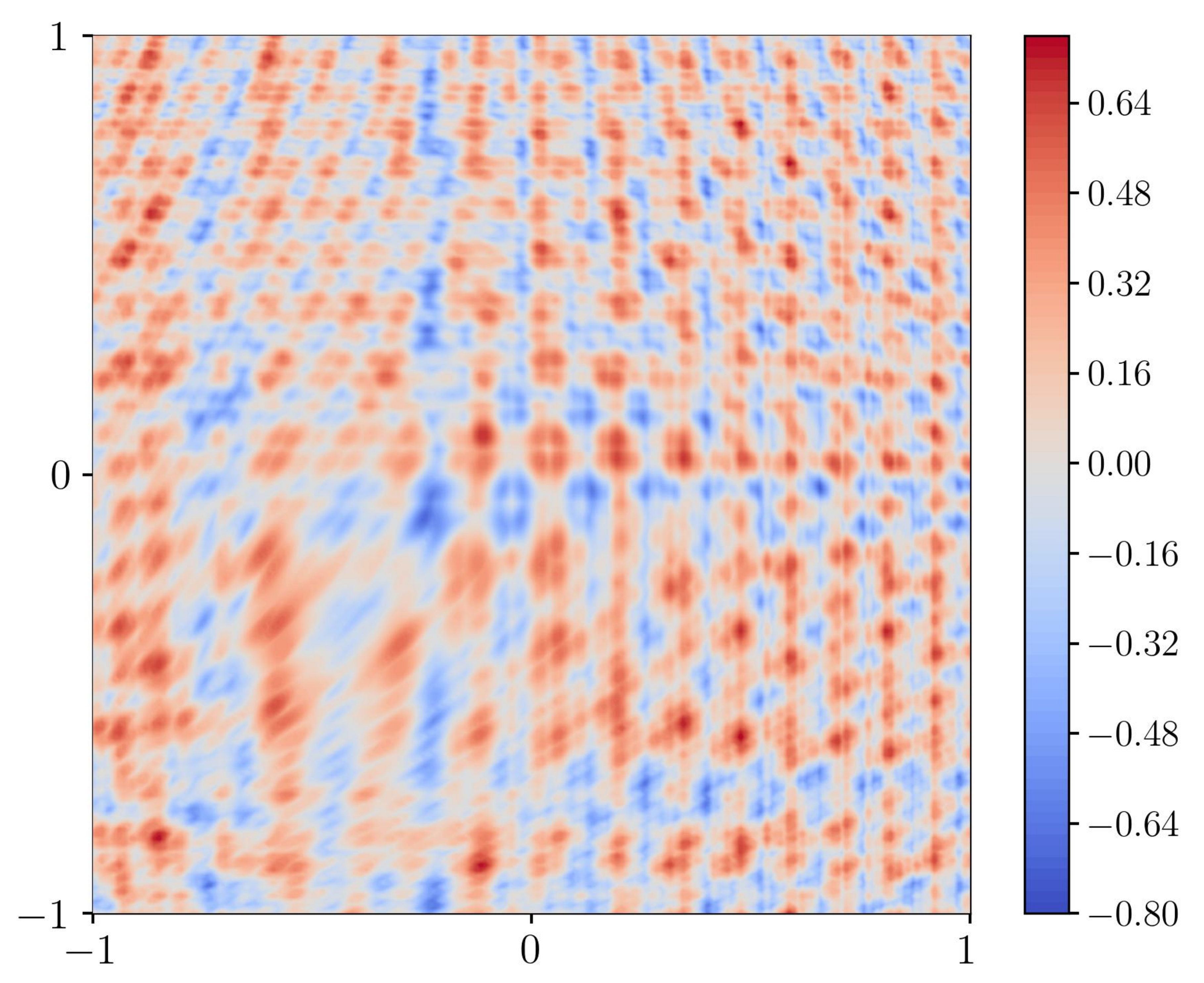}
                    \subcaption{$f_5$.}
                \end{subfigure}
                \hfill
            \begin{subfigure}[b]{0.3227300245\textwidth}
                    \centering            \includegraphics[width=0.8999\textwidth]{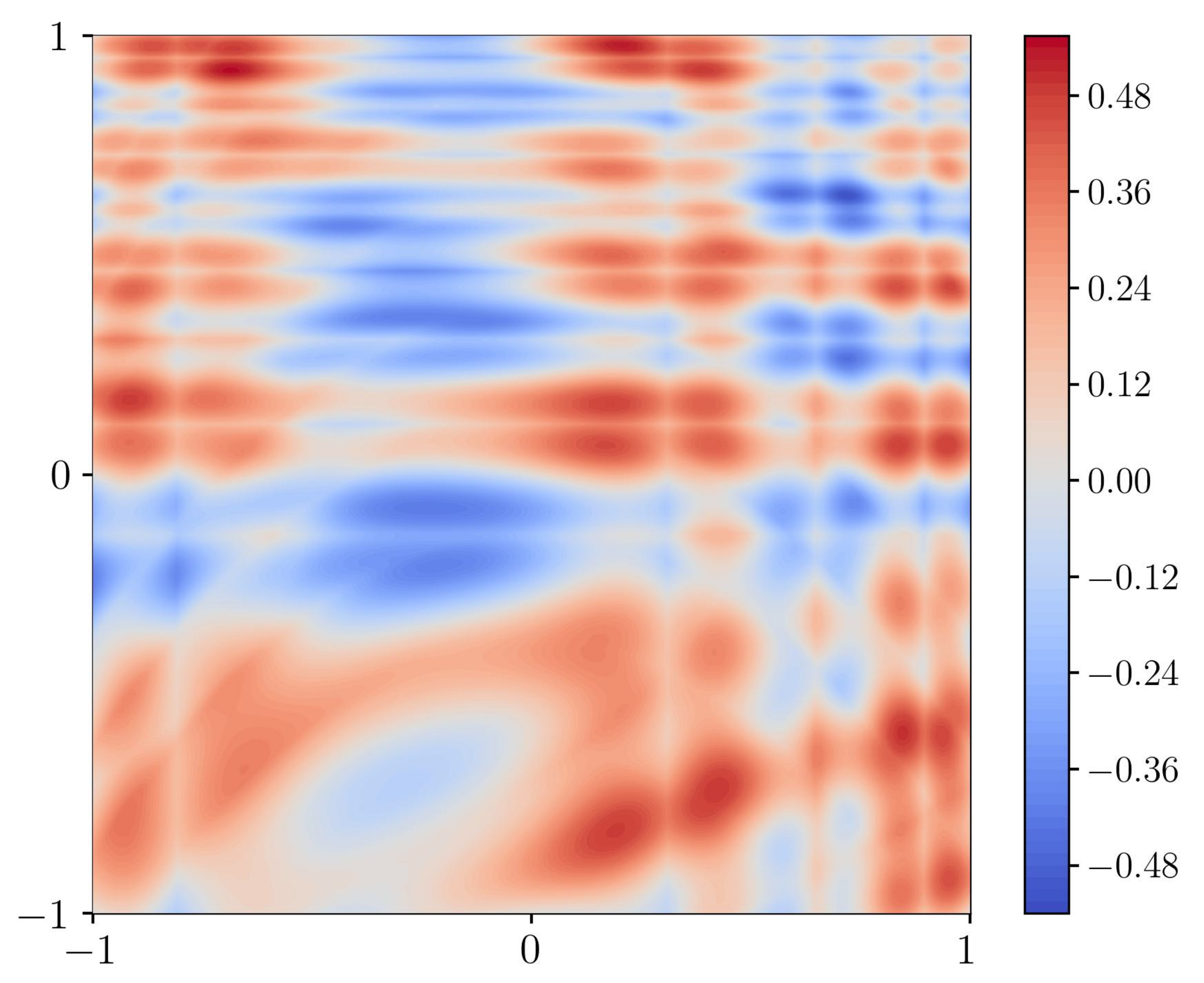}
                    \subcaption{$f_6(x,y,z=0)$.}
                \end{subfigure}
\hfill\, \\[5pt]
       \,\hfill
            \begin{subfigure}[b]{0.3227300245\textwidth}
                    \centering            
                    \includegraphics[width=0.8999\textwidth]{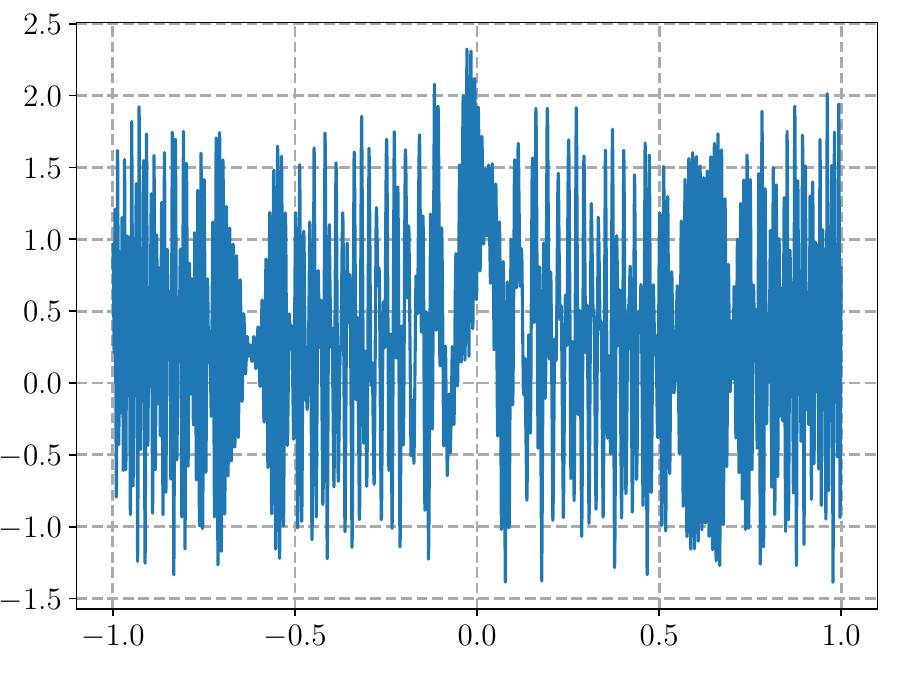}
                    \subcaption{$f_4$.}
                \end{subfigure}
                \hfill
            \begin{subfigure}[b]{0.3227300245\textwidth}
                    \centering            \includegraphics[width=0.999\textwidth]{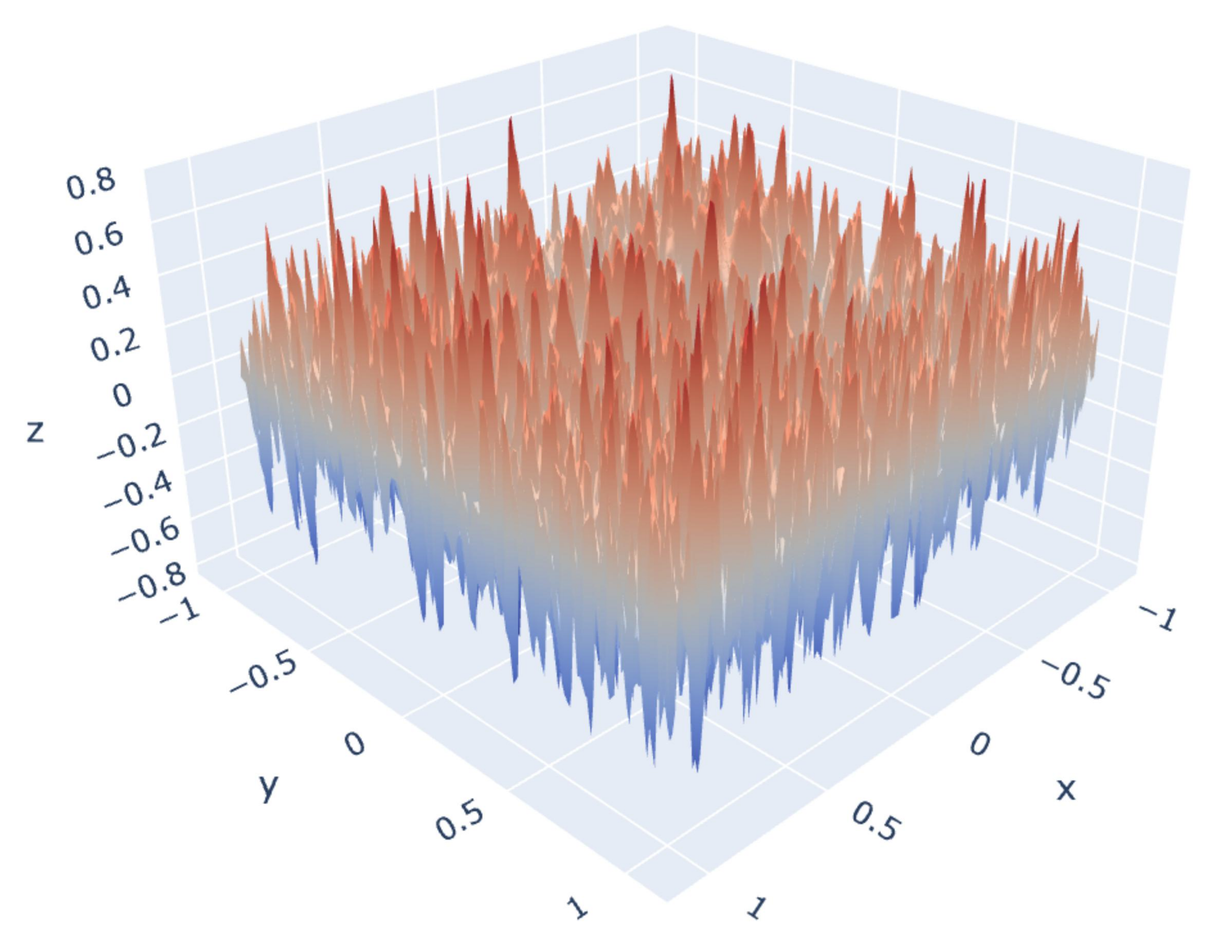}
                    \subcaption{$z=f_5(x,y)$.}
                \end{subfigure}
                \hfill
            \begin{subfigure}[b]{0.3227300245\textwidth}
                    \centering            \includegraphics[width=0.999\textwidth]{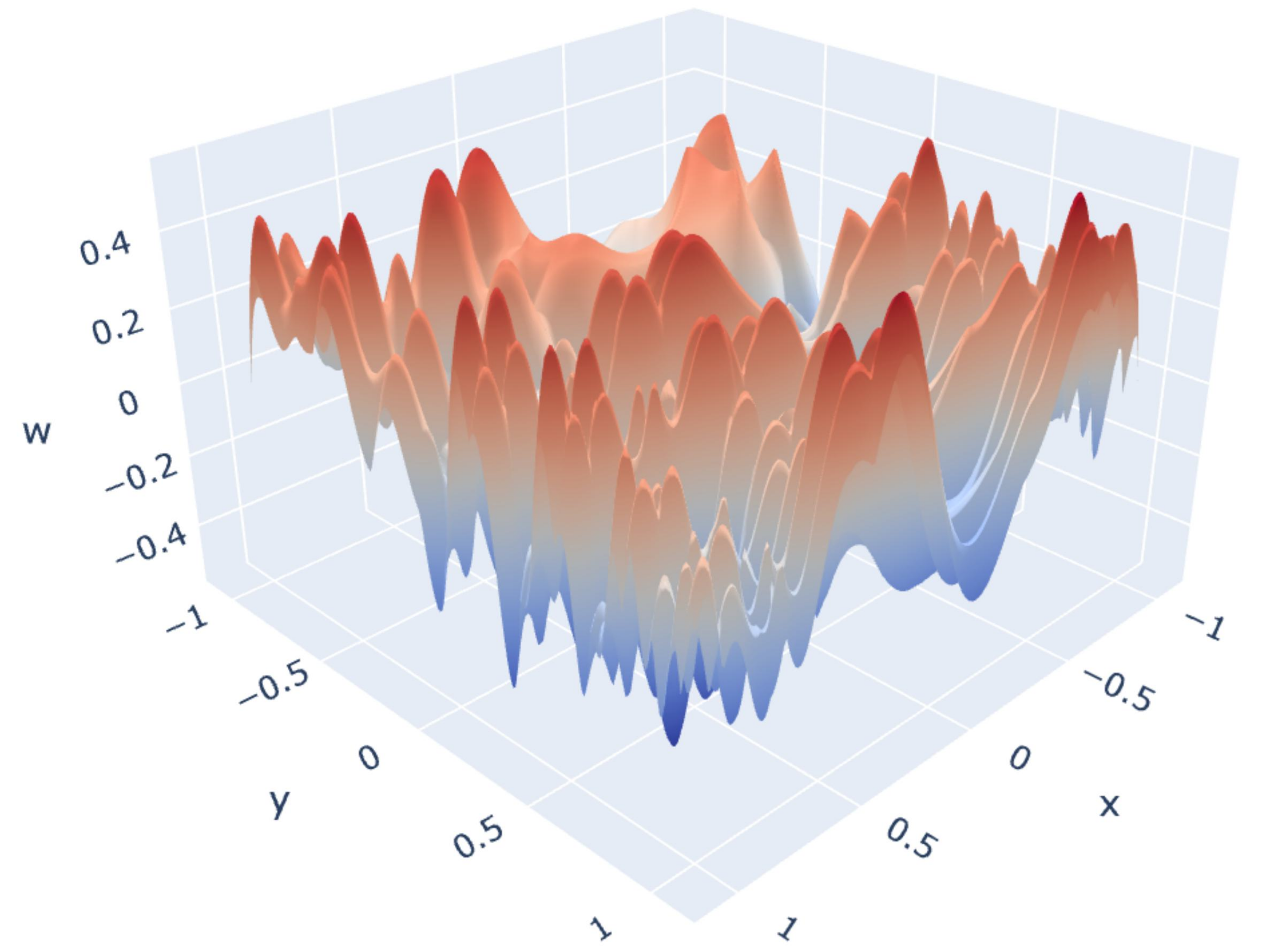}
                    \subcaption{$w=f_6(x,y,z=0)$.}
                \end{subfigure}
\hfill\,
\caption{Illustrations of $f_i$ for 
$i=4,5,6.$
}

    \label{fig:f:123}
    
\end{figure}


We employ MMNN structures with different activation functions to approximate the target functions and evaluate their performance.
For the one-dimensional case, we use 60000 uniformly sampled points from \([-1, 1]\) for training, with a mini-batch size of 600 and a learning rate defined as \( 10^{-3} \times 0.9^{\lfloor k/100 \rfloor} \), where \( k = 1, 2, \ldots, 10000 \) represents the epoch number. To ensure accurate computation of the test error, we select 60000 test samples from the uniform distribution in \( [-1,1] \).
In the two-dimensional case, \( 600^2 \) uniformly sampled points from \( [-1,1]^2 \) are used for training, with a mini-batch size of 1200 and a learning rate set to \( 10^{-3} \times 0.9^{\lfloor k/30 \rfloor} \), where \( k = 1, 2, \ldots, 3000 \). For test error evaluation, we select \( 300^2 \) samples from the uniform distribution in \( [-1,1]^2 \).
For the three-dimensional case, \( 150^3 \) points from \( [-1,1]^3 \) are uniformly sampled for training, with a mini-batch size of 1500 and a learning rate defined as \( 10^{-3} \times 0.9^{\lfloor k/4 \rfloor} \), where \( k = 1, 2, \ldots, 400 \). To ensure accurate computation of the test error, we select \( 100^3 \) samples from the uniform distribution in \( [-1,1]^3 \).





\begin{table}[ht]
	\centering  
 \setlength{\tabcolsep}{0.68em} 
 \renewcommand{\arraystretch}{1.15}
\caption{Comparison of test errors. Training is performed in single precision.}
	\label{tab:error:comparison:MMNNs}
	\resizebox{0.8292\textwidth}{!}{ 
		\begin{tabular}{ccccccccc} 
			\toprule
            {
            target function
            }
            & \multicolumn{2}{c}{$f_4:[-1,1]\to \R$}
            & \multicolumn{2}{c}{$f_5:[-1,1]^2\to \R$}
            & \multicolumn{2}{c}{$f_6:[-1,1]^3\to \R$}\\        
            \cmidrule(lr){2-3}
            \cmidrule(lr){4-5}
            \cmidrule(lr){6-7}
        \rowcolor{mygray}   
        &\multicolumn{2}{c}{MMNN of size (1024,16,6)} &\multicolumn{2}{c}{MMNN of size (1024,36,8)} &\multicolumn{2}{c}{ResMMNN of size (1024,36,10)}
        
        \\
            \cmidrule(lr){2-3}
            \cmidrule(lr){4-5}
            \cmidrule(lr){6-7}
        {activation} &         MSE 
    &  MaxE &    MSE 
    &  MaxE &    MSE 
    &  MaxE 
    \\
    
			\midrule

$\mathtt{ReLU}$ &  $ 3.52 \times 10^{-2} $  &  $ 1.57 \times 10^{0} $  &  $ 6.50 \times 10^{-5} $  &  $ 6.99 \times 10^{-2} $  &  $ 8.18 \times 10^{-5} $  &  $ 7.86 \times 10^{-2} $ 
 \\ 

\rowcolor{mygray}$\mathtt{ELU}$ &  $ 1.62 \times 10^{-1} $  &  $ 1.78 \times 10^{0} $  &  $ 2.43 \times 10^{-3} $  &  $ 2.68 \times 10^{-1} $  &  $ 6.70 \times 10^{-5} $  &  $ 7.03 \times 10^{-2} $ 
 \\ 

 $\mathtt{GELU}$ &  $ 1.51 \times 10^{-1} $  &  $ 1.61 \times 10^{0} $  &  $ 6.19 \times 10^{-5} $  &  $ 6.67 \times 10^{-2} $  &  $ 6.66 \times 10^{-5} $  &  $ 6.14 \times 10^{-2} $ 
 \\ 

\rowcolor{mygray}$\mathtt{sigmoid}$ &  $ 5.68 \times 10^{-1} $  &  $ 1.95 \times 10^{0} $  &  $ 4.97 \times 10^{-2} $  &  $ 8.00 \times 10^{-1} $  &  $ 1.04 \times 10^{-3} $  &  $ 1.91 \times 10^{-1} $ 
 \\ 

 $\mathtt{tanh}$ &  $ 1.84 \times 10^{-1} $  &  $ 1.77 \times 10^{0} $  &  $ 9.76 \times 10^{-4} $  &  $ 2.33 \times 10^{-1} $  &  $ 1.06 \times 10^{-4} $  &  $ 1.02 \times 10^{-1} $ 
 \\ 

\rowcolor{mygray}$\mathtt{sine}$ &  $ 1.16 \times 10^{-5} $  &  $ 3.18 \times 10^{-2} $  &  $ 2.26 \times 10^{-5} $  &  $ 5.18 \times 10^{-2} $  &  $ \bm{2.67 \times 10^{-5}} $  &  $ 5.22 \times 10^{-2} $ 
 \\ 

 $\mathtt{cosine}$ &  $ 1.55 \times 10^{-5} $  &  $ 3.50 \times 10^{-2} $  &  $ 2.46 \times 10^{-5} $  &  $ 4.26 \times 10^{-2} $  &  $ 3.06 \times 10^{-5} $  &  $ 5.74 \times 10^{-2} $ 
 \\ 


\rowcolor{mygray}  $\mathtt{SinTU}_{0}$ &  $ 2.14 \times 10^{-6} $  &  $ 3.04 \times 10^{-2} $  &  $ 3.18 \times 10^{-5} $  &  $ 5.18 \times 10^{-2} $  &  $ 4.21 \times 10^{-5} $  &  $ 6.17 \times 10^{-2} $ 
 \\ 

$\mathtt{SinTU}_{-\pi}$ &  $ \bm{1.72 \times 10^{-6} } $  &  $ \bm{2.80 \times 10^{-2}} $  &  $ 2.27 \times 10^{-5} $  &  $ 5.38 \times 10^{-2} $  &  $ 3.27 \times 10^{-5} $  &  $ 5.56 \times 10^{-2} $ 
 \\ 
\rowcolor{mygray}
 $\mathtt{SinTU}_{-2\pi}$ &  $ 5.87 \times 10^{-6} $  &  $ 3.17 \times 10^{-2} $  &  $ 1.62 \times 10^{-5} $  &  $ \bm{3.90 \times 10^{-2}} $  &  $ 2.85 \times 10^{-5} $  &  $ 5.27 \times 10^{-2} $ 
 \\ 

$\mathtt{SinTU}_{-4\pi}$ &  $ 1.21 \times 10^{-5} $  &  $ 3.17 \times 10^{-2} $  &  $ \bm{ 1.52 \times 10^{-5}} $  &  $ 4.73 \times 10^{-2} $  &  $ 2.68 \times 10^{-5} $  &  $ 5.43 \times 10^{-2} $ 
 \\ 
\rowcolor{mygray}
 $\mathtt{SinTU}_{-8\pi}$ &  $ 9.47 \times 10^{-6} $  &  $ 3.03 \times 10^{-2} $  &  $ 1.88 \times 10^{-5} $  &  $ 4.53 \times 10^{-2} $  &  $ 2.74 \times 10^{-5} $  &  $ \bm{5.16 \times 10^{-2}} $ 
 \\ 



 \bottomrule
		\end{tabular} 
	}
\end{table} 


As shown in Table~\ref{tab:error:comparison:MMNNs}, \texttt{sine} and \SinTU{s} are the most effective activation functions for MMNNs. Our results further confirm that the combination of sinusoidal activation functions and MMNN structures is particularly well-suited for function approximation.  
\begin{colorenv}[blue]
Table~\ref{tab:error:comparison:MMNNs} also illustrates the effect of the
truncation parameter \(s\) in \(\SinTU{s}\). Since
\[
    \SinTU_{s}(x)=\sin(\max\{x,s\}),
\]
the activation coincides with \(\sin(x)\) on the active region \(x\ge s\), while
it is constant on the truncated region \(x<s\). Thus, \(s\) controls the balance
between Fourier-like oscillation and a \ReLU-like singular feature. When \(s\) is
very negative, \(\SinTU{s}\) behaves close to the pure sine activation on most
preactivation values; when \(s\) is larger, more neurons may enter the flat
region, which can introduce localized nonsmoothness but may also reduce the
effective number of oscillatory neurons and slow training.

The results suggest that the best choice of \(s\) depends on the target
regularity. For smoother targets, pure \(\sine\) or a more negative truncation
level is often preferable, since the extra singularity is less useful and may
slightly affect stability or derivative accuracy. For nonsmooth targets, an
intermediate truncation level can be beneficial because it combines oscillatory
representation with localized nonsmooth features. Therefore, \(\SinTU{s}\) is
not intended to uniformly dominate \(\sine\); rather, it provides an additional
degree of freedom for matching the activation to the regularity of the target
function.
\end{colorenv}
\subsection{Training Acceleration via Scaled Weight Initialization}
\label{sec:scaling:init}

The initialization of neural network parameters plays a crucial role in determining the efficiency and effectiveness of training. 
Traditional initialization methods
focus on maintaining stable gradients but do not explicitly consider the frequency characteristics of the target function. In our previous study on frequency bias of two-layer networks~\cite{ZZZZ-23}, we showed that one can reduce the frequency bias and enhance the representation ability of the network by scaling the initial random slope weights inside the activation function, which is equivalent to the introduction of high-frequency components to the network representation initially. We apply this initial scaling strategy to the first layer of FMMNNs so that finer features in the samples can be captured by the first layer. We keep the standard initialization in all other layers to avoid instability due to fast multilayer amplification in training. Specifically, we first initialize all parameters by PyTorch's default initialization, and then 
scale $\bmW_1,\,\bmb_1$ by $\sqrt{d_0}(n_1/2)^{1/d_0}$, where $n_1$ is the width of the first layer and $d_0$ is the input dimension (all notation here is based on Section~\ref{sec:MMNN:structure}). \begin{colorenv}[blue] It was shown in~\citep{ZZZZ-23} that the best scaling strategy is to make the scaling compatible with the network resolution which is proportional to $n_1^{1/d_0}$ in the well-sampled regime. When data sampling is barely enough or even under-sampled for the underlying target function, the initial scaling rule needs to accommodate the sampling condition to avoid overfitting. This is an important and interesting issue which will be studied in our future work.
This subsection is also related to SIREN-type initialization ideas, in the sense
that both approaches recognize the importance of initial frequency scaling for
sine-based networks. Our goal here, however, is not to optimize the hand-tuned SIREN
initialization, but to study a simple architecture-compatible scaling rule for
FMMNNs and FCNNs under the same experimental setting. This allows us to separate
the effect of scaled initialization from the effect of the MMNN architecture.
\end{colorenv}

The motivation for this approach is to address the spectral bias commonly observed in neural networks, where higher frequency components are learned more slowly. By initializing the network with higher frequency modes, the model is able to capture more complex patterns at an earlier stage, which results in faster convergence and improved overall performance, particularly when there is sufficient training data. However, the introduction of higher frequency modes in the network can also lead to overfitting if the sample size is limited. Therefore, it is crucial to tailor both the initialization strategy and the network size to the amount of available data.


All training hyperparameters for the experiments shown in Figure~\ref{fig:scaling:init:MMNNvsFCNN} are consistent with those outlined in Section~\ref{sec:FMMNN_vs_FCNN}, except for minor adjustments to the mini-batch size and learning rate, as detailed below. 
For each function $f_i$ where $i=1,2,3$, the mini-batch size is set to 200. The learning rate is specified as $10^{-3} \times 0.9^{\lfloor k/m \rfloor}$, where $k = 1,2,\dots,100m$ represents the epoch number, and $m$ is set to 600, 30, and 100 for $f_1$, $f_2$, and $f_3$, respectively.
For Figure~\ref{fig:scaling:init:MMNN}, all training hyperparameters are identical to those used in Section~\ref{sec:sine_vs_others}. Additionally, all activation functions employed in these experiments are \texttt{sine}.


\begin{figure}[ht]
            \centering
                            \,\hfill
            \begin{subfigure}[b]{0.3227300245\textwidth}
                    \centering            
                    \includegraphics[width=0.8999\textwidth]{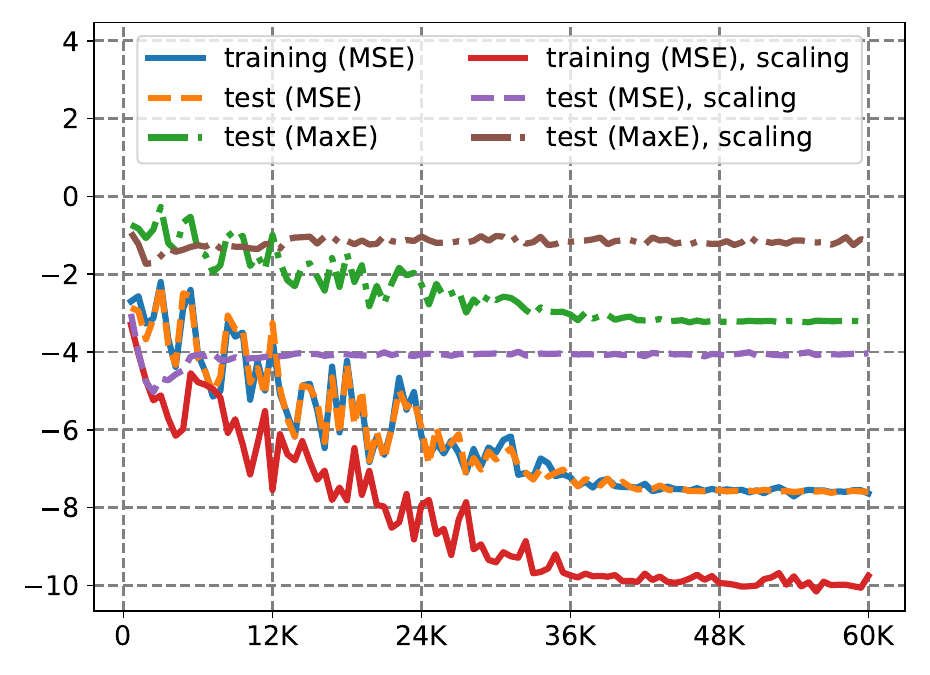}
                    \subcaption{MMNN of size (434,16,6).}
                \end{subfigure}
                \hfill
                            \begin{subfigure}[b]{0.3227300245\textwidth}
                    \centering            
                    \includegraphics[width=0.8999\textwidth]{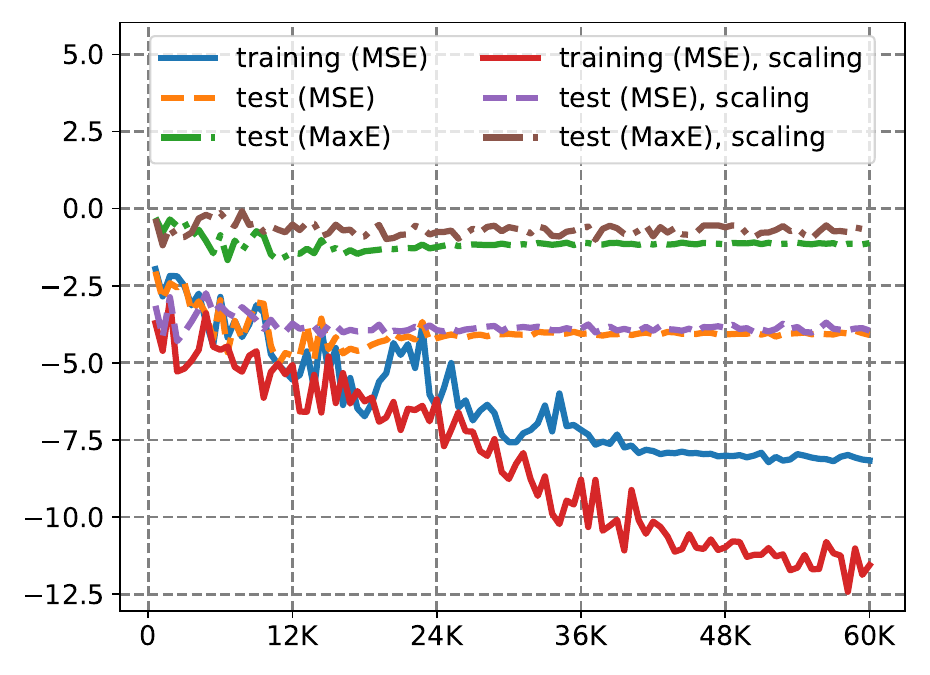}
                    \subcaption{MMNN of size (900,16,6).}
                \end{subfigure}
                \hfill
                            \begin{subfigure}[b]{0.3227300245\textwidth}
                    \centering            
                    \includegraphics[width=0.8999\textwidth]{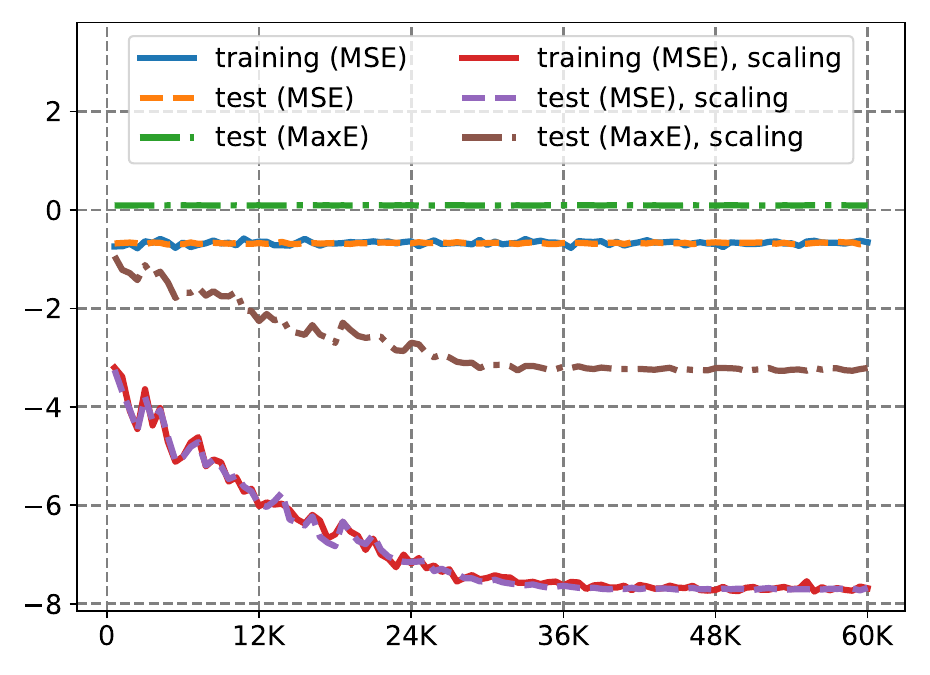}
                    \subcaption{FCNN of size (120,--,6).}
                \end{subfigure}
                \hfill
                \,
                \\
                    \,\hfill
            \begin{subfigure}[b]{0.3227300245\textwidth}
                    \centering            \includegraphics[width=0.8999\textwidth]{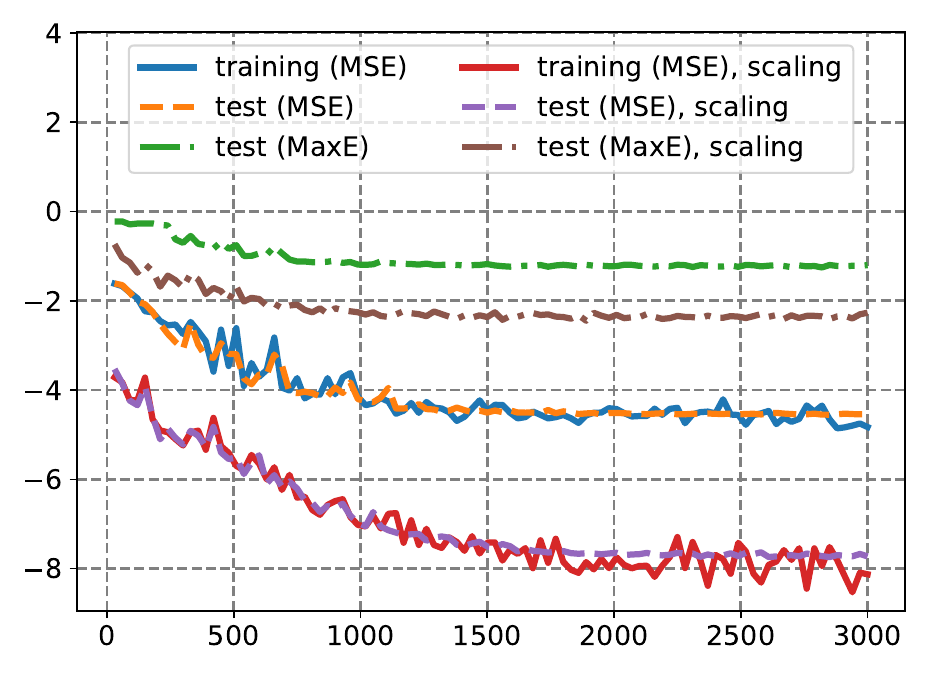}
                    \subcaption{MMNN of size (434,16,6).}
                \end{subfigure}
                \hfill
            \begin{subfigure}[b]{0.3227300245\textwidth}
                    \centering            \includegraphics[width=0.8999\textwidth]{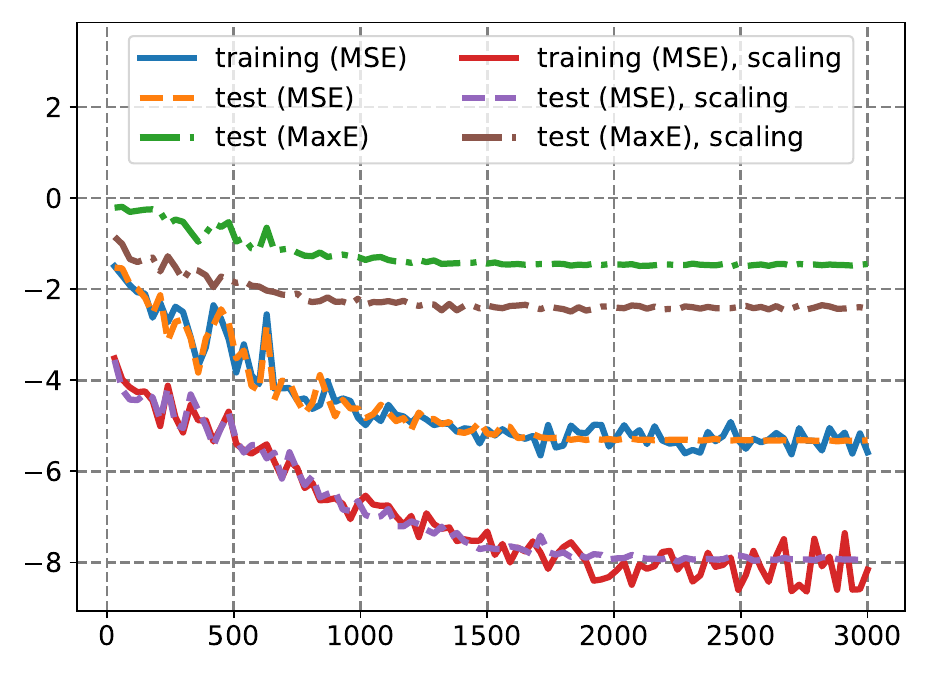}
                    \subcaption{MMNN of size (900,16,6).}
                \end{subfigure}
                \hfill
            \begin{subfigure}[b]{0.3227300245\textwidth}
                    \centering            \includegraphics[width=0.8999\textwidth]{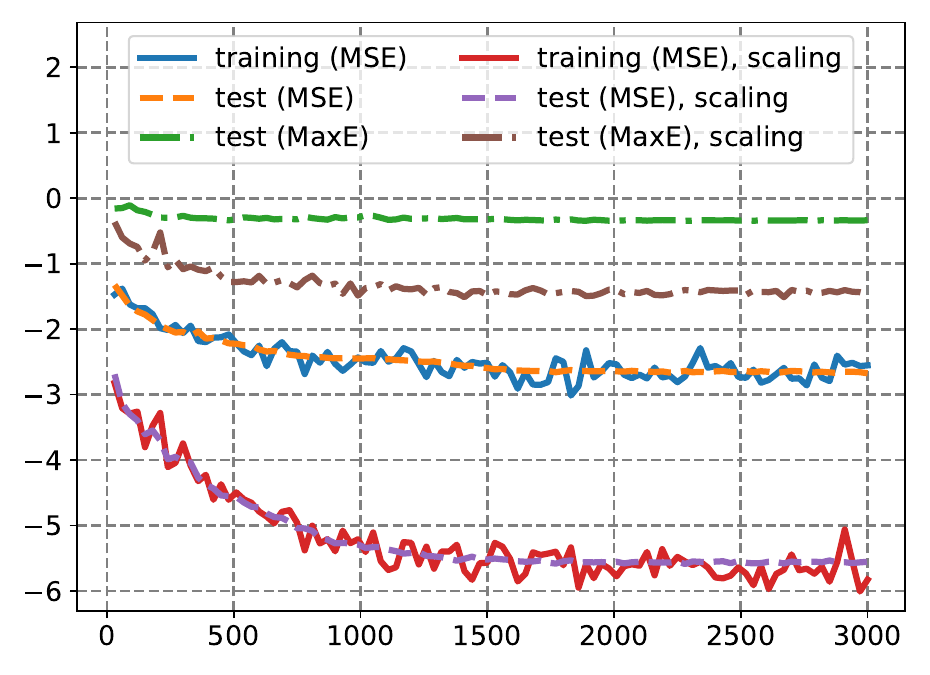}
                    \subcaption{FCNN of size (120,--,6).}
                \end{subfigure}
                \hfill
                \,
                \\
                       \,\hfill
            \begin{subfigure}[b]{0.3227300245\textwidth}
                    \centering            \includegraphics[width=0.8999\textwidth]{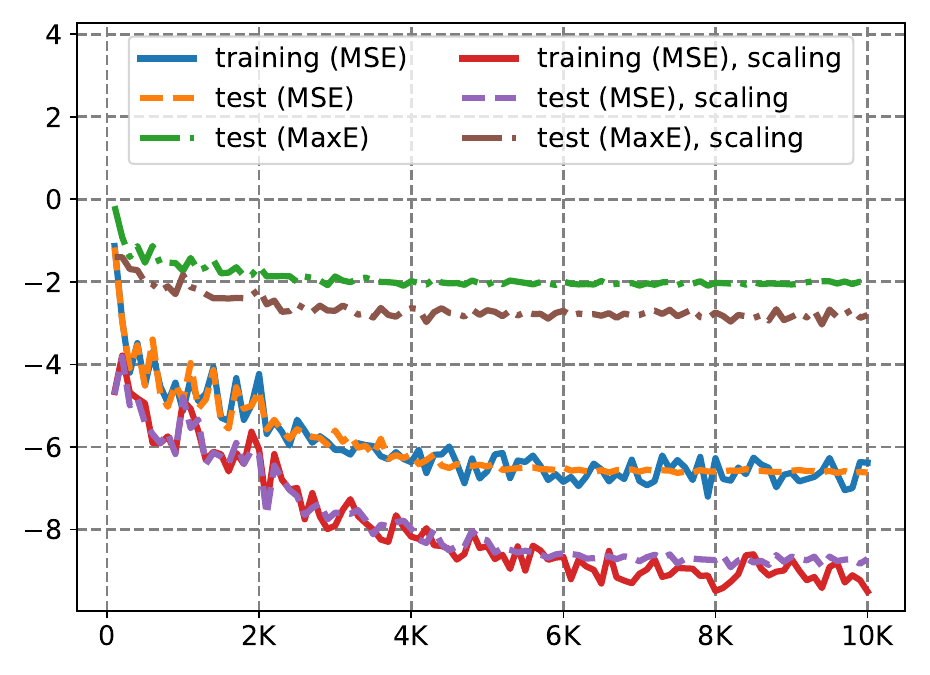}
                    \subcaption{MMNN of size (434,16,6).}
                \end{subfigure}\hfill
                            \begin{subfigure}[b]{0.3227300245\textwidth}
                    \centering            \includegraphics[width=0.8999\textwidth]{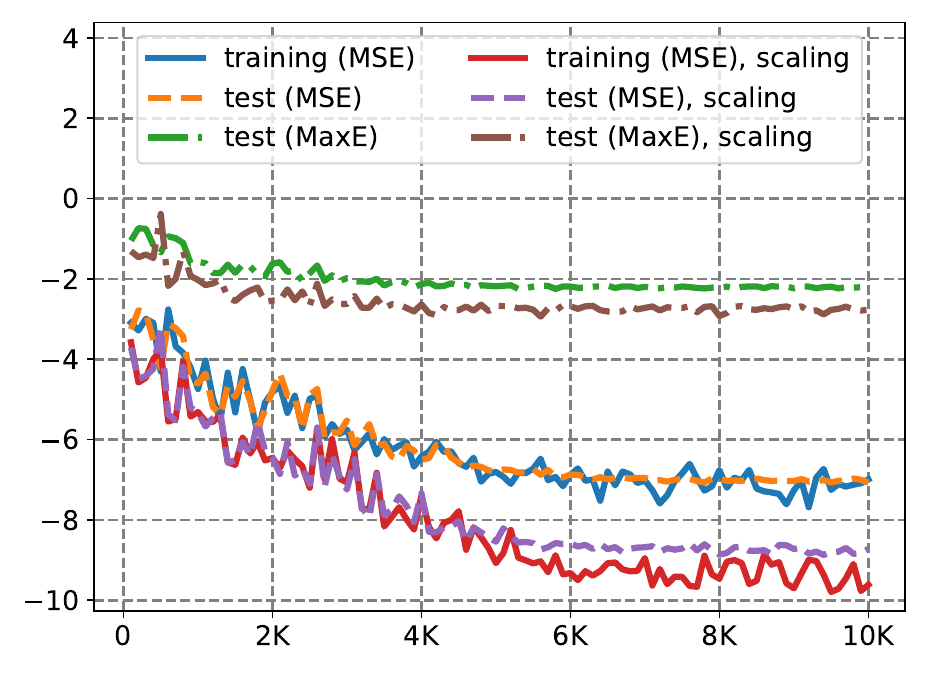}
                    \subcaption{MMNN of size (900,16,6).}
                \end{subfigure}\hfill
            \begin{subfigure}[b]{0.3227300245\textwidth}
                    \centering            \includegraphics[width=0.8999\textwidth]{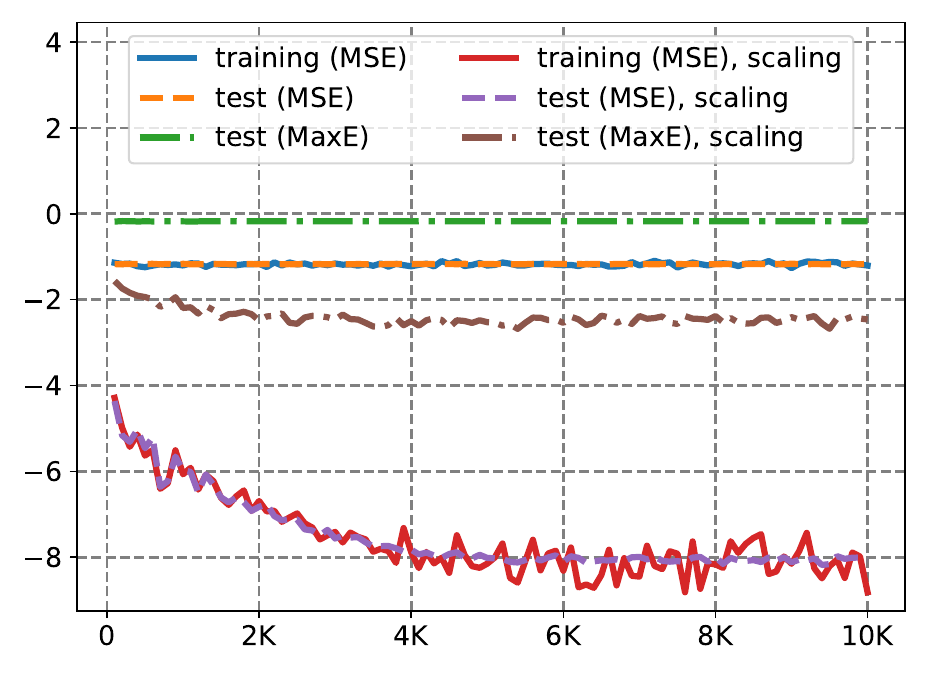}
                    \subcaption{FCNN of size (120,--,6).}
                \end{subfigure}\hfill\,
\caption{Visualization of training and test errors (base-10 logarithm) versus epoch for default and scaling initialization. The first, second, and third rows represent $f_1$, $f_2$, and $f_3$ from Section~\ref{sec:FMMNN_vs_FCNN}, respectively.}
\label{fig:scaling:init:MMNNvsFCNN}
\end{figure}

\begin{figure}[ht]
            \centering
            \,\hfill
            \begin{subfigure}[b]{0.32127300245\textwidth}
                    \centering            
                    \includegraphics[width=0.8999\textwidth]{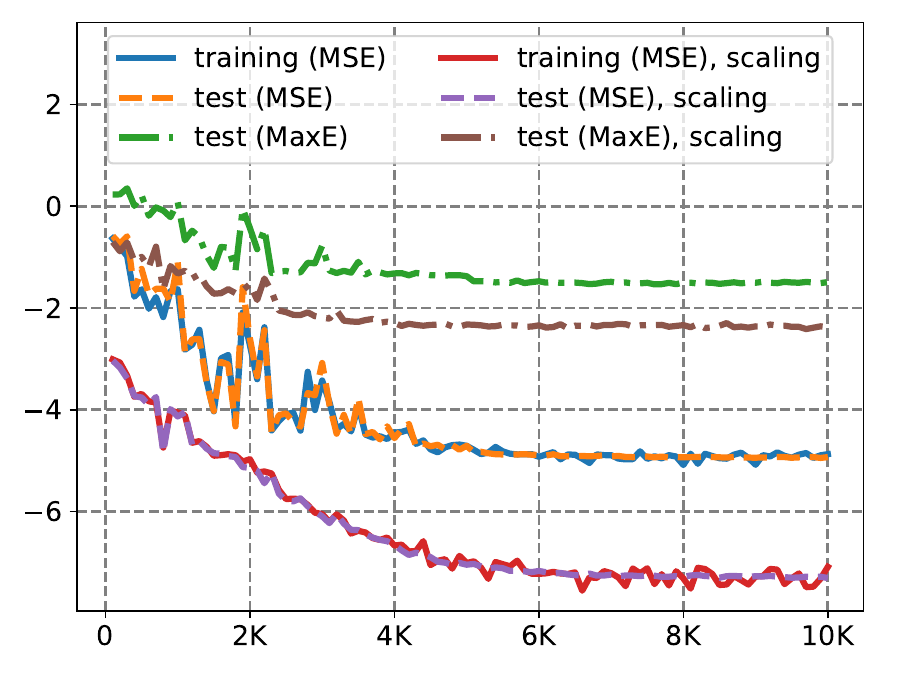}
                    \subcaption{MMNN of size (1024,16,6).}
                \end{subfigure}
                \hfill
            \begin{subfigure}[b]{0.32127300245\textwidth}
                    \centering            \includegraphics[width=0.8999\textwidth]{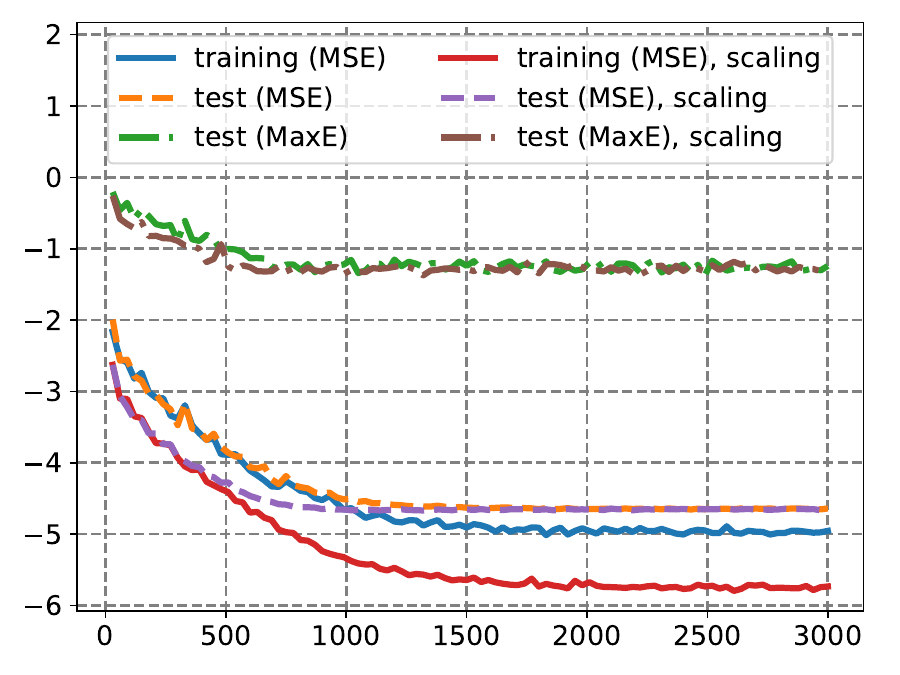}
                    \subcaption{MMNN of size (1024,36,8).}
                \end{subfigure}
                \hfill
            \begin{subfigure}[b]{0.331227300245\textwidth}
                    \centering            \includegraphics[width=0.8999\textwidth]{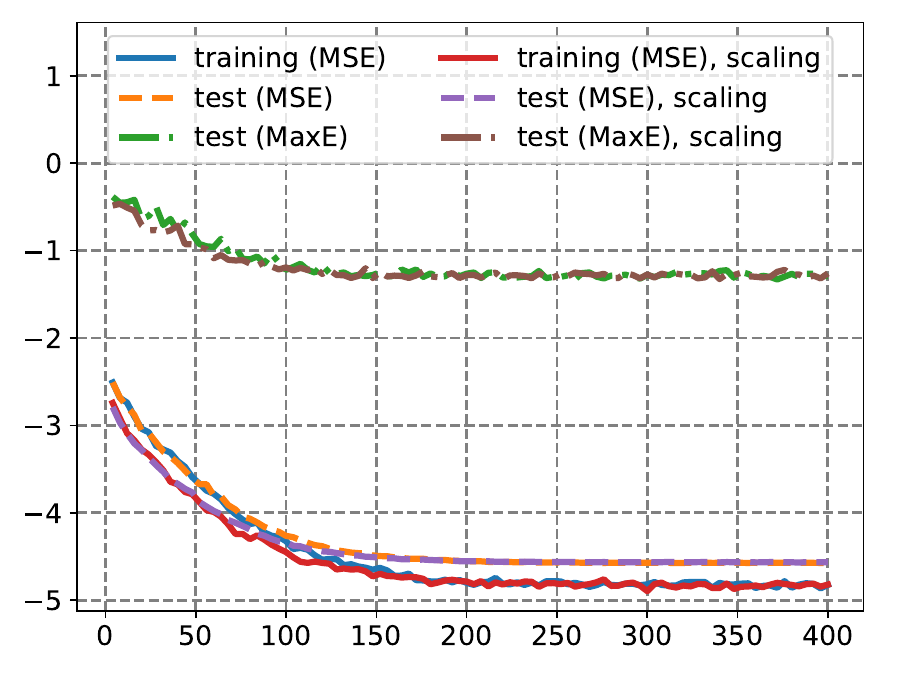}
                    \subcaption{ResMMNN of size (1024,36,10).}
                \end{subfigure}\hfill\,
\caption{Visualization of training and test errors (base-10 logarithm) versus epoch for default and scaling initialization. (a), (b), and (c) correspond to $f_4$, $f_5$, and $f_6$ from Section~\ref{sec:sine_vs_others}, respectively.}
\label{fig:scaling:init:MMNN}
\end{figure}

Figures~\ref{fig:scaling:init:MMNNvsFCNN}, \ref{fig:scaling:init:MMNN}, and \ref{fig:small:MMNNs:scaling} demonstrate that our scaled weight initialization approach is particularly effective when there are enough training samples, resulting in significantly faster learning and better final performance. When the sample size is small, as shown in the first column of Figure~\ref{fig:scaling:init:MMNNvsFCNN} for $f_1$ (a highly oscillatory function) with only 3000 samples, the MMNN tends to overfit. This overfitting occurs because the network width is too large, which introduces excessively high frequencies at initialization. In contrast, the FCNN does not exhibit overfitting in this case, since its width is much smaller.

\begin{figure}[ht]
            \centering
            \,\hfill
            \begin{subfigure}[b]{0.32127300245\textwidth}
                    \centering            
                    \includegraphics[width=0.8999\textwidth]{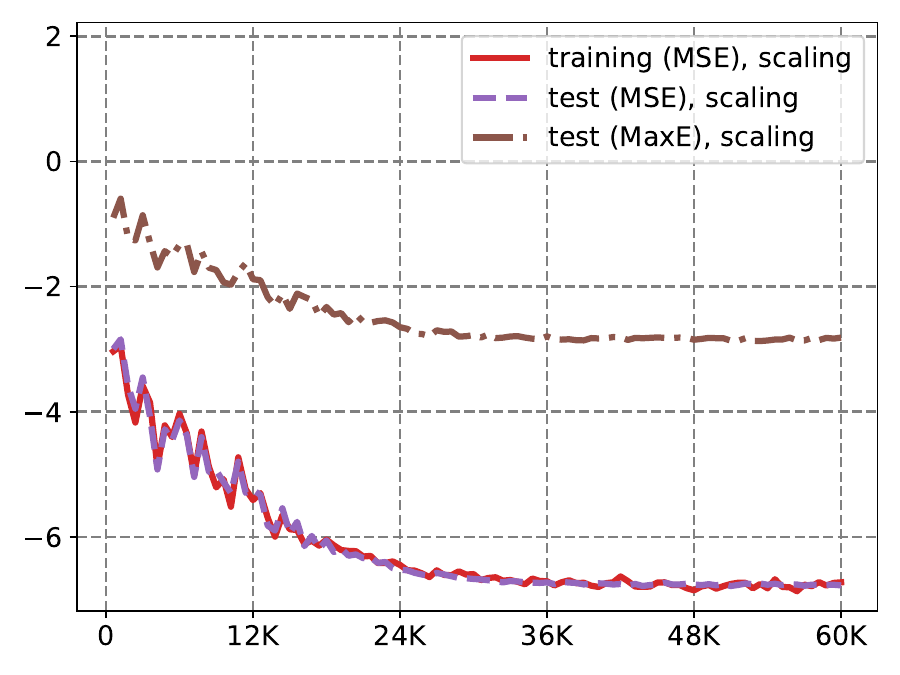}
                    \subcaption{MMNN of size (120,16,6).}
                \end{subfigure}
                \hfill
            \begin{subfigure}[b]{0.331227300245\textwidth}
                    \centering            \includegraphics[width=0.8999\textwidth]{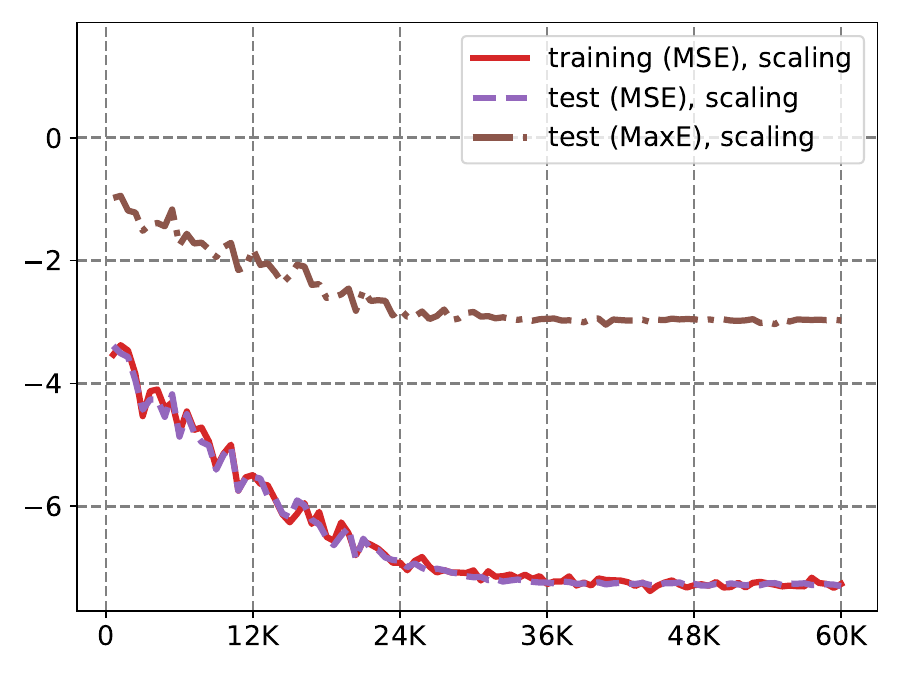}
                    \subcaption{MMNN of size (120,24,6).}
                \end{subfigure}\hfill
            \begin{subfigure}[b]{0.32127300245\textwidth}
                    \centering            
                    \includegraphics[width=0.8999\textwidth]{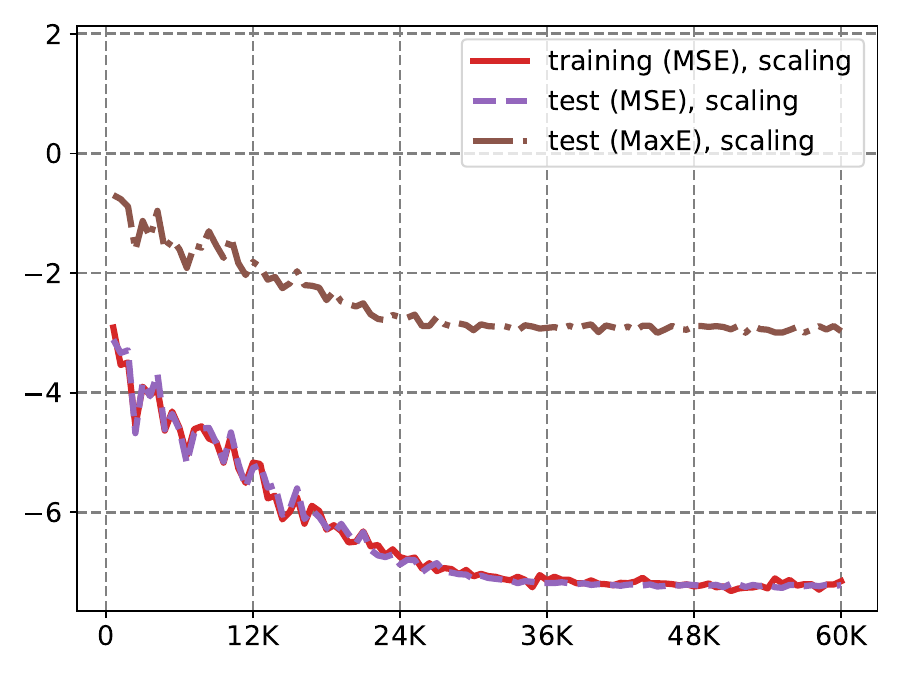}
                    \subcaption{MMNN of size (120,32,6).}
                \end{subfigure}\hfill\,
                \\
            \begin{subfigure}[b]{0.331227300245\textwidth}
                    \centering            \includegraphics[width=0.8999\textwidth]{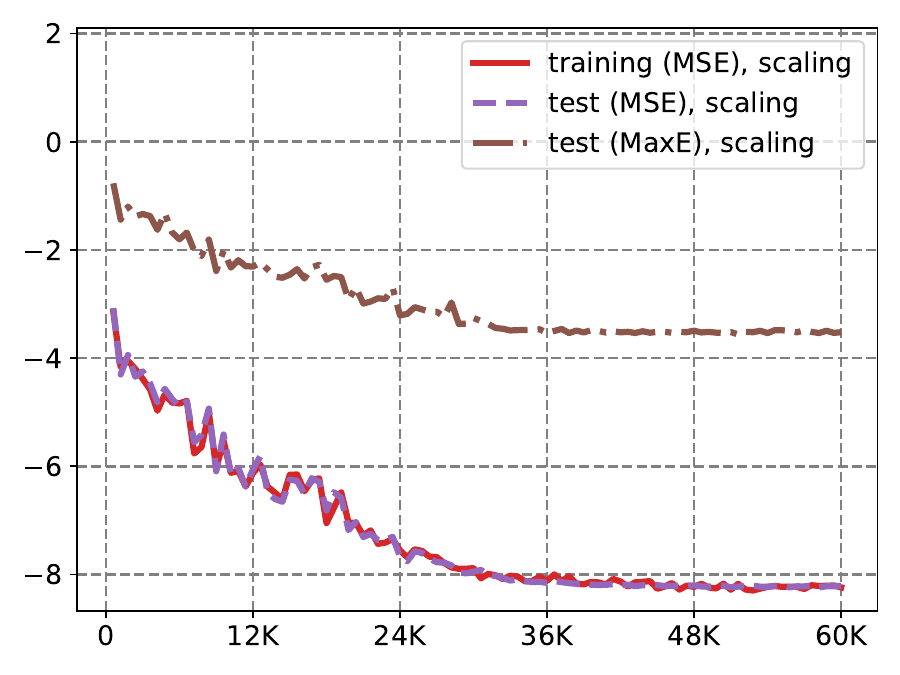}
                    \subcaption{MMNN of size (180,16,6).}
                \end{subfigure}\hfill
            \begin{subfigure}[b]{0.32127300245\textwidth}
                    \centering            
                    \includegraphics[width=0.8999\textwidth]{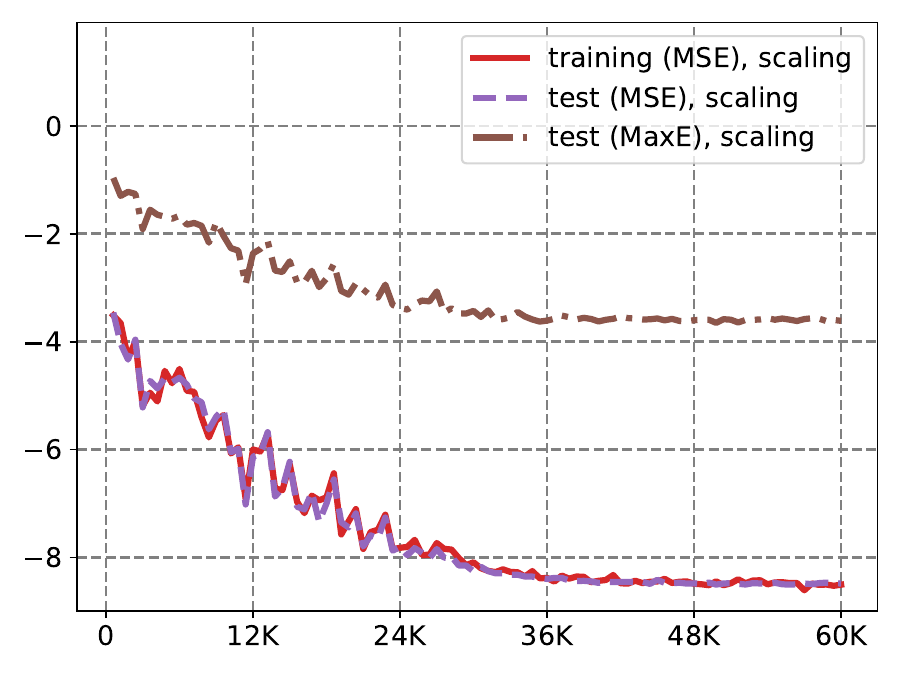}
                    \subcaption{MMNN of size (180,24,6).}
                \end{subfigure}
                \hfill
            \begin{subfigure}[b]{0.331227300245\textwidth}
                    \centering            \includegraphics[width=0.8999\textwidth]{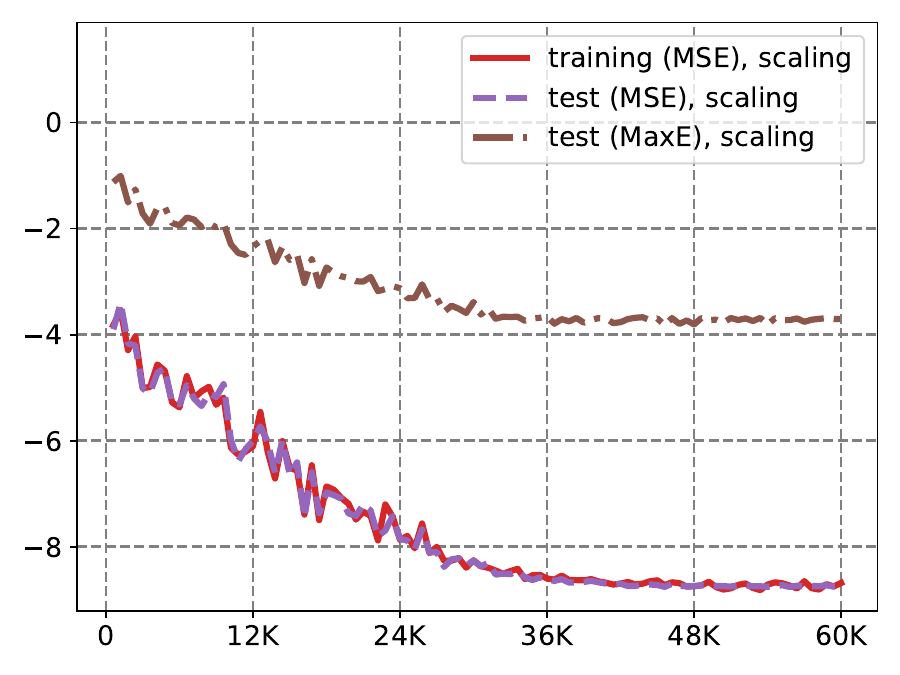}
                    \subcaption{MMNN of size (180,32,6).}
                \end{subfigure}\hfill\,
\caption{Visualization of training and test errors (base-10 logarithm) versus epoch for $f_1$ from Section~\ref{sec:FMMNN_vs_FCNN} using our scaling initialization.}
\label{fig:small:MMNNs:scaling}
\end{figure}

Therefore, it is important to adjust the initialization method and select an appropriate network size based on the available sample size. In Figure~\ref{fig:small:MMNNs:scaling}, we present additional experiments with the same settings except for a smaller network size. As shown in the figure, overfitting is avoided and the results are quite good.
Finally, we note that our scaling initialization is also effective for FCNNs, as demonstrated by the results in the third row of Figure~\ref{fig:scaling:init:MMNNvsFCNN}. The results clearly show that our scaling initialization brings significant improvements.



\begin{colorenv}[blue]

\subsection{Comparison with Fourier Features}
\label{sec:fourier-feature-comparison}

Fourier-feature MLPs \cite{NEURIPS2020_55053683} are among the most commonly used frequency-aware
coordinate networks for low-dimensional regression. Their main idea is to
preprocess the input coordinate by fixed sinusoidal features before applying a
standard MLP, thereby injecting high-frequency basis functions into the input
representation. In this subsection, we compare FMMNN with this widely
used baseline. 
The purpose of this comparison is to separate two mechanisms:
whether the improvement comes merely from providing trigonometric input
features, or from the multi-component and multi-layer composition structure of
FMMNNs.
We consider the same one-dimensional high-frequency approximation benchmark $f_1$ defined in \eqref{eq:def:f1:Cinfty:MMNN:vs:FCNN}. 
Together with the sine-activated FCNN comparisons in Section~\ref{sec:FMMNN_vs_FCNN} and the scaled
initialization study in Section~\ref{sec:scaling:init}, this experiment provides a representative
comparison against modern frequency-aware coordinate-network mechanisms while
keeping the experimental scope focused.

For the Fourier-feature baseline, we use the Gaussian random Fourier feature mapping proposed in \cite{NEURIPS2020_55053683}. In the one-dimensional case, with \(M\) Fourier modes, this mapping is defined by
\[
    \gamma_{\bm{B}}(x)
    =
    \begin{bmatrix}
        \cos(2\pi \bm{B}x) \\
        \sin(2\pi \bm{B}x)
    \end{bmatrix}
    \in \mathbb{R}^{2M},
    \qquad
    \bm{B}\in\mathbb{R}^{M\times 1},
\]
where \(\cos\) and \(\sin\) are applied componentwise. The entries of \(\bm{B}\) are sampled independently from Gaussian distribution \(\mathcal{N}(0,\sigma^2)\), and \(\bm{B}\) is fixed after initialization rather than optimized during training.
The prediction is then given by
\[
    \phi^{\mathrm{FF}}(x)
    =
    \phi\bigl(\gamma_{\bm{B}}(x)\bigr),
\]
where \(\phi\) is a \ReLU-activated fully connected MLP with six hidden
layers. We sweep the Fourier-feature scale
\(\sigma\in\{1,2,4,8,16,32,64,128\}\)
and the number of modes
\(M\in\{32,64,128\}.\)
For \(M=32,64,128\), the MLP widths are set to \(50,45,37\), respectively, so that the number of trainable parameters is comparable to that of the FMMNN baseline with size \((200,16,6)\), where the first-layer parameters \(\bmW_1\) and \(\bmb_1\) are scaled by \(100\) after the default initialization. 
The scaling factor \(100\) is fixed according to
Section~\ref{sec:scaling:init} before the Fourier-feature sweep, and is not
tuned separately for this comparison.
All models are trained on the same uniform grid of \(3000\) points in \([-1,1]\) and evaluated on an independent test set of \(3000\) points sampled uniformly from \([-1,1]\). We use double precision, a mini-batch size of \(200\), and the same training budget. The learning rate at epoch \(k\) is
\(10^{-3}\times 0.9^{\lfloor k/600 \rfloor}\)
for $k=1,2,\dots,60000.$
\begin{table}[!htbp]
    \centering
    \footnotesize
    \setlength{\tabcolsep}{4.2pt}
    \renewcommand{\arraystretch}{1.08}
    \caption{Comparison of FMMNN and Fourier-feature models for the one-dimensional approximation benchmark. For each run, the last 10 logged evaluations are averaged first; each entry reports mean $\pm$ standard deviation over 16 trials.}
    \label{tab:fmmnn-fourier-feature-comparison}
    \resizebox{0.998\textwidth}{!}{%
    \begin{tabular}{@{}lccccccccc@{}}
        \toprule
        model & width & rank & depth & scale $\sigma$ & \#modes $M$ & \#trainable-parameters  & training error (MSE) & test error (MSE) & test error (MaxE) \\
        \midrule
        FMMNN 
        & 200 & 16 & 6 & -- & -- & 16\,281 & $2.74\times 10^{-9}\,\pm\,5.94\times 10^{-10}$ & $\mathbf{2.77\times 10^{-9}}\,\pm\,6.28\times 10^{-10}$ & $\mathbf{2.41\times 10^{-4}}\,\pm\,4.15\times 10^{-5}$ \\
        \midrule
        Fourier Features & 50 & -- & 6 & 1 & 32 & 16\,051 & $1.36\times 10^{-2}\,\pm\,1.07\times 10^{-2}$ & $1.38\times 10^{-2}\,\pm\,1.09\times 10^{-2}$ & $6.41\times 10^{-1}\,\pm\,2.65\times 10^{-1}$ \\
        Fourier Features & 50 & -- & 6 & 2 & 32 & 16\,051 & $3.28\times 10^{-4}\,\pm\,8.05\times 10^{-4}$ & $3.12\times 10^{-4}\,\pm\,7.41\times 10^{-4}$ & $9.31\times 10^{-2}\,\pm\,1.09\times 10^{-1}$ \\
        Fourier Features & 50 & -- & 6 & 4 & 32 & 16\,051 & $2.51\times 10^{-6}\,\pm\,1.33\times 10^{-6}$ & $4.56\times 10^{-6}\,\pm\,1.69\times 10^{-6}$ & $1.66\times 10^{-2}\,\pm\,4.10\times 10^{-3}$ \\
        Fourier Features & 50 & -- & 6 & 8 & 32 & 16\,051 & $3.50\times 10^{-7}\,\pm\,2.20\times 10^{-7}$ & $1.69\times 10^{-6}\,\pm\,3.73\times 10^{-7}$ & $9.50\times 10^{-3}\,\pm\,1.67\times 10^{-3}$ \\
        Fourier Features & 50 & -- & 6 & 16 & 32 & 16\,051 & $5.44\times 10^{-8}\,\pm\,5.09\times 10^{-8}$ & $1.24\times 10^{-6}\,\pm\,2.66\times 10^{-7}$ & $8.40\times 10^{-3}\,\pm\,1.28\times 10^{-3}$ \\
        Fourier Features & 50 & -- & 6 & 32 & 32 & 16\,051 & $1.30\times 10^{-8}\,\pm\,1.03\times 10^{-8}$ & $1.16\times 10^{-6}\,\pm\,1.63\times 10^{-7}$ & $8.34\times 10^{-3}\,\pm\,1.29\times 10^{-3}$ \\
        Fourier Features & 50 & -- & 6 & 64 & 32 & 16\,051 & $1.06\times 10^{-8}\,\pm\,4.70\times 10^{-9}$ & $1.72\times 10^{-6}\,\pm\,3.88\times 10^{-7}$ & $1.05\times 10^{-2}\,\pm\,1.54\times 10^{-3}$ \\
        Fourier Features & 50 & -- & 6 & 128 & 32 & 16\,051 & $1.55\times 10^{-8}\,\pm\,1.73\times 10^{-8}$ & $1.05\times 10^{-5}\,\pm\,8.17\times 10^{-6}$ & $2.43\times 10^{-2}\,\pm\,8.92\times 10^{-3}$ \\
        \midrule
        Fourier Features & 45 & -- & 6 & 1 & 64 & 16\,201 & $1.86\times 10^{-2}\,\pm\,1.14\times 10^{-2}$ & $1.89\times 10^{-2}\,\pm\,1.12\times 10^{-2}$ & $7.43\times 10^{-1}\,\pm\,2.51\times 10^{-1}$ \\
        Fourier Features & 45 & -- & 6 & 2 & 64 & 16\,201 & $7.00\times 10^{-5}\,\pm\,1.85\times 10^{-4}$ & $7.45\times 10^{-5}\,\pm\,1.72\times 10^{-4}$ & $5.95\times 10^{-2}\,\pm\,5.79\times 10^{-2}$ \\
        Fourier Features & 45 & -- & 6 & 4 & 64 & 16\,201 & $2.85\times 10^{-6}\,\pm\,2.10\times 10^{-6}$ & $5.15\times 10^{-6}\,\pm\,3.32\times 10^{-6}$ & $1.81\times 10^{-2}\,\pm\,7.26\times 10^{-3}$ \\
        Fourier Features & 45 & -- & 6 & 8 & 64 & 16\,201 & $3.10\times 10^{-7}\,\pm\,1.31\times 10^{-7}$ & $1.36\times 10^{-6}\,\pm\,2.46\times 10^{-7}$ & $8.19\times 10^{-3}\,\pm\,1.30\times 10^{-3}$ \\
        Fourier Features & 45 & -- & 6 & 16 & 64 & 16\,201 & $6.05\times 10^{-8}\,\pm\,4.69\times 10^{-8}$ & $1.14\times 10^{-6}\,\pm\,2.53\times 10^{-7}$ & $8.32\times 10^{-3}\,\pm\,1.98\times 10^{-3}$ \\
        Fourier Features & 45 & -- & 6 & 32 & 64 & 16\,201 & $1.04\times 10^{-8}\,\pm\,5.24\times 10^{-9}$ & $8.81\times 10^{-7}\,\pm\,1.11\times 10^{-7}$ & $6.97\times 10^{-3}\,\pm\,7.23\times 10^{-4}$ \\
        Fourier Features & 45 & -- & 6 & 64 & 64 & 16\,201 & $2.66\times 10^{-9}\,\pm\,1.47\times 10^{-9}$ & $1.06\times 10^{-6}\,\pm\,1.71\times 10^{-7}$ & $8.14\times 10^{-3}\,\pm\,9.57\times 10^{-4}$ \\
        Fourier Features & 45 & -- & 6 & 128 & 64 & 16\,201 & $3.19\times 10^{-9}\,\pm\,2.41\times 10^{-9}$ & $3.45\times 10^{-6}\,\pm\,1.49\times 10^{-6}$ & $1.41\times 10^{-2}\,\pm\,3.40\times 10^{-3}$ \\
        \midrule
        Fourier Features & 37 & -- & 6 & 1 & 128 & 16\,577 & $4.07\times 10^{-2}\,\pm\,2.67\times 10^{-2}$ & $4.12\times 10^{-2}\,\pm\,2.60\times 10^{-2}$ & $9.75\times 10^{-1}\,\pm\,1.65\times 10^{-1}$ \\
        Fourier Features & 37 & -- & 6 & 2 & 128 & 16\,577 & $1.08\times 10^{-3}\,\pm\,1.91\times 10^{-3}$ & $1.05\times 10^{-3}\,\pm\,1.87\times 10^{-3}$ & $1.93\times 10^{-1}\,\pm\,1.64\times 10^{-1}$ \\
        Fourier Features & 37 & -- & 6 & 4 & 128 & 16\,577 & $8.27\times 10^{-6}\,\pm\,5.73\times 10^{-6}$ & $1.16\times 10^{-5}\,\pm\,7.03\times 10^{-6}$ & $2.63\times 10^{-2}\,\pm\,9.86\times 10^{-3}$ \\
        Fourier Features & 37 & -- & 6 & 8 & 128 & 16\,577 & $5.32\times 10^{-7}\,\pm\,2.63\times 10^{-7}$ & $1.60\times 10^{-6}\,\pm\,4.59\times 10^{-7}$ & $9.70\times 10^{-3}\,\pm\,2.37\times 10^{-3}$ \\
        Fourier Features & 37 & -- & 6 & 16 & 128 & 16\,577 & $9.89\times 10^{-8}\,\pm\,6.63\times 10^{-8}$ & $1.10\times 10^{-6}\,\pm\,2.74\times 10^{-7}$ & $7.72\times 10^{-3}\,\pm\,1.69\times 10^{-3}$ \\
        Fourier Features & 37 & -- & 6 & 32 & 128 & 16\,577 & $1.31\times 10^{-8}\,\pm\,1.36\times 10^{-8}$ & $8.68\times 10^{-7}\,\pm\,1.28\times 10^{-7}$ & $6.84\times 10^{-3}\,\pm\,8.01\times 10^{-4}$ \\
        Fourier Features & 37 & -- & 6 & 64 & 128 & 16\,577 & $1.48\times 10^{-9}\,\pm\,6.91\times 10^{-10}$ & $1.61\times 10^{-6}\,\pm\,3.85\times 10^{-6}$ & $8.06\times 10^{-3}\,\pm\,5.95\times 10^{-3}$ \\
        Fourier Features & 37 & -- & 6 & 128 & 128 & 16\,577 & $\mathbf{4.76\times 10^{-10}}\,\pm\,5.81\times 10^{-10}$ & $1.67\times 10^{-6}\,\pm\,6.89\times 10^{-7}$ & $1.06\times 10^{-2}\,\pm\,2.38\times 10^{-3}$ \\
        \bottomrule
    \end{tabular}%
    }
\end{table}

\begin{figure}[ht]
                    \centering            
\includegraphics[width=0.958999\textwidth]{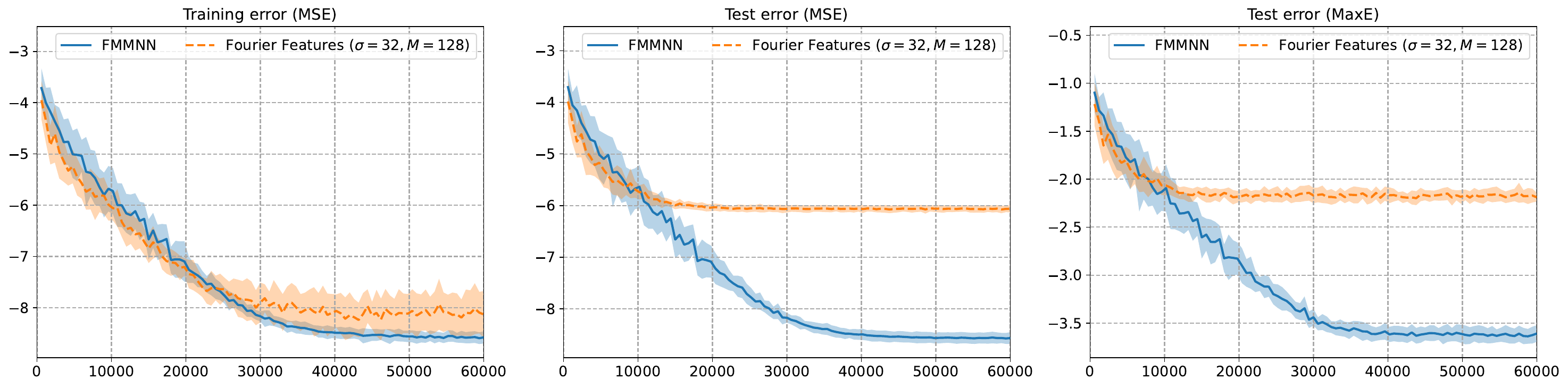}
\caption{
The curves show the mean over \(16\) independent trials, with translucent envelopes indicating one standard deviation, computed after applying the base-\(10\) logarithm to each error.
}
\label{fig:FMMNNs:vs:FF}
\end{figure}

For a more reliable comparison, we conduct \(16\) independent trials for each
configuration. Figure~\ref{fig:FMMNNs:vs:FF} compares FMMNN with the best
Fourier-feature MLP selected from the parameter sweep, namely
\((\sigma,M)=(32,128)\). The curves show the mean over the \(16\) trials, with
translucent envelopes indicating one standard deviation.
Table~\ref{tab:fmmnn-fourier-feature-comparison} reports the mean and one
standard deviation of the final errors.
Figure~\ref{fig:FMMNNs:vs:FF} shows that both methods achieve small training
errors, whereas FMMNN generalizes much better: its test MSE and test MaxE are
both substantially lower than those of the best Fourier-feature MLP.
Table~\ref{tab:fmmnn-fourier-feature-comparison} further reveals a clear
bias--variance tradeoff for Fourier-feature MLPs. When the frequency scale
\(\sigma\) is too small, the model underfits the target function; for example,
\(\sigma=1\) yields test MSEs of order \(10^{-2}\). Increasing \(\sigma\)
substantially reduces the training error and initially also improves the test
error. However, excessively large frequency scales can further reduce the
training error without improving generalization. For instance, the configuration
\((\sigma,M)=(128,128)\) achieves the smallest training MSE among all models,
\(4.76\times 10^{-10}\), but its test MSE is \(1.67\times 10^{-6}\), which is
worse than the best Fourier-feature test result.

Among all Fourier-feature MLPs in the sweep, the best test MSE is
\(8.68\times 10^{-7}\), attained at \((\sigma,M)=(32,128)\), and the best test
MaxE is \(6.84\times 10^{-3}\). In contrast, with a comparable number of
trainable parameters, the FMMNN achieves a test MSE of \(2.77\times 10^{-9}\)
and a test MaxE of \(2.41\times 10^{-4}\).
These results indicate that the advantage of FMMNN is not simply due to the
presence of trigonometric functions. Fourier-feature MLPs inject fixed
sinusoidal features only at the input layer, after which the model is a standard
fully connected network. In contrast, FMMNNs use random sine bases inside each
multi-component layer and then compose these layers. Each component only needs
to learn a linear combination of fixed random bases, while high-frequency
structures can be generated through multi-layer compositions. This provides a
more structured and stable way to represent high-frequency functions under a
similar trainable-parameter budget.
\end{colorenv}

\section{Conclusion}
\label{sec:conclusion}

In this work, we investigate the crucial interplay between neural network architectures and activation functions, emphasizing how their proper alignment significantly influences practical performance. Specifically, we propose the use of \texttt{sine} and a new class of activation functions, \SinTU{s}, and examine their effectiveness within the MMNN structure, termed FMMNN. Our findings demonstrate that the combination of \texttt{sine} or \SinTU{s} with MMNNs establishes a highly synergistic framework, offering both theoretical and empirical advantages, especially for capturing high-frequency components in the target functions.

First, we establish that MMNNs equipped with \texttt{sine} or \SinTU{s} exhibit strong approximation capabilities, surpassing traditional architectures in mathematical expressiveness. We further analyze the optimization landscape of MMNNs, revealing that their training dynamics are considerably more favorable than those of standard FCNNs. This insight suggests that MMNNs benefit from reduced training complexity and improved convergence properties.

To validate our theoretical analysis, we conduct extensive numerical experiments focused on function approximation. The results consistently show that FMMNNs outperform conventional models in both accuracy and computational efficiency. These findings highlight the potential of MMNNs with \texttt{sine}-based activation functions as 
\blue{a promising paradigm for approximation of complicated functions with significant high-frequency features
and related scientific-computing applications.}

While our current experiments primarily focus on function approximation, applying FMMNNs to broader practical tasks remains an important direction for future research. 
From a theoretical standpoint, we have established approximation/expressiveness
results and provided preliminary landscape observations. A rigorous analysis of
the optimization dynamics remains an important direction for future work.
These aspects are left for future investigation.

\section*{Acknowledgments}
S. Zhang was partially supported by
start-up fund  P0053092  from 
Hong Kong Polytechnic University.
H. Zhao was partially supported by NSF grants DMS-2309551 and DMS-2012860. Y. Zhong was partially supported by NSF grant DMS-2309530. H. Zhou was partially supported by NSF grants DMS-2307465 and DMS-2510829.




\hypersetup{
citecolor=black,linkcolor=black,urlcolor=black}
\renewcommand{\doi}[1]{\textnormal{\doitext}~\texttt{\href{https://doi.org/#1}{\detokenize{#1}}}}

\bibliographystyle{plainnat}
  
\bibliography{references}

\appendix

\section{Proofs of Theorems~\ref{thm:main1} and \ref{thm:main2}}
\label{sec:proof:thm:main}

In this section, we establish the proofs of Theorems~\ref{thm:main1} and \ref{thm:main2}. To facilitate understanding, Section~\ref{sec:notation} provides a concise overview of the notations used throughout the paper. In Section~\ref{sec:proof:ideas:thm:main}, we outline the main ideas behind the proofs of Theorems~\ref{thm:main1} and \ref{thm:main2}. Additionally, for simplification, we introduce three propositions whose proofs are deferred to later sections. Assuming the validity of these propositions, we present the full detailed proofs of Theorems~\ref{thm:main1} and \ref{thm:main2} in Section~\ref{sec:proof:thm:main}.


\subsection{Notations}
\label{sec:notation}

Below is a summary of the fundamental notations used throughout this paper.

\begin{itemize}
    \item The difference between two sets \( A \) and \( B \) is denoted by \( A \setminus  B \coloneqq \{ x : x \in A, \ x \notin B \} \).
    
    \item The symbols \( \mathbb{N} \), \( \mathbb{Z} \), \( \mathbb{Q} \), and \( \mathbb{R} \) represent the sets of natural numbers (including 0), integers, rational numbers, and real numbers, respectively. We denote the set of positive natural numbers by \( \mathbb{N}^+ = \mathbb{N} \setminus  \{0\} = \{1, 2, 3, \dots \} \).
    
    \item The floor and ceiling functions of a real number \( x \) are given by
    \( \lfloor x \rfloor = \max \{ n : n \le x, \ n \in \mathbb{Z} \} \) and \( \lceil x \rceil = \min \{ n : n \ge x, \ n \in \mathbb{Z} \} \).
    
    \item For any \( p \in [1, \infty] \), the \( p \)-norm (or \( \ell^p \)-norm) of a vector \( \bm{x} = (x_1, \dots, x_d) \in \mathbb{R}^d \) is defined as
    \begin{equation*}
        \|\bm{x}\|_p = \|\bm{x}\|_{\ell^p} \coloneqq \big( |x_1|^p + \dots + |x_d|^p \big)^{1/p}\quad \text{for } p \in [1, \infty),
    \end{equation*}
    and
    \begin{equation*}
        \|\bm{x}\|_{\infty} = \|\bm{x}\|_{\ell^\infty} \coloneqq \max\big\{ |x_i| : i = 1, 2, \dots, d \big\}.
    \end{equation*}

    \item Let $\cpl(n)$ denote the space of all continuous piecewise linear functions on $\R$ with at most $n \in \N$ breakpoints.
    
    \item The supremum norm of a bounded vector-valued function $\bmf: \Omega\subseteq \R^d \to \R^n$ is defined as
    \begin{equation*}
        \|\bmf\|_{\sup(\Omega)}\coloneqq \sup\big\{|f_i(\bmx)|: \bmx\in \Omega,\   i\in\{1,2,\dots,n\}\big\},
    \end{equation*}
    where $f_i$ represents the $i$-th component of $\bmf$ for $i = 1,2,\dots,n$.

    \item The symbol ``$\rightrightarrows$'' denotes uniform convergence. Specifically, if $\bmf:\R^d\to\R^n$ is a vector-valued function and $\bmf_\delta(\bmx) \rightrightarrows \bmf(\bmx)$ as $\delta\to 0$ for all $\bmx\in \Omega\subseteq \R^d$, then for any $\eps>0$, there exists $\delta_\eps\in (0,1)$ such that
    \begin{equation*}		
        \|\bmf_\delta-\bmf\|_{\sup(\Omega)}< \eps\quad \text{for all } \delta\in (0,\delta_\eps).
    \end{equation*}

    \item We adopt slicing notation for vectors and matrices. Given a vector $\bmx = (x_1, \dots, x_d) \in \mathbb{R}^d$, the notation $\bmx[n:m]$ refers to the slice from the $n$-th to the $m$-th entry for any $n, m \in \{1,2,\dots,d\}$ with $n \leq m$, while $\bmx[n]$ represents the $n$-th entry. For example, if $\bmx = (x_1, x_2, x_3) \in \mathbb{R}^3$, then 
    \(    (6\bmx)[2:3] = (6x_2, 6x_3)\) and \( (8\bmx+1)[3] = 8x_3 + 1.\)
    Similarly, for a matrix $\bmA$, the notation $\bmA[:,i]$ denotes its $i$-th column, while $\bmA[i,:]$ represents its $i$-th row. Moreover, $\bmA[i,n:m]$ is equivalent to $(\bmA[i,:])[n:m]$, extracting the $n$-th to $m$-th entries from the $i$-th row.

\end{itemize}

\subsection{Ideas and Propositions for Proving Theorems~\ref{thm:main1} and \ref{thm:main2}}
\label{sec:proof:ideas:thm:main}

Before presenting the detailed proofs of Theorems~\ref{thm:main1} and \ref{thm:main2}, let us first outline the key ideas underlying our approach. The main strategy in the proof involves constructing a piecewise constant function that approximates the desired continuous target function. However, achieving a uniform approximation with piecewise constants is challenging due to the continuity of \ReLU{} and \sine{}  functions. To address this, we design networks that approximate piecewise constant behavior over most of the domain, specifically outside a small region, ensuring that the approximation error remains well-controlled. Within this small region, the error is manageable, as its measure can be made arbitrarily small.

With this foundation, we now proceed to the details. We divide the domain $[0,1]^d$ into a collection of ``important'' cubes, denoted $\{Q_k\}_{k \in \{1,2,\cdots,M^d\}}$, along with a ``negligible'' region $\Omega$, where $M = N^L$. Each cube $Q_k$ is associated with a representative point $\bm{x}_k \in Q_k$. An illustration of $\bm{x}_k$, $\Omega$, and $Q_k$ can be seen in Figure~\ref{fig:idea:main}. The construction of the desired network to approximate the target function is organized into two main steps below.

\begin{enumerate}
    \item First, we construct a sub-network that realizes a function $\phi_1$ which maps each cube $Q_k$ to its respective index $k$. Specifically, we have $\phi_1(\bm{x}) = k$ for any $\bm{x} \in Q_k$ and $k \in \{1,2,\cdots,M^d\}$.
    
    \item Next, we design a sub-network to implement a function $\phi_2$ that maps each index $k$ approximately to $f(\bm{x}_k)$. Consequently, we obtain $\phi_2 \circ \phi_1(\bm{x}) = \phi_2(k) \approx f(\bm{x}_k) \approx f(\bm{x})$ for any $\bm{x} \in Q_k$ and $k \in \{1,2,\cdots,M^d\}$, implying that $\phi_2 \circ \phi_1 \approx f$ outside of $\Omega$.
\end{enumerate}

\begin{figure}[ht]
	\centering
	\includegraphics[width=0.825\linewidth]{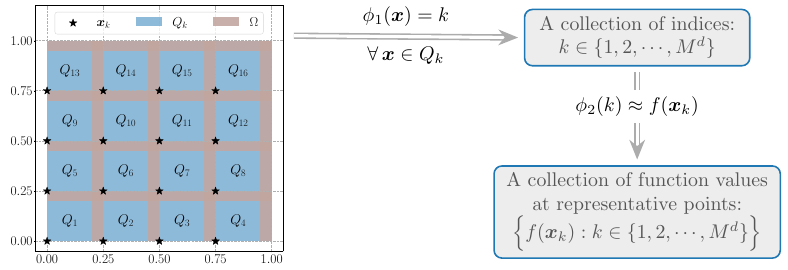}
\caption{An illustration of the approach for constructing a network to approximate \( f \). 
Observe that \(\phi_2 \circ \phi_1 \approx f\) outside of \(\Omega\), as \(\phi_2 \circ \phi_1(\bm{x}) = \phi_2(k) \approx f(\bm{x}_k) \approx f(\bm{x})\) for any \(\bm{x} \in Q_k\) and \(k \in \{1, 2, \dots, M^d\}\). 
}
	\label{fig:idea:main}
\end{figure}


The floor function is quite effective for handling the first step. To simplify the final proof, we introduce Proposition~\ref{prop:floor:approx} below, which demonstrates how to construct a network that efficiently approximates the floor function. The proof of Proposition~\ref{prop:floor:approx} is provided in Section~\ref{sec:proof:prop:floor:approx}.

\begin{proposition}
	\label{prop:floor:approx}
	Given any $\delta\in (0,1)$ and $N,L\in \N^+$,
	there exists 
 \[\phi\in \mn[\big]{\ReLU}{4N-1}{3}{L}{\R}{\R}\]
 such that 
	\begin{equation*}
		\phi(x)=\lfloor x\rfloor\quad \tn{for any $x\in \bigcup_{k=0}^{N^L-1}\big[k,\, k+1-\delta\big].$}
	\end{equation*}
\end{proposition}




The purpose of $\phi_2$ is to map each $k$ approximately to $f(\bmx_k)$ for $k \in \{1, 2, \cdots, M^d\}$. Notably, in constructing $\phi_2$, we only need to ensure correct values at a finite set of points $\{1, 2, \cdots, M^d\}$, rather than over an entire continuous domain. This key insight significantly simplifies the design of a network that realizes $\phi_2$.
However, even with this simplification, the \ReLU\ activation function is not particularly effective for this type of point-matching problem. In Proposition~\ref{prop:k:to:yk} below, we demonstrate that the \sine{}  function is exceptionally efficient for this task. The proof of Proposition~\ref{prop:k:to:yk} is provided in Section~\ref{sec:proof:prop:k:to:yk}.


\begin{proposition}
    \label{prop:k:to:yk}
    Given any  $\eps>0$ and $y_k\in \R$ for $k=1,2,\cdots,K$, there exist $ v ,w \in\R$ such that
    \begin{equation*}
        \big|u \cdot \sin\big(v  \cdot \sin(kw )\big)-y_k\big|<\eps \quad \tn{for $k=1,2,\cdots,K$,}
    \end{equation*}
    where   \( u = \max\{|y_k| : k = 1, 2, \cdots, K\} \).
\end{proposition}

We remark that Proposition~\ref{prop:k:to:yk} can also be understood through the concept of density. Specifically, for any \( K \in \mathbb{N}^+ \), the set
\begin{equation*}
    \Bigg\{\bigg(u \cdot \sin\big(v \cdot \sin(w)\big),\; u \cdot \sin\big(v \cdot \sin(2w)\big),\; \cdots,\; u \cdot \sin\big(v \cdot \sin(Kw)\big)\bigg) : u, v, w \in \mathbb{R} \Bigg\}
\end{equation*}
is dense in \( \mathbb{R}^K \). 

When analyzing the approximation power of MMNNs activated by \SinTU{s}, we need to leverage the singularity of  \SinTU{s}  for spatial partitioning. To simplify this process, we use \(\ReLU\) for spatial partitioning and employ sub-MMNNs activated by  \SinTU{s}  to reproduce/approximate  \ReLU. Proposition~\ref{prop:approx:ReLU} below is specifically introduced to streamline the proof. The detailed proof of Proposition~\ref{prop:approx:ReLU} can be found in Section~\ref{sec:proof:prop:approx:ReLU}.

\begin{proposition}
	\label{prop:approx:ReLU}
	Given any $B>0$ and $\varrho\in \calS$, there exists
	$\phi_\alpha\in \nn{\varrho}{2}{1}{\R}{\R}$ for each $\alpha\in (0,1)$
	such that
	\begin{equation*}
		\phi_\alpha(x)\rightrightarrows \ReLU(x)\quad \tn{as}\   \alpha\to 0^+ \quad \tn{for any $x\in [-B,B]$.}
	\end{equation*}
\end{proposition}

The above proposition demonstrates that two active \(\varrho\)-activated neurons are sufficient to approximate \(\ReLU\) arbitrarily well.



\subsection{Detailed Proofs of Theorems~\ref{thm:main1} and \ref{thm:main2} Based on Propositions}
\label{sec:detailed:proof:thm:main}

We are now prepared to present the detailed proofs of Theorems~\ref{thm:main1} and \ref{thm:main2}, assuming the validity of Propositions~\ref{prop:floor:approx},  \ref{prop:k:to:yk}, and~\ref{prop:approx:ReLU}, which will be proven in Sections~\ref{sec:proof:prop:floor:approx}, \ref{sec:proof:prop:k:to:yk}, and~\ref{sec:proof:prop:approx:ReLU}, respectively.

	
	

Let $M=N^L$ and $\delta\in (0,1)$ be a small number determined later.
We first divide $[0,1]^d$ into a set of sub-cubes  and a small region. To this end, we
define  $\tildebmx_\bmbeta \coloneqq \bmbeta/M$ and 
\[
\tildeQ_{\bm{\beta}}\coloneqq\Big\{{\bm{x}}= (x_1,\cdots,x_d) \in [0,1]^d:x_i\in[\tfrac{\beta_i}{M},\tfrac{\beta_i+1-\delta}{M}], \ i=1,\cdots,d\Big\}
\]
for each $d$-dimensional index  ${\bm{\beta}}= (\beta_1,\cdots,\beta_d) \in \{0,1,\cdots,M-1\}^d$. Then the ``negligible'' region $\Omega$, given by
\[\Omega=[0,1]^d \setminus  \big(\cup_{\bm{\beta}\in \{0,1,\cdots,M-1\}^d}\tildeQ_{\bm{\beta}}\big),\]
has a sufficiently small measure for small \(\delta\).

To simplify notation, we reindex the \(d\)-dimensional indices as one-dimensional indices. For this purpose, we establish a one-to-one mapping between \(\{0, 1, \cdots, M-1\}^d\) and \(\{1, 2, \cdots, M^d\}\), defined by\footnote{Note that the definition of \(g\) is inspired by concepts from representations of integers in various bases.}
\[
g(\bm{\beta}) = 1 + \sum_{i=1}^{d} \beta_i \cdot M^{i-1}\quad \tn{for any $\bmbeta=(\beta_1,\cdots,\beta_d)\in \{0, 1, \cdots, M-1\}^d$}.
\]
Thus, for each \(k \in \{1, 2, \cdots, M^d\}\), there exists a unique \(\bm{\beta} \in \{0, 1, \cdots, M-1\}^d\) such that \(g(\bm{\beta}) = k\). Accordingly, we reindex \(\tilde{\bm{x}}_{\bm{\beta}}\) and \(\tilde{Q}_{\bm{\beta}}\) as \(\bm{x}_k\) and \(Q_k\), respectively. That is,
\begin{equation*}
    \bm{x}_{k}= \tilde{\bm{x}}_{\bm{\beta}}
 \quad \tn{and}\quad  Q_{k}=  \tilde{Q}_{\bm{\beta}}\quad \tn{with $k=g(\bmbeta)$}\quad \tn{for any $\bmbeta=(\beta_1,\cdots,\beta_d)\in \{0, 1, \cdots, M-1\}^d$}.
\end{equation*}
See Figure~\ref{fig:Omega:Q:x} for illustrations of  $\Omega$,  $\tildeQ_\bmbeta$,  $\tildebmx_\bmbeta$, $Q_k$, and $\bmx_k$ for $\bmbeta\in \{0,1,\cdots,M-1\}^d$ and $k\in \{1,2,\cdots,M^d\}$
 when $M=4$ and $d=2$.

 \begin{figure}[!htp]
	\centering
	\begin{minipage}{0.905\textwidth}
		\centering
		\begin{subfigure}[b]{0.4\textwidth}
			\centering
			\includegraphics[width=0.7999\textwidth]{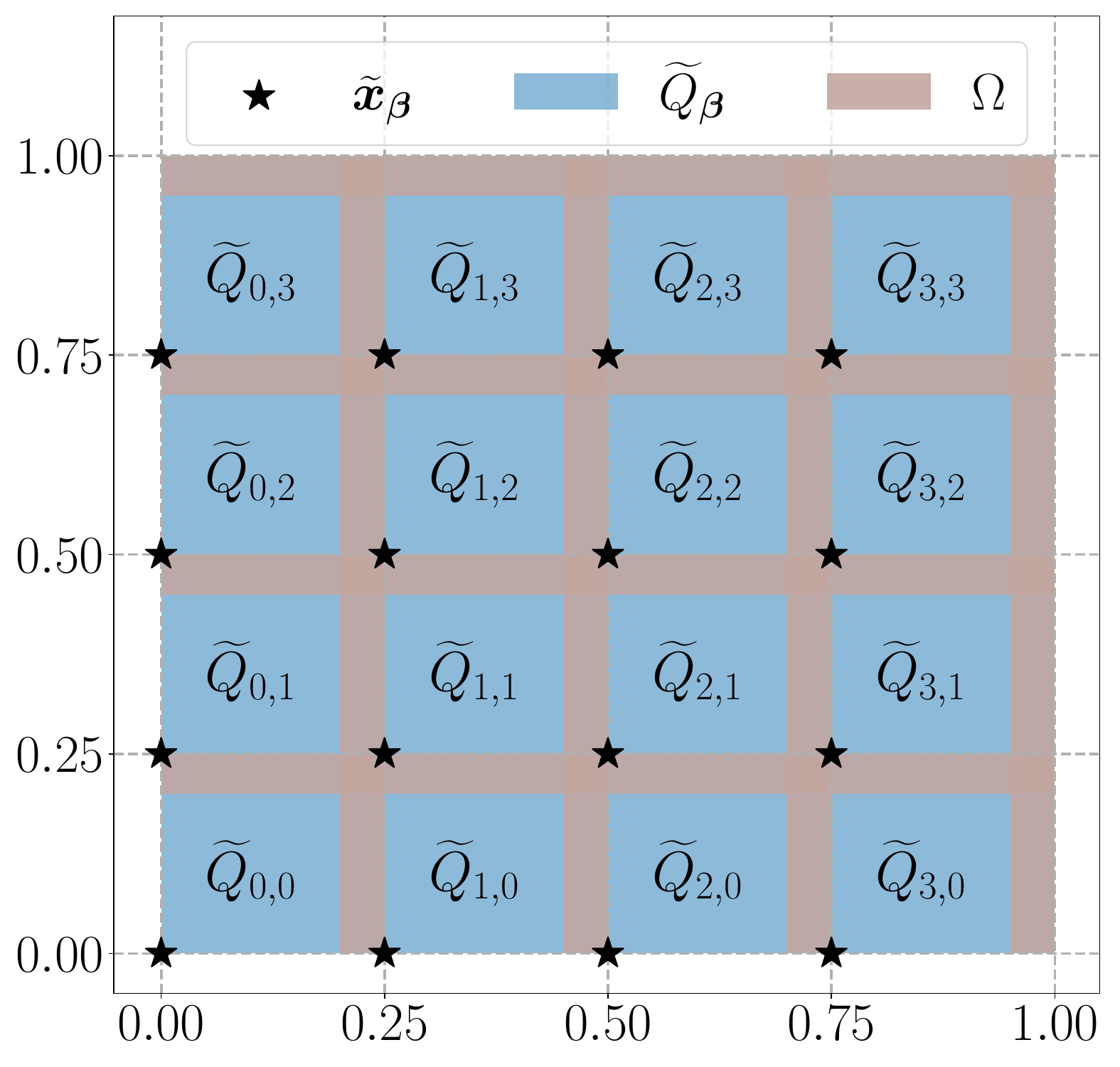}
		\end{subfigure}
			\begin{minipage}{0.1\textwidth}
				\
			\end{minipage}
		\begin{subfigure}[b]{0.4\textwidth}
			\centering
			\includegraphics[width=0.7999\textwidth]{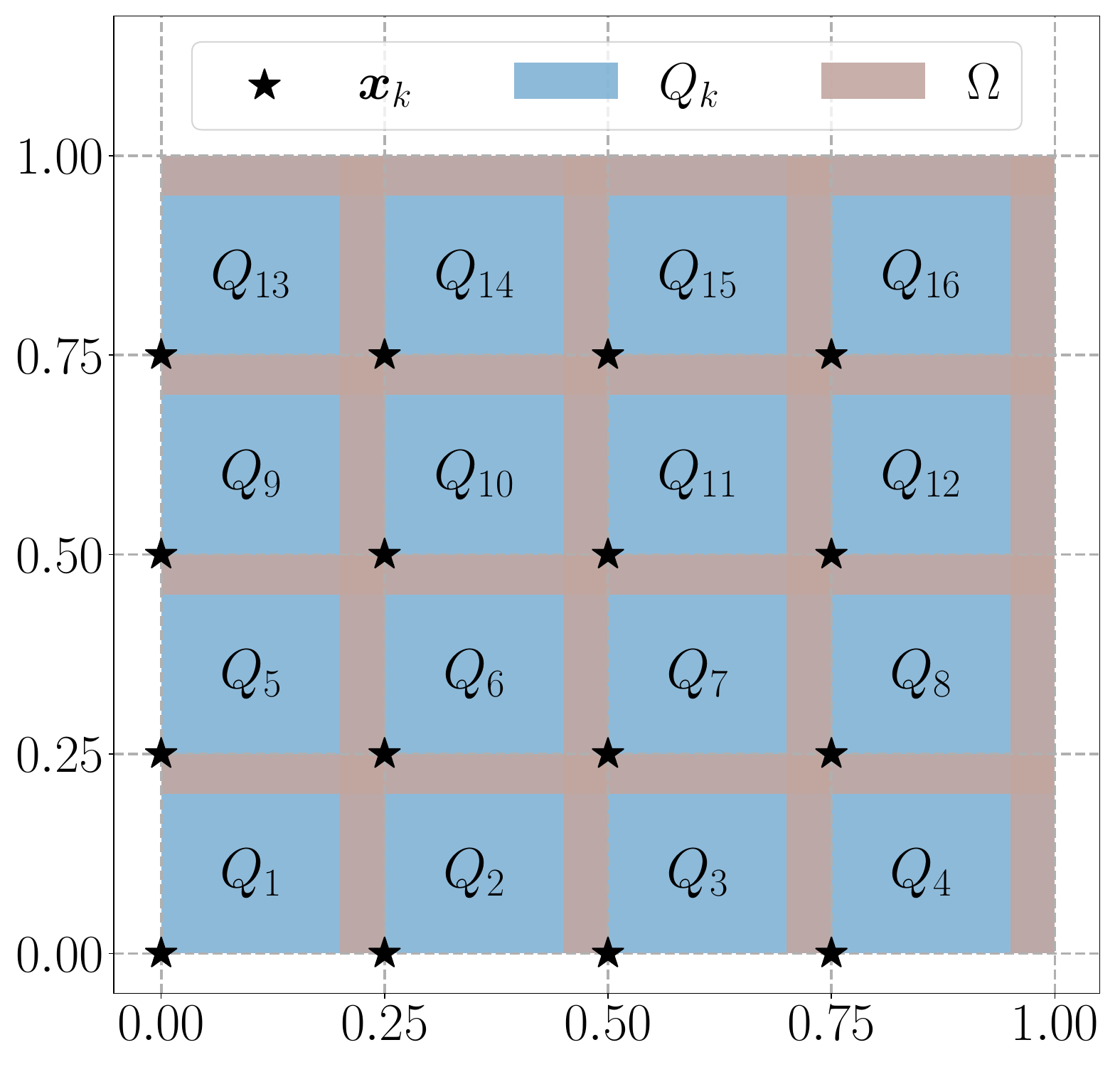}
		\end{subfigure}
	\end{minipage}
	\caption{Illustrations of  $\Omega$,  $\tildeQ_\bmbeta$,  $\tildebmx_\bmbeta$, $Q_k$, and $\bmx_k$ for $\bmbeta\in \{0,1,\cdots,M-1\}^d$ and $k\in \{1,2,\cdots,M^d\}$
 when $M=4$ and $d=2$. }
	\label{fig:Omega:Q:x}
\end{figure}

Next, we construct a network-realized function that maps \(\bm{x}  \in Q_k\) to \(k\) for any \(k \in \{1, 2, \cdots, M^d\}\). Fixing \(\bm{x}= (x_1, \cdots, x_d) \in Q_k\) for some \(k \in \{1, 2, \cdots, M^d\}\), there exists a unique \(\bm{\beta} = (\beta_1, \cdots, \beta_d)\) such that \(g(\bm{\beta}) = k\). 
Then, \(\bm{x} \in Q_k = \tilde{Q}_{\bm{\beta}}\) implies
\begin{equation*}
    x_i \in \Big[\frac{\beta_i}{M},\; \frac{\beta_i + 1 - \delta}{M}\Big]\quad \tn{for } i = 1, \cdots, d,
\end{equation*}
from which we deduce
\begin{equation}\label{eq:Mxi:in:beta:i:i+1}
    Mx_i \in [\beta_i,\; \beta_i + 1 - \delta]  \quad \tn{with } \beta_i \in \{0, 1, \cdots, M-1\} \quad  \tn{ for } i = 1, \cdots, d.
\end{equation}
By Proposition~\ref{prop:floor:approx}, there exists a function \( {\psi}  \in \mn[\big]{\ReLU}{4N-1}{3}{L}{\mathbb{R}}{\mathbb{R}}\) such that
\begin{equation*}
     {\psi} (x) = \lfloor x \rfloor \quad \text{for any } x \in \bigcup_{m=0}^{N^L - 1} \big[m, m + 1 - \delta\big] = \bigcup_{m=0}^{M - 1} \big[m, m + 1 - \delta\big].
\end{equation*}
Then, by Equation~\eqref{eq:Mxi:in:beta:i:i+1}, we have
\begin{equation*}
    {\psi}(Mx_i) = \beta_i \quad \text{for } i = 1,   \cdots, d.
\end{equation*}
By defining
\begin{equation*}
\bmPsi(\bmy)\coloneqq \Big(\psi(My_1),\psi(My_2),\cdots,\psi(My_d)\Big) \quad \tn{for any } \bmy=(y_1,\cdots,y_d) \in \R^d,
\end{equation*}
we have $\bmPsi(\bmx)=\bmbeta$, implying
\begin{equation*}
    g\circ \bmPsi(\bmx)=g(\bmbeta)=k.
\end{equation*}
Since \(\bm{x} \in Q_k\) and \(k \in \{1, 2, \cdots, M^d\}\) are arbitrary, we have
\begin{equation}\label{eq:g:tildePsi:x:to:k}
    g \circ \bmPsi(\bm{x}) = k \quad \tn{for any } \bm{x} \in Q_k \text{ and } k \in \{1, 2, \cdots, M^d\}.
\end{equation}

To construct a function that maps \( k \in \{1, 2, \cdots, M^d\} \) approximately to \( f(\bm{x}_k) \), we apply Proposition~\ref{prop:k:to:yk} with \( K = M^d \) and \( y_k = f(\bm{x}_k) \). This yields the existence of constants \( u, v, w \in \mathbb{R} \) such that
\begin{equation}
\label{eq:sinsin:approx:f(xk):eps}
    \left| u \cdot \sin\big(v \cdot \sin(k w)\big) - f(\bmx_k) \right| < \epsilon \quad \text{for } k = 1, 2, \cdots, M^d,
\end{equation}
where \( \epsilon > 0 \) is a small number to be determined later.

We define 
\begin{equation*}
    \phi_1(\bmy)\coloneqq  g\circ \bmPsi(\bmy)\quad \tn{for any $\bmy\in \R^d$}
\end{equation*}
and
\begin{equation*}
    \phi_2(z) = u \cdot \sin\big(v \cdot \sin(wz)\big)
    \quad \tn{for any $z\in \R$.}
\end{equation*}
It follows from
\(\psi \in \mn[\big]{\ReLU}{4N-1}{3}{L}{\mathbb{R}}{\mathbb{R}}\)
that 
\[{\bm{\Psi}} \in \mn[\big]{\ReLU}{d(4N-1)}{3d}{L}{\mathbb{R}^d}{\mathbb{R}^d}.\]
Recall that \( g:\R^d\to\R\) is an affine linear map. 
By combining \(g\) with the final affine linear map (corresponding to \(\bm{\Psi}\)) into a new mapping, we obtain
\[\phi_1 =  g \circ {\bm{\Psi}} \in \mn[\big]{\ReLU}{d(4N-1)}{3d}{L}{\mathbb{R}^d}{\mathbb{R}}.\]
Since $\phi_2(z) = u \cdot \sin\big(v \cdot \sin(wz)\big)$, we have
\[ \phi_2\circ\phi_1  \in \mn[\big]{(\sin,\ReLU)}{d(4N-1)}{3d}{L+2}{\mathbb{R}^d}{\mathbb{R}}.\]

To complete the proof of Theorem~\ref{thm:main2}, where the corresponding $\phi$ is defined as $\phi \coloneqq \phi_2 \circ \phi_1$, it remains to bound the approximation error.
For any \(\bm{x} \in Q_k\) and \(k \in \{1, 2, \cdots, M^d\}\), by Equations~\eqref{eq:g:tildePsi:x:to:k} and \eqref{eq:sinsin:approx:f(xk):eps}, we have
\begin{equation*}
    \begin{split}
        |\phi_2\circ \phi_1(\bmx)-f(\bmx_k)|  
        =\Big|\phi_2\Big(
        \underbrace{g\circ\bmPsi(\bmx)}_{\tn{$=k$ by \eqref{eq:g:tildePsi:x:to:k}}}
        \Big)-f(\bmx_k)\Big|        
        &= |\phi_2 ( k )-f(\bmx_k) |
        \\ &
         =\left| u \cdot \sin\big(v \cdot \sin(k w)\big) - f(\bmx_k) \right| < \epsilon
    \end{split}
\end{equation*}
and $\|\bmx_k-\bmx\|_2\le 
\frac{\sqrt{d}}{M}$,
from which we deduce
\begin{equation*}
    \begin{split}
        |\phi_2\circ \phi_1(\bmx)-f(\bmx)|
         \le 
        |\phi_2\circ \phi_1(\bmx)-f(\bmx_k)|
        +|f(\bmx_k)-f(\bmx)|
        \le \eps + \omega_f(\tfrac{\sqrt{d}}{M}).
    \end{split}
\end{equation*}
Recall that 
\begin{equation*}
    \bigcup_{k\in \{1,2,\cdots,M^d\}} Q_{k} = [0,1]^d \setminus \Omega\quad \tn{ and}  \quad
  \|\phi_2 \circ \phi_1\|_{L^\infty([0,1]^d)} \le \|\phi_2\|_{L^\infty(\mathbb{R})}  \le |u|,
\end{equation*}
where \(u\) is a constant determined by \(f\) and is independent of \(\delta\).
Therefore
	\begin{equation*}
	\begin{split}
	\| \phi_2\circ \phi_1-f\|_{L^p([0,1]^d)}^p
	&=\int_{\Omega}| \phi_2\circ \phi_1(\bmx)-f(\bmx)|^p\tn{d}\bmx
    +
    \int_{[0,1]^d\setminus \Omega}|\phi_2\circ \phi_1(\bmx)-f(\bmx)|^p\tn{d}\bmx
 \\
 	&\le 
 \mu(\Omega)\cdot \| \phi_2\circ \phi_1-f\|_{L^\infty(\Omega)}^p
 + \sum_{k=1}^{M^d}\int_{Q_k}| \phi_2\circ \phi_1(\bmx)-f(\bmx)|^p\tn{d}\bmx
	\\&\le 
 \mu(\Omega)\cdot \Big(|u|+\|f\|_{L^\infty([0,1]^d)}\Big)^p
 + \sum_{k=1}^{M^d}
 \mu(Q_k)\cdot\Big(\eps + \omega_f(\tfrac{\sqrt{d}}{M})\Big)^p
 \\&\le 
 \mu(\Omega)\cdot \Big(|u|+\|f\|_{L^\infty([0,1]^d)}\Big)^p
 + \Big(\eps + \omega_f(\tfrac{\sqrt{d}}{M})\Big)^p
 \le \Big(\tfrac{11}{10}\omega_f(\tfrac{\sqrt{d}}{M})\Big)^p
	\end{split}
	\end{equation*}
where the last inequality is achieved by setting
\begin{equation*}
    \eps = \tfrac{1}{11}\omega_f(\tfrac{\sqrt{d}}{M})>0
\end{equation*}
and choosing a sufficiently small \(\delta \in (0,1)\) to make \(\mu(\Omega)\) small enough, ensuring that
\begin{equation*}
    \mu(\Omega)\cdot \Big(|u|+\|f\|_{L^\infty([0,1]^d)}\Big)^p
    \le  \Big(\tfrac{11}{10}\omega_f(\tfrac{\sqrt{d}}{M})\Big)^p - \Big(\tfrac{12}{11}\omega_f(\tfrac{\sqrt{d}}{M})\Big)^p
    >0.
\end{equation*}
We note that the condition \(\omega_f\big(\frac{\sqrt{d}}{M}\big) > 0\) can be satisfied; otherwise, \(f\) would be a constant function, which is a trivial case.
That is, we obtain
\begin{equation}
\label{eq:phi2:phi1:-f}
\begin{split}
\| \phi_2\circ \phi_1-f\|_{L^p([0,1]^d)}
\le \tfrac{11}{10}\omega_f(\tfrac{\sqrt{d}}{M}).
\end{split}
\end{equation}
Recall that  $\omega_f(n\cdot  t)\le n\cdot \omega_f(t)$ for any $n\in\N^+$ and $t\in [0,\infty)$. Thus, we have
	\begin{equation*}
	\begin{split}
	\| \phi-f\|_{L^p([0,1]^d)}
 \le  \tfrac{11}{10}\omega_f(\tfrac{\sqrt{d}}{M}) \le \tfrac{11}{10}\omega_f(\tfrac{\lceil\sqrt{d}\rceil}{M}) 
 \le   \tfrac{11}{10}\big\lceil\sqrt{d}\big\rceil\cdot\omega_f(\tfrac{1}{M}) 
 \le 2\sqrt{d}\cdot\omega_f(N^{-L}),
	\end{split}
	\end{equation*}
where the last inequality follows from \(M = N^L\) and the fact that \(\frac{11}{10} \lceil \sqrt{n} \, \rceil \le 2 \sqrt{n}\) for any \(n \in \mathbb{N}^+\). This completes the proof of Theorem~\ref{thm:main2}.

Recall that \(\varrho \in \mathcal{S}\) is a \SinTU{} activation function. 
To complete the proof of Theorem~\ref{thm:main1}, it is necessary to implement/approximate \(\phi_1\) and \(\phi_2\) using $\varrho$-activated MMNNs, rather than \ReLU{} or \(\sine\) as was done in the proof of Theorem~\ref{thm:main2}.

First, we will construct 
$\phi_{1,\eta}\in \mn{\varrho}{2d(4N-1)}{3d}{L}{\R^d}{\R}$ for any $\eta\in(0,1)$ such that
\begin{equation*}
		\phi_{1,\eta}(\bmx)\rightrightarrows \phi_{1 }(\bmx)\quad \tn{as}\  \eta\to 0^+\quad \tn{for any $\bmx\in [0,1]^d$.}
	\end{equation*}
To this end, we first construct a $\varrho$-activation MMNN to approximate the \ReLU\ function, since $\phi_{1 }\in \mn{\ReLU}{d(4N-1)}{3d}{L}{\R^d}{\R}$. 
Without loss of generality, we assume that $\phi_1$ can be expressed as a composition of functions
	\begin{equation*}
		\phi_1(\bmx) =\calA_L\circ\ReLU\circ
        \calhatA_{L-1}\circ\caltildeA_{L-1}\circ  \ \cdots \  \circ \ReLU\circ \calhatA_{ 1}\circ\caltildeA_{ 1}\circ \ReLU\circ\calA_0(\bmx)\quad \tn{for any $\bmx\in\R^d$},
	\end{equation*}
	where $\calA_0:\R^d\to \R^{d(4N-1)}$, $\calbmtildeA_\ell:\R^{d(4N-1)}\to\R^{3d}$, $\calbmhatA_\ell:\R^{3d}\to \R^{d(4N-1)}$, and $\calA_L:\R^{d(4N-1)}\to\R $ are affine linear maps for any $\ell\in \{1,2,\cdots,L-1\}$.
    By setting $\calA_\ell=\calbmhatA_\ell\circ\calbmtildeA_\ell$ for any $\ell\in \{1,2,\cdots,L-1\}$, we have
	\begin{equation*}
		\phi_1(\bmx) =\calA_L\circ\ReLU\circ\calA_{L-1}\circ  \ \cdots \  \circ \ReLU\circ\calA_1\circ\ReLU\circ\calA_0(\bmx)\quad \tn{for any $\bmx\in\R^d$},
	\end{equation*}

Since $\varrho\in\calS$, by Proposition~\ref{prop:approx:ReLU}, there exists 
$\sigma_\eta\in\nn{\varrho}{2}{1}{\R}{\R}$
for each $\eta\in(0,1)$
such that
\begin{equation*}
		\sigma_{ \eta}( x)\rightrightarrows \ReLU(x)\quad \tn{as}\  \eta\to 0^+\quad \tn{for any $ x\in [-B,B]$,}
	\end{equation*}
    where $B>0$ is a large number determined later.

	For each $\eta\in (0,1)$, we define 
		\begin{equation*}
		\phi_{1,\eta}(\bmx) =\calA_L\circ\sigma_{\eta}\circ\calA_{L-1}\circ  \ \cdots \  \circ \sigma_{\eta}\circ\calA_1\circ\sigma_{\eta}\circ\calA_0(\bmx)\quad \tn{for any $\bmx\in\R^d$}.
	\end{equation*}
    In other words,
    	\begin{equation*}
		\phi_{1,\eta}(\bmx) =\calA_L\circ\sigma_{\eta}\circ
        \calhatA_{L-1}\circ\caltildeA_{L-1}\circ  \ \cdots \  \circ \sigma_{\eta}\circ \calhatA_{ 1}\circ\caltildeA_{ 1}\circ \sigma_{\eta}\circ\calA_0(\bmx)\quad \tn{for any $\bmx\in\R^d$}.
	\end{equation*}
    Recall that 
$\sigma_\eta\in\nn{\varrho}{2}{1}{\R}{\R}$. To replace the \ReLU{} activation function with $\sigma_\eta$ in a network, we substitute each \ReLU{}   with two $\varrho$-activated neurons. Consequently,
$\phi_{1}\in \mn{\ReLU}{d(4N-1)}{3d}{L}{\R^d}{\R}$ implies
	\begin{equation*}
		\phi_{1,\eta}\in \mn{\varrho}{2d(4N-1)}{3d}{L}{\R^d}{\R}.
	\end{equation*}
	
    Next, we will prove 
	\begin{equation*}
		\phi_{1,\eta}( \bmx) 
		\rightrightarrows
		\phi_{1}( \bmx)
		\quad \tn{as}\   \eta\to 0^+
		\quad \tn{for any $\bmx\in[0,1]^d$}.
	\end{equation*}	
     For each $\eta \in (0,1)$ and $\ell = 0, 1, \cdots, L$, let $\bmh_\ell$ and $\bmh_{\ell,\eta}$ denote the functions represented by the first $\ell$ hidden layers of the MMNNs corresponding to $\phi_1$ and $\phi_{1,\eta}$, respectively, i.e.,
	\begin{equation*}
		\bmh_\ell(\bmx)
		\coloneqq \calA_{\ell }\circ\ReLU\circ\calA_{\ell-1}\circ  \ \cdots \  \circ \ReLU\circ\calA_1\circ\ReLU\circ\calA_0(\bmx)\quad \tn{for any $\bmx\in\R^d$}
	\end{equation*}
	and 
	\begin{equation*}
		\bmh_{\ell,\eta}(\bmx)
		\coloneqq \calA_{\ell}\circ\sigma_\eta\circ\calA_{\ell-1}\circ  \ \cdots \  \circ \sigma_\eta\circ\calA_1\circ\sigma_\eta\circ\calA_0(\bmx)\quad \tn{for any $\bmx\in\R^d$}.
	\end{equation*}
	For  $\ell = 0, 1, \cdots, L$, we will prove by induction that	\begin{equation}\label{eq:induction:h:ell}
		\bmh_{\ell,\eta}(\bmx)\rightrightarrows \bmh_{\ell}(\bmx) \quad \tn{as}\    \eta\to 0^+ \quad \tn{for any $\bmx\in [0,1]^d$.}
	\end{equation}

	First, we consider the base case $\ell=0$. Clearly,
	\begin{equation*}
		\bmh_{0,\eta}(\bmx)=\calA_0(\bmx)= \bmh_{0}(\bmx)\rightrightarrows \bmh_{0}(\bmx) \quad \tn{as}\  \eta\to 0^+\quad \tn{for any $\bmx\in[0,1]^d$.}
	\end{equation*}
	This means Equation~\eqref{eq:induction:h:ell} holds for $\ell=0$.
	
	Next, supposing Equation~\eqref{eq:induction:h:ell} holds for $\ell=i\in \{0,1,\cdots,L-1\}$, our goal is to prove that it also holds for $\ell=i+1$. Determine $B>0$ via
	\begin{equation*}
		B=  \max \Big\{\|\bmh_j\|_{\sup([0,1]^d)}+1:
		 j=0,1,\cdots,L \Big\},
	\end{equation*}
	where the continuity of $\ReLU$ guarantees the above supremum is finite, i.e., $B\in [1,\infty)$.
	By the induction hypothesis, we have
	\begin{equation*}
		\bmh_{i,\eta}(\bmx)\rightrightarrows \bmh_{i}(\bmx) \quad \tn{as}\    \eta\to 0^+\quad \tn{for any $\bmx\in [0,1]^d$.}
	\end{equation*}
	Clearly, for any $\bmx\in [0,1]^d$, we have $\|\bmh_{i}(\bmx)\|_{\ell^\infty}\le B$ and  \begin{equation*}
	    \|\bmh_{i,\eta}(\bmx)\|_{\ell^\infty} \le \|\bmh_{i}(\bmx)\|_{\ell^\infty}+1\le  B\quad \tn{ for small $\eta>0$.}
	\end{equation*}
	
	Recall that $\sigma_\eta(t)\rightrightarrows \ReLU(t)$ as $\eta\to 0^+ $ for any $t\in [-B,B]$. Then, we have 
	\begin{equation*}
		\sigma_\eta\circ \bmh_{i,\eta}(\bmx)
		-\ReLU\circ \bmh_{i,\eta}(\bmx)
		\rightrightarrows \bmzero\quad \tn{as}\   \eta\to 0^+ \quad \tn{for any $\bmx\in [0,1]^d$.}
	\end{equation*}
	The continuity of $\ReLU$ implies
	the uniform continuity of $\ReLU$ on $[-B,B]$, from which we deduce
	\begin{equation*}
		\ReLU\circ \bmh_{i,\eta}(\bmx)
		-
		\ReLU\circ\bmh_{i}(\bmx) 
		\rightrightarrows \bmzero
		\quad \tn{as}\    \eta\to 0 ^+\quad \tn{for any $\bmx\in [0,1]^d$.}
	\end{equation*}

	Therefore, for any $\bmx\in [0,1]^d$, as $\eta\to 0^+ $, we have
	\begin{equation*}
		\sigma_\eta\circ \bmh_{i,\eta}(\bmx)
		-
		\ReLU\circ\bmh_{i}(\bmx) 
		=
		\underbrace{
			\sigma_\eta\circ \bmh_{i,\eta}(\bmx)
			-\ReLU\circ \bmh_{i,\eta}(\bmx)
		}_{\rightrightarrows \bmzero}
		+
		\underbrace{\ReLU\circ \bmh_{i,\eta}(\bmx)
			-
			\ReLU\circ\bmh_{i}(\bmx) 
		}_{\rightrightarrows \bmzero}
		\rightrightarrows \bmzero,
	\end{equation*}
	from which we deduce
	\begin{equation*}
		\bmh_{i+1,\eta}(\bmx)=\calA_{i+1}\circ\sigma_\eta\circ \bmh_{i,\eta}(\bmx)
		\rightrightarrows
		\calA_{i+1}\circ\ReLU\circ \bmh_{i}(\bmx)=\bmh_{i+1}(\bmx).
	\end{equation*}
	This means Equation~\eqref{eq:induction:h:ell} holds for $\ell=i+1$. So we complete the inductive step.

	By the principle of induction,  we have
	\begin{equation*}
		\phi_{1,{\eta}}(\bmx) =\bmh_{L ,\eta}(\bmx)
		\rightrightarrows
		\bmh_{L }(\bmx)=
		\phi_{1}(\bmx)
		\quad \tn{as}\   \eta\to 0^+ 
		\quad \tn{for any $\bmx\in[0,1]^d$}.
	\end{equation*}

Next, we consider  replacing \(\sine\) in $\phi_2$ with \(\varrho\). 
Since $\varrho\in\calS$, there exists \( s\in\R \) such that \(\varrho(x) = \sin(x)\) for all \( x \ge s \).
Since $\phi_1$ is bounded on $[0,1]^d$, there exists a sufficiently large integer \( m \in \mathbb{Z} \) such that
\[
    2m\pi  + w\cdot \phi_1(\bmx) \ge s 
    \quad \text{and} \quad
    2m\pi  + v \cdot \sin\big(w\cdot \phi_1(\bmx)\big)  \ge s\quad \tn{for any $\bmx\in[0,1]^d$},
\]
from which we deduce
\begin{equation*}
\begin{split}
    &\phantom{=\;\;}\sin\Big(v \cdot \sin\big(w\cdot \phi_1(\bmx)\big)\Big)
    = \sin\Big(\underbrace{2m \pi + v \cdot \sin\big(w\cdot \phi_1(\bmx)\big)}_{\ge\; s}\Big)
    = \varrho\Big(2m \pi + v \cdot \sin\big(w\cdot \phi_1(\bmx)\big) \Big)
    \\ & = \varrho\Big(2m \pi + v \cdot \sin\big(\underbrace{2m \pi +  w\cdot \phi_1(\bmx)}_{\ge \; s}\big) \Big)
    = \varrho\Big(2m \pi + v \cdot \varrho\big( 2m \pi +   w\cdot \phi_1(\bmx) \big) \Big).
\end{split}
\end{equation*}
Then by defining 
\begin{equation*}
   \tildephi_2(y)
   \coloneqq 
   u\cdot \varrho\Big(2m \pi + v \cdot \varrho\big( 2m \pi +   w y \big) \Big)\quad \tn{for any $y\in\R $,}
\end{equation*}
for any $\bmx\in[0,1]^d$, we have
\begin{equation}
\label{eq:tildephi2:to:phi2}
\begin{split}
\tildephi_2\circ \phi_1(\bmx)
&= u\cdot
\varrho\Big(2m \pi + v \cdot \varrho\big( 2m \pi +   w\cdot \phi_1(\bmx) \big) \Big)
\\&=u\cdot \sin\Big(v \cdot \sin\big(w\cdot \phi_1(\bmx)\big)\Big)
    = \phi_2\circ \phi_1(\bmx).
\end{split}
\end{equation}

Recall that
\[|\phi_{1,{\eta}}(\bmx) |\le |\phi_1(\bmx)|+1\le B\quad \tn{ for small $\eta>0$ and any $\bmx\in [0,1]^d$}.\]
The continuity of $\tildephi_2$ implies
	the uniform continuity of $\tildephi_2$ on $[-B,B]$, from which we deduce
	\begin{equation*}
		\tildephi_2\circ \phi_{1,{\eta}}(\bmx) 
		\rightrightarrows		
		\tildephi_2\circ\phi_{1}(\bmx)
		\quad \tn{as}\   \eta\to 0 
		\quad \tn{for any $\bmx\in[0,1]^d$}.
	\end{equation*}
    Then we can choose sufficiently small $\eta_0>0$ such that
    \begin{equation}
    \label{eq:tildephi2:phi1:error}
        \|\tildephi_2\circ \phi_{1,{\eta_0}}-		
		\tildephi_2\circ\phi_{1}\|_{L^p([0,1]^d)}\le \tfrac{1}{10}\omega_f(\tfrac{\sqrt{d}}{M}).
    \end{equation}
Now we can define the desired $\phi$ for Theorem~\ref{thm:main1} via $\phi\coloneqq \tildephi_2\circ \phi_{1,{\eta_0}}$.

 By Equations~\eqref{eq:phi2:phi1:-f}, \eqref{eq:tildephi2:to:phi2}, and \eqref{eq:tildephi2:phi1:error},
 we have
\begin{equation*}
\begin{split}
\| \tildephi_2\circ \phi_{1,{\eta_0}} -f\|_{L^p([0,1]^d)}
&\le \| \tildephi_2\circ \phi_{1,{\eta_0}} -\tildephi_2\circ \phi_{1 }\|_{L^p([0,1]^d)}+\| \tildephi_2\circ \phi_{1 } -f\|_{L^p([0,1]^d)}
\\&
\le \tfrac{1}{10}\omega_f(\tfrac{\sqrt{d}}{M})+\| \phi_2\circ \phi_1-f\|_{L^p([0,1]^d)}
\\& \le \tfrac{1}{10}\omega_f(\tfrac{\sqrt{d}}{M}) + \tfrac{11}{10}\omega_f(\tfrac{\sqrt{d}}{M})=\tfrac{6}{5}\omega_f(\tfrac{\sqrt{d}}{M}).
\end{split}
\end{equation*}
Recall that  $\omega_f(n\cdot  t)\le n\cdot \omega_f(t)$ for any $n\in\N^+$ and $t\in [0,\infty)$. Thus, we have
	\begin{equation*}
	\begin{split}
	\|  \tildephi_2\circ \phi_{1,{\eta_0}} -f\|_{L^p([0,1]^d)}
 \le  \tfrac{6}{5}\omega_f(\tfrac{\sqrt{d}}{M}) \le \tfrac{6}{5}\omega_f(\tfrac{\lceil\sqrt{d}\rceil}{M}) 
 \le   \tfrac{6}{5}\big\lceil\sqrt{d}\big\rceil\cdot\omega_f(\tfrac{1}{M}) 
 \le 2\sqrt{d}\cdot\omega_f(N^{-L}),
	\end{split}
	\end{equation*}
where the last inequality follows from \(M = N^L\) and the fact that \(\frac{6}{5} \lceil \sqrt{n} \, \rceil \le 2 \sqrt{n}\) for any \(n \in \mathbb{N}^+\).
Recall that $\tildephi_2(y)
   \coloneqq 
   u\cdot \varrho\Big(2m \pi + v \cdot \varrho\big( 2m \pi +   w y \big) \Big)$
   for any $y\in\R$
   and 
	\begin{equation*}
		\phi_{1,\eta_0}\in \mn{\varrho}{2d(4N-1)}{3d}{L}{\R^d}{\R}.
	\end{equation*}
    We conclude that
 	\begin{equation*}
		\tildephi_2\circ \phi_{1,\eta_0}\in \mn{\varrho}{2d(4N-1)}{3d}{L+2}{\R^d}{\R},
	\end{equation*}   
which completes the proof of Theorem~\ref{thm:main1}.

\section{Proofs of Propositions in Section~\ref{sec:proof:ideas:thm:main}}
\label{sec:proof:props}


In this section, we present the detailed proofs of all propositions stated in Section~\ref{sec:proof:ideas:thm:main}. Specifically, the proofs of Propositions~\ref{prop:floor:approx}, \ref{prop:k:to:yk}, and \ref{prop:approx:ReLU} are provided in Sections~\ref{sec:proof:prop:floor:approx}, \ref{sec:proof:prop:k:to:yk}, and \ref{sec:proof:prop:approx:ReLU}, respectively.

\subsection{Proof of Proposition~\ref{prop:floor:approx}}
\label{sec:proof:prop:floor:approx}

We will prove Proposition~\ref{prop:floor:approx}, which demonstrates the efficiency of \ReLU\ FCNNs in approximating the floor function. The core idea is to use compositions of continuous piecewise functions to approximate the floor function effectively. To simplify the proof, we introduce a lemma below, which shows that continuous piecewise functions can be exactly represented by one-hidden-layer \ReLU\ FCNNs.

\begin{lemma}\label{lem:cpwl(n):in:nn}
	For any $n\in \N^+$, it holds that
	\begin{equation}\label{eq:cpl:subset:nn}
		\cpl\big(n\big)\subseteq  \nn[\big]{\ReLU}{n+1}{1}{\R}{\R}.
	\end{equation}
\end{lemma}
\begin{proof}
We proceed by mathematical induction to prove Equation~\eqref{eq:cpl:subset:nn}. We begin with the base case \( n = 1 \).
For any \( f \in \cpl\big(1\big) \), there exist \( a_1, a_2, x_0 \in \mathbb{R} \) such that
\begin{equation*}
    f(x) = 
    \begin{cases}
                a_1(x - x_0) + f(x_0)  & \text{if } x \ge x_0,\\
        a_2(x_0 - x) + f(x_0)  & \text{if } x < x_0.
    \end{cases}
\end{equation*}
Thus, we can express \( f(x) \) as \( f(x) = a_1 \sigma(x - x_0) + a_2 \sigma(x_0 - x) + f(x_0) \) for any \( x \in \mathbb{R} \), which implies
\(f \in \nn[\big]{\ReLU}{2}{1}{\mathbb{R}}{\mathbb{R}}.\)
Therefore, Equation~\eqref{eq:cpl:subset:nn} holds for \( n = 1 \).

Now, suppose Equation~\eqref{eq:cpl:subset:nn} holds for \( n = k \in \mathbb{N}^+ \). We aim to show that it also holds for \( n = k + 1 \).
For any \( f \in \cpl\big(k+1\big) \), we assume without loss of generality that \( f \) has a largest breakpoint at \( x_0 \) (the case where \( f \) has no breakpoints is trivial). 
Let \( a_1 \) and \( a_2 \) represent the slopes of the linear segments directly to the left and right of \( x_0 \), respectively. Define
\[
\tilde{f}(x) \coloneqq f(x) - (a_2 - a_1) \sigma(x - x_0) \quad \text{for any } x \in \mathbb{R}.
\]
With this construction, \( \tilde{f} \) has slope \( a_1 \) on both sides of \( x_0 \), effectively smoothing out the breakpoint at \( x_0 \) in \( f \). Thus, \( \tilde{f} \) is obtained by eliminating this breakpoint, leaving it with at most \( k \) breakpoints.
By the induction hypothesis, we know that
\[
\tilde{f} \in \cpl\big(k\big) \subseteq \nn[\big]{\ReLU}{k+1}{1}{\mathbb{R}}{\mathbb{R}}.
\]
Thus, there exist constants \( u_j,v_j,w_j,c \) for \( j = 1, 2, \dots, k+1 \) such that
\[
\tilde{f}(x) = \sum_{j=1}^{k+1} u_{j} \sigma(v_{j} x + w_{j}) + c \quad \text{for any } x \in \mathbb{R}.
\]
Therefore, for any \( x \in \mathbb{R} \), we can write
\begin{equation*}
    f(x) = (a_2 - a_1) \sigma(x - x_0) + \tilde{f}(x) = (a_2 - a_1) \sigma(x - x_0) + \sum_{j=1}^{k+1} u_{j} \sigma(v_{j} x + w_{j}) + c,
\end{equation*}
implying that \( f \in \nn[\big]{\ReLU}{k+2}{1}{\mathbb{R}}{\mathbb{R}} \).
Thus, Equation~\eqref{eq:cpl:subset:nn} holds for \( k + 1 \), completing the induction process and, hence, the proof of Lemma~\ref{lem:cpwl(n):in:nn}.
\end{proof}

With Lemma~\ref{lem:cpwl(n):in:nn} established, we are now ready to prove Proposition~\ref{prop:floor:approx}.

\begin{proof}[Proof of Proposition~\ref{prop:floor:approx}]
Since the case $N=1$ is trivial, we consider $N\ge 2$ below.
Given any $x\in \bigcup_{k=0}^{N^L-1}\big[k,\, k+1-\delta\big]$, our goal is to construct $\phi$, realized by a network with desired size, mapping $x$ to $\lfloor x\rfloor$.
Clearly  $\lfloor x\rfloor\in \{0,1,\cdots,N^L-1\}$ and hence there exists unique $(n_1,n_2,\cdots,n_L)\in \{0,1,\cdots,N-1\}^L$ such that
    \begin{equation}\label{eq:floor:x:decomposition}
        \lfloor x\rfloor= \sum_{i=1}^L n_i\cdot N^{L-i}.
    \end{equation}
In other words, the above equation forms a one-to-one map between $\{0,1,\cdots,N^L-1\}$ and $\{0,1,\cdots,N-1\}^L$.

        For $\ell=0,1,\cdots,L$, we define
    \begin{equation*}
    m_\ell = \sum_{i=\ell+1}^{L}
    n_i\cdot N^{L-i}\quad \tn{and}\quad
        z_\ell = x- \sum_{i=1}^{\ell}
    n_i\cdot N^{L-i}.
\end{equation*}
Clearly, $z_0=x$,
\begin{equation*}
\begin{split}
    z_\ell
    =
    x - \lfloor x\rfloor + \lfloor  x\rfloor 
    - 
    \sum_{i=1}^{\ell}
    n_i\cdot N^{L-i} 
   & =
   x - \lfloor x\rfloor + \sum_{i=1}^{L}
    n_i\cdot N^{L-i} 
    -     \sum_{i=1}^{\ell}
    n_i\cdot N^{L-i} 
   \\& =x - \lfloor x\rfloor +     \sum_{i=\ell+1}^{L}
    n_i\cdot N^{L-i} 
   =x - \lfloor x\rfloor + m_\ell,
\end{split}
\end{equation*}
and
\begin{equation*}
\begin{split}
        \lfloor z_\ell\rfloor
    =
    \Big\lfloor 
    \underbrace{x - \lfloor x\rfloor}_{\in\, [0,1)} + \underbrace{m_\ell}_{\in\, \bbZ}
    \Big\rfloor
   & =
    m_\ell
\end{split}
\end{equation*}
for $\ell=0,1,\cdots,L$.
We claim 
\begin{equation}\label{eq:z:ell:lower:upper:bounds}
    n_{\ell+1} \le \frac{z_\ell}{N^{L-\ell-1}}\le n_{\ell+1}+1 -\frac{\delta}{N^{L-\ell-1}}\quad \tn{for $\ell=0,1,\cdots,L-1$}.
\end{equation}
To demonstrate this, we first establish the lower bound. Clearly,
\begin{equation*}
    \frac{z_\ell}{N^{L-\ell-1}}
    =\frac{x - \lfloor x\rfloor + m_\ell}{N^{L-\ell-1}}
    \ge \frac{ m_\ell}{N^{L-\ell-1}}
    = \frac{ \sum_{i=\ell+1}^{L}
    n_i\cdot N^{L-i}}{N^{L-\ell-1}}
    \ge \frac{ 
    n_{\ell+1}\cdot N^{L-(\ell+1)}}{N^{L-\ell-1}}
    =  n_{\ell+1}.
\end{equation*}
Next, we proceed to verify the upper bound. Clearly, 
\begin{equation*}
\begin{split}
        \frac{z_\ell}{N^{L-\ell-1}}
    = \frac{x - \lfloor x\rfloor + m_\ell}{N^{L-\ell-1}}
    &= \frac{x - \lfloor x\rfloor +\delta + m_\ell}{N^{L-\ell-1}}
   - \frac{\delta}{N^{L-\ell-1}}
   \\& \le \frac{1 + m_\ell}{N^{L-\ell-1}}
   - \frac{\delta}{N^{L-\ell-1}}
    \le  n_{\ell+1}+1-\frac{\delta}{N^{L-\ell-1}},
\end{split}
\end{equation*}
where the last inequality comes from
\begin{equation*}
    \begin{split}
        \frac{1 + m_\ell}{N^{L-\ell-1}}
    =\frac{1 + \sum_{i=\ell+1}^{L}
    n_i\cdot N^{L-i}}{N^{L-\ell-1}}
    & =\frac{1 + \sum_{i=\ell+2}^{L}
    n_i\cdot N^{L-i}+ n_{\ell+1}\cdot N^{L-(\ell+1)}}{N^{L-\ell-1}}
   \\&  \le 
   \frac{1 + \sum_{i=\ell+2}^{L}
    (N-1)\cdot N^{L-i}+ n_{\ell+1}\cdot N^{L-(\ell+1)}}{N^{L-\ell-1}}
    \\&  =
   \frac{N^{L-\ell-1} + n_{\ell+1}\cdot N^{L-(\ell+1)}}{N^{L-\ell-1}}
    =  n_{\ell+1}+1.
    \end{split}
\end{equation*}
Thus, we complete the proof of Equation~\eqref{eq:z:ell:lower:upper:bounds}.

Let $h$ be a continuous piecewise linear function with $h\in \cpwl(2N-2)$ and
\begin{equation*}
    h(k)=h\Big(k+1-\frac{\delta}{N^{L-1}}\Big)=k\quad \tn{for $k=0,1,\cdots,N-1$}.
\end{equation*}

\begin{figure}[ht]
    \centering   \includegraphics[width=0.55\textwidth]{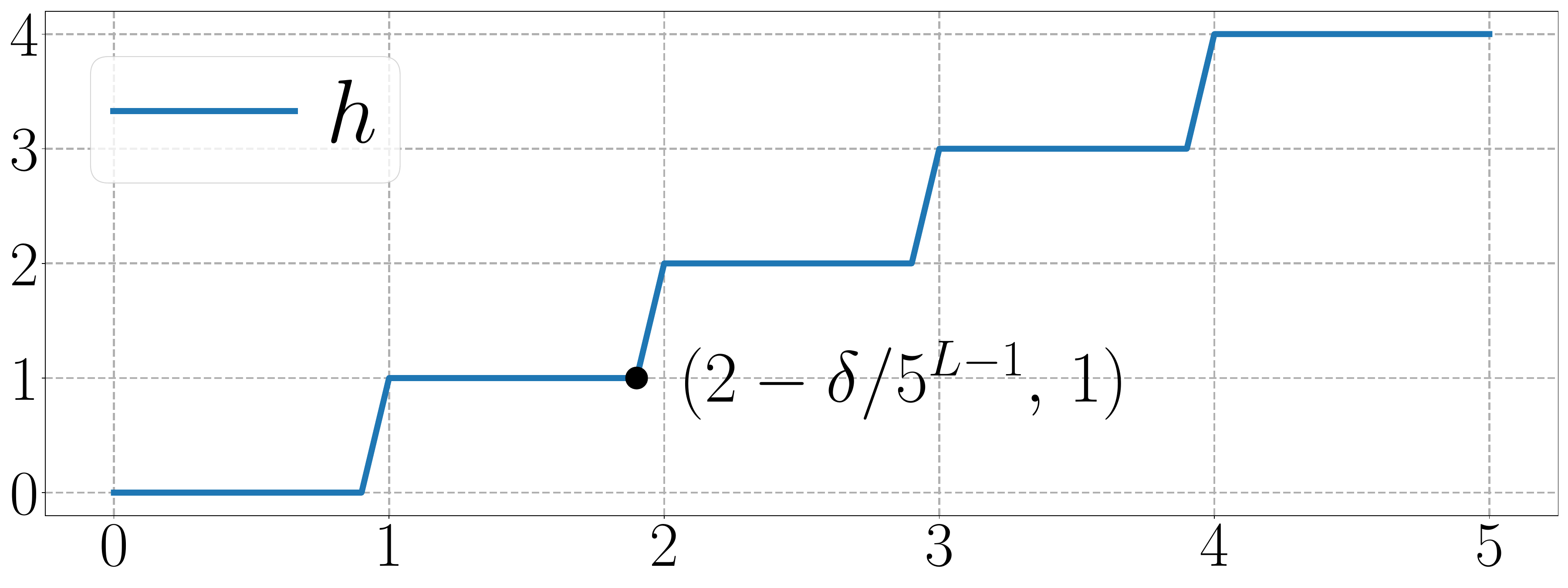}
\caption{An illustration of $h$ when $N=5$.}
    \label{fig:h}
\end{figure}

See Figure~\ref{fig:h} for an illustration of $h$ when $N=5$.
Obviously, 
\begin{equation*}
    h(t)=k\quad \tn{if $t\in \Big[k,\; k+1-\frac{\delta}{N^{L-1}}\Big]$}\quad \tn{for $k=0,1,\cdots,N-1$.}
\end{equation*}
For $\ell=0,1,\cdots,L-1$, we have $n_{\ell+1}\in \{0,1,\cdots,N-1\}$ and Equation~\eqref{eq:z:ell:lower:upper:bounds} implies 
\begin{equation*}
     \frac{z_\ell}{N^{L-\ell-1}}\in
     \Big[n_{\ell+1},\; n_{\ell+1}+1 -\frac{\delta}{N^{L-\ell-1}}\Big]
     \subseteq 
     \Big[n_{\ell+1},\; n_{\ell+1}+1 -\frac{\delta}{N^{L-1}}\Big].
\end{equation*}
It follows that 
\begin{equation*}
    h\Big(\frac{z_\ell}{N^{L-\ell-1}}\Big)= n_{\ell+1}\quad \tn{for $\ell=0,1,\cdots,L-1$,}
\end{equation*}
from which we deduce
\begin{equation*}
    \begin{split}
        z_{\ell+1} = x- \sum_{i=1}^{\ell+1}
    n_i\cdot N^{L-i}
    &= x- \sum_{i=1}^{\ell}
    n_i\cdot N^{L-i}-n_{\ell+1} \cdot N^{L-(\ell+1)}
    \\& = z_\ell - n_{\ell+1} \cdot N^{L-(\ell+1)}
    = z_\ell - h\Big(\frac{z_\ell}{N^{L-\ell-1}}\Big) \cdot N^{L-\ell-1}.   
    \end{split}
\end{equation*}
Therefore, by defining 
\begin{equation*}
    h_\ell(z) =  h\Big(\frac{z}{N^{L-\ell-1}}\Big)\quad \tn{and}\quad 
    \tildeh_\ell(z) =  z- h\Big(\frac{z}{N^{L-\ell-1}}\Big)\cdot N^{L-\ell-1}\quad \tn{for any $z\in\R$,}
\end{equation*}
we have 
\begin{equation}
\label{eq:z:ell:to:n:z:ell+1}
    h_\ell(z_\ell) = n_{\ell+1}\quad \tn{and}\quad 
    \tildeh_\ell(z_\ell) =  z_{\ell+1}\quad \tn{for $\ell=0,1,\cdots,L-1$.}
\end{equation}
Moreover, $h\in \cpwl(2N-2)$ implies
$h_\ell,\tildeh_\ell \in \cpwl(2N-2)$
for $\ell=0,1,\cdots,L-1$. Then by Lemma~\ref{lem:cpwl(n):in:nn}, 
\begin{equation}\label{eq:h:ell:tildeh:ell:in:NN}
    h_\ell,\tildeh_\ell \in \cpwl(2N-2)\subseteq \nn[\big]{\ReLU}{2N-1}{1}{\R}{\R}.
\end{equation}
This means that \( h_\ell \) and \( \tilde{h}_\ell \) can be implemented by one-hidden-layer \ReLU\ FCNNs of width \( 2N-1 \).

Consequently, the desired function \( \phi \) can be realized by a \ReLU\ MMNN of width \( 1 + (2N - 1) + (2N - 1) = 4N - 1 \), rank $3$, and depth \( L \), as illustrated in Figure~\ref{fig:floorApprox}. That is,
\[\phi \in \mn[\big]{\ReLU}{4N-1}{3}{L}{\R}{\R},\]
which completes the proof of Proposition~\ref{prop:floor:approx}.
\end{proof}
\begin{figure}[ht]
    \centering   \includegraphics[width=0.975\textwidth]{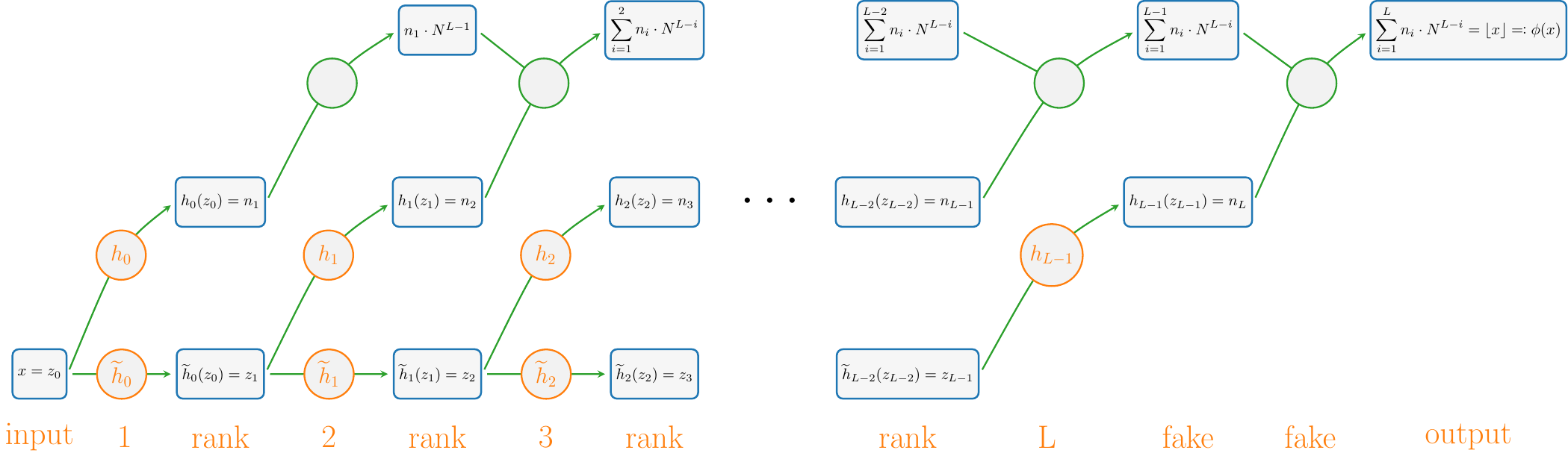}
\caption{The target network implementing $\phi$ based on Equations~\eqref{eq:floor:x:decomposition}, \eqref{eq:z:ell:to:n:z:ell+1}, and \eqref{eq:h:ell:tildeh:ell:in:NN}. In this architecture, redundant (fake) layers can be removed by merging the preceding and succeeding affine linear transformations into a single one.}
    \label{fig:floorApprox}
\end{figure}

%

\subsection{Proof of Proposition~\ref{prop:k:to:yk}}
\label{sec:proof:prop:k:to:yk}

We will establish the proof of Proposition~\ref{prop:k:to:yk}. To facilitate this, we introduce two key lemmas that serve as intermediate steps in proving Proposition~\ref{prop:k:to:yk}. 
The first lemma demonstrates how to use the \sine{}  function with a single parameter to generate rationally independent numbers. The second lemma shows the density of point sets generated by the \sine{}  function combined with rational numbers within a high-dimensional hypercube.

\begin{lemma}
    \label{lem:rationally:independent}
	Given $K\in\bbN^+$, there exists $w_0\in \big(-\tfrac{1}{K}, \tfrac{1}{K}\big)$ such that $\sin (k w_0  )$, for $k=1,2,\cdots,K$, are rationally independent.
\end{lemma}

\begin{lemma}\label{lem:dense:sin}
	Given any rationally independent numbers $a_1,a_2,\cdots,a_K$, for any $K\in\N^+$, the following set 
	\begin{equation*}
		\Big\{\Big(\sin(wa_1),\, \sin(wa_2),\, \cdots,\,\sin(wa_K)\Big) : w\in\R
		\Big\}
	\end{equation*}
is dense in $[-1,1]^K$.
\end{lemma}

We are now prepared to prove Proposition~\ref{prop:k:to:yk}, assuming the validity of Lemmas~\ref{lem:rationally:independent} and \ref{lem:dense:sin}, which will be proven in Sections~\ref{sec:proof:lem:rationally:independent} and \ref{sec:proof:lem:dense:sin}, respectively.

\begin{proof}[Proof of Proposition~\ref{prop:k:to:yk}]
Given any \(\epsilon > 0\) and \(y_k \in \mathbb{R}\) for \(k = 1, 2, \cdots, K\), we define
\begin{equation*}
    u \coloneqq \max\big\{|y_k| : k = 1, 2, \cdots, K\big\}.
\end{equation*}
Assuming \(u > 0\) (the case \(u = 0\) is trivial), we set \(z_k = y_k / u \in [-1,1]\). Then, by Lemma~\ref{lem:rationally:independent}, there exists \(w\) such that 
\(a_k = \sin(k w)\) for \(k = 1, 2, \cdots, K\) are rationally independent.
Moreover, by Lemma~\ref{lem:dense:sin}, 
the following set 
\begin{equation*}
    \Big\{\big[\sin(va_1),\, \sin(va_2),\, \cdots,\,\sin(va_K)\big]^\ts : v\in\R
    \Big\}
\end{equation*}
is dense in $[-1,1]^K$.
That is, there exists $v\in\R$ such that
\begin{equation*}
    \big|\sin(v a_k)-z_k\big| < \eps/u \quad \tn{for $k=1,2,\cdots,K$.}
\end{equation*}
Therefore, for $k=1,2,\cdots,K$, we have
\begin{equation*}
\begin{split}
    \Big|u\cdot \sin\big(v \cdot \sin(kw)\big)-y_k\Big|
    &= u\Big| \sin\big(v \cdot \sin(kw)\big)-y_k/u\Big|\\
    &= u\big| \sin\big(v \cdot a_k\big)-z_k\big|
    <u\cdot \eps/u = \eps.
\end{split}
\end{equation*}
So we finish the proof of Proposition~\ref{prop:k:to:yk}.
\end{proof}

\subsubsection{Proof of Lemma~\ref{lem:rationally:independent}}
\label{sec:proof:lem:rationally:independent}

We prove this lemma by contradiction. If it does not hold, then $\sin( k w )$, for $k=1,2,\cdots,K$, 
are rationally dependent for any $w\in \big(-\tfrac{1}{K}, \tfrac{1}{K}\big)=\calI$. That means, for all $w\in\calI$, there exists $\bmlambda=(\lambda_1,\cdots,\lambda_K)\in \bbQ^K\setminus \{\bmzero\}$ such that $\sum_{k=1}^K \lambda_k \sin(k w )=0$.
For each $\bmlambda\in \Q^K \setminus \{\bmzero\}$, we define
\begin{equation*}
    \calI_\bmlambda \coloneqq 
    \left\{w \in \calI: \sum_{k=1}^K \lambda_k \sin(k w )=0\right\}.
\end{equation*}
It follows that
\begin{equation*}
    \calI=\bigcup_{\bmlambda\in\Q^K\setminus \{\bmzero\}} \left\{w \in \calI: \sum_{k=1}^K \lambda_k \sin(k w )=0\right\}=\bigcup_{\bmlambda\in\Q^K\setminus \{\bmzero\}} \calI_\bmlambda
\end{equation*}
Recall that a countable union of countable sets remains countable. However, we observe that while
\(\mathbb{Q}^K \setminus  \{\mathbf{0}\}\)
is countable, the union
\(\calI = \bigcup_{\bmlambda \in \mathbb{Q}^K\setminus \{\bmzero\}} \calI_\bmlambda\)
is uncountable.
Then there exists $\bmlambda=(\lambda_1,\cdots,\lambda_K)\in \bbQ^K\setminus \{\bmzero\}$ such that
$\calI_\bmlambda$ is uncountable, i.e.,
\begin{equation*}
    \sum_{k=1}^K \lambda_k \sin(k w )=0\quad \tn{ for all $w\in \calI_\bmlambda$.}
\end{equation*} 
Define the function 
\[
g(w) \coloneqq \sum_{k=1}^K \lambda_k \sin\left(k w\right) \quad \text{for all } w \in \mathcal{I}.
\]
The real analyticity of \(\sine\) ensures that \(g\) is also real analytic. 
By the property that each zero of a real analytic function is isolated, a non-zero real analytic function has only countably many zeros. It follows that \(g(w) = 0\) for all \(w \in \mathcal{I}\).
Thus, we have
\begin{equation*}
    0= g^{(m)}(w)=\sum_{k=1}^K \lambda_k k^m \sin^{(m)}\left(kw\right)\quad \tn{for all $w\in \calI$ and $m\in \N$,}
\end{equation*}
from which we deduce

\[
\sum_{k=1}^K \lambda_k k^m \sin^{(m)}\left(0\right)= 0 \quad \text{for all } m\in\N.
\]
It is easy to verify that 
$|\sin^{(m)}(0)|=1$ for odd $m\in\N$. It follows that
\begin{equation}\label{eq:lambdak:km:sum0:infinite:m}
    \sum_{k=1}^K \lambda_k k^m = 0\quad \tn{for odd $m\in\N$.}
\end{equation}

We assert that Equation~\eqref{eq:lambdak:km:sum0:infinite:m} leads to \(\bmlambda = (\lambda_1, \cdots, \lambda_K) = \bmzero\), which contradicts the assumption that \(\bmlambda \in \mathbb{Q}^K \setminus \{\bmzero\}\). To complete the proof, it suffices to establish this assertion. We assume \(K \ge 2\), as the case \(K = 1\) is straightforward.
By Equation~\eqref{eq:lambdak:km:sum0:infinite:m},
for any odd $m\in \N$, we have
    \begin{equation*}
    0=\frac{1}{(K-1)^m}\sum_{k=1}^K \lambda_k k^m
    =\sum_{k=1}^K \lambda_k \Big(\frac{k}{K-1}\Big)^m
    =\sum_{k=1}^{K-1} \lambda_k \Big(\frac{k}{K-1}\Big)^m +  \lambda_K\Big(\frac{K}{K-1}\Big)^m,
\end{equation*}
implying
    \begin{equation*}
    |\lambda_K|\cdot\Big(\frac{K}{K-1}\Big)^m=
    \left| -\sum_{k=1}^{K-1} \lambda_k \Big(\frac{k}{K-1}\Big)^m\right|\le \sum_{k=1}^{K-1} |\lambda_k|.
\end{equation*}
If \(\lambda_K \neq 0\), then in the above equation, the left-hand side becomes unbounded as \(m\) grows large, while the right-hand side remains bounded.  Thus, we must have \(\lambda_K = 0\). Using a similar argument, we can show that \(\lambda_k = 0\) for \(k = K-1, K-2, \ldots, 1\).
So we finish the proof of Lemma~\ref{lem:rationally:independent}.

\subsubsection{Proof of Lemma~\ref{lem:dense:sin}}
\label{sec:proof:lem:dense:sin}

The proof of Lemma~\ref{lem:dense:sin} primarily relies on the fact that an irrational winding is dense on the torus, which is a fascinating phenomenon in transcendental number theory and Diophantine approximations. For completeness, we establish the following lemma.

\begin{lemma}
\label{lem:dense:decimal}
Given any \( K \in \mathbb{N}^+ \) and rationally independent numbers \( a_1, a_2, \dots, a_K \), the set
\begin{equation*}
    \Big\{ \Big(\tau(wa_1), \, \tau(wa_2), \, \cdots, \, \tau(wa_K)\Big) : w \in \mathbb{R} \Big\} \subseteq [0,1)^K
\end{equation*}
is dense in \( [0,1]^K \), where \( \tau(x) \coloneqq x - \lfloor x \rfloor \) for any \( x \in \mathbb{R} \).
\end{lemma}

Lemma~\ref{lem:dense:decimal} is equivalent to Lemma~22 in \cite{shijun:arbitrary:error:with:fixed:size} and Lemma~2 in \cite{yarotsky:2021:02}, where proofs can be found. Now, assuming Lemma~\ref{lem:dense:decimal} holds, let us proceed with the proof of Lemma~\ref{lem:dense:sin}.
 
\begin{proof}[Proof of Lemma~\ref{lem:dense:sin}]
	Define $g(x)\coloneqq \sin(2\pi x)$ for any $x\in \R$. Clearly, $g$ is  periodic with period $1$ and  uniformly continuous on $[-\tfrac{1}{4},\tfrac{1}{4}]$. 
	For any $\varepsilon>0$, there exists $\delta\in (0,\tfrac{1}{2})$ such that
	\begin{equation}\label{eq:g:uniformly:cont}
		|g(u)-g(v)|<\varepsilon\quad\tn{for any $u,v\in [-\tfrac{1}{4},\tfrac{1}{4}]$ with $|u-v|<\delta$.}
	\end{equation}
	
Given any $\bmxi=[\xi_1,\xi_2,\cdots,\xi_K]\in [-1,1]^K=[g(-\tfrac{1}{4}),\, g(\tfrac{1}{4})]^K$,  there exists
\[y_1,y_2,\cdots,y_K\in [-\tfrac{1}{4},\tfrac{1}{4}]\] 
such that 
\begin{equation}\label{eq:g:zk:xik}
    g(y_k)=\xi_k\quad \tn{ for any $k=1,2,\cdots,K$.}
\end{equation}
For $k=1,2,\cdots,K$, 
by setting \begin{equation*}
	\tildey_k=y_k +\tfrac{\delta}{2}\cdot\one_{\{y_k\le -\tfrac{1}{4}+\tfrac{\delta}{2}\}}-\tfrac{\delta}{2}\cdot\one_{\{y_k\ge \tfrac{1}{4}-\tfrac{\delta}{2}\}},
\end{equation*}
 we have 
\begin{equation*}
	\tildey_k=y_k +\tfrac{\delta}{2}\cdot\one_{\{y_k\le -\tfrac{1}{4}+\tfrac{\delta}{2}\}}-\tfrac{\delta}{2}\cdot\one_{\{y_k\ge \tfrac{1}{4}-\tfrac{\delta}{2}\}}\in \big[-\tfrac{1}{4}+\tfrac{\delta}{2},\,\tfrac{1}{4}-\tfrac{\delta}{2}\big]
\end{equation*}
and 
\begin{equation*}
	|\tildey_k-y_k|\le \Big|\tfrac{\delta}{2}\cdot\one_{\{y_k\le -\tfrac{1}{4}+\tfrac{\delta}{2}\}}-\tfrac{\delta}{2}\cdot\one_{\{y_k\ge \tfrac{1}{4}-\tfrac{\delta}{2}\}}\Big|\le \delta/2.
\end{equation*}

Define $\tau(x)\coloneqq  x-\lfloor x\rfloor$ for any $x\in\R$. Clearly, $[\tau(\tildey_1),\tau(\tildey_2),\cdots,\tau(\tildey_K)]^\ts \in [0,1]^K$. Then, by Lemma~\ref{lem:dense:decimal}, there exists $w_0\in\R$ such that
\begin{equation*}
	|\tau(w_0a_k)-\tau(\tildey_k)|<\delta/2\quad \tn{for $k=1,2,\cdots,K$,}
\end{equation*}
from which we deduce
\begin{equation*}
	\Big|\tau(w_0a_k)+\lfloor \tildey_k\rfloor- \tildey_k\Big|=\Big|\tau(w_0a_k)- \big(\tildey_k-\lfloor \tildey_k\rfloor\big)\Big|=\big|\tau(w_0a_k)-\tau(\tildey_k)\big|<\delta/2.
\end{equation*}
It follows from 
 $\tildey_k\in [-\tfrac{1}{4}+\tfrac{\delta}{2},\tfrac{1}{4}-\tfrac{\delta}{2}]$ that
\begin{equation*}
    \tau(w_0a_k)+\lfloor \tildey_k\rfloor\in [-\tfrac{1}{4},\tfrac{1}{4}]\quad \tn{ for $k=1,2,\cdots,K$.}
\end{equation*}
Moreover,
\begin{equation*}
	\Big|\tau(w_0a_k)+\lfloor \tildey_k\rfloor-y_k\Big|\le \Big|\tau(w_0a_k)+\lfloor \tildey_k\rfloor-\tildey_k\Big|+ \big|\tildey_k-y_k\big|<\delta/2+\delta/2=\delta
\end{equation*}
for $k=1,2,\cdots,K$. Then, by Equation~\eqref{eq:g:uniformly:cont}, we have
\begin{equation*}
	\Big| g\big(\tau(w_0a_k)+\lfloor \tildey_k\rfloor\big)-g(y_k) \Big|<\varepsilon\quad \tn{for $k=1,2,\cdots,K$.}
\end{equation*}

Recall that \( g \) is periodic with a period of \( 1 \). Thus, we have
\begin{equation*}
	g\big(\tau(w_0a_k)+\lfloor \tildey_k\rfloor\big)=g\big(w_0a_k-\lfloor w_0 a_k\rfloor+\lfloor \tildey_k\rfloor\big)=g(w_0a_k)=\sin(2\pi  w_0a_k)
\end{equation*}
for $k=1,2,\cdots,K$, implying
\begin{equation*}
	\begin{split}
		\big|\sin(2\pi    w_0a_k)- \xi_k\big|
  =\big|\sin(2\pi    w_0a_k)- g(y_k) \big|
  =\Big| g\big(\tau(w_0a_k)+\lfloor \tildey_k\rfloor\big)-g(y_k) \Big|<\varepsilon.
	\end{split}
\end{equation*}
Therefore, by setting $w_1=2\pi w_0\in\R$, we get 
\begin{equation*}
	\Big\| \Big(\sin(w_1a_1),\, \sin(w_1a_2),\, \cdots,\, \sin(w_1a_K)\Big)  -\bmxi    \Big\|_{\ell^\infty}<\varepsilon.
\end{equation*}
Since $\varepsilon>0$ and $\bmxi\in [-1,1]^K$ are arbitrary, the set
\begin{equation*}
	\Big\{\big(\sin(wa_1),\, \sin(wa_2),\, \cdots,\,\sin(wa_K)\Big) : w\in\R
	\Big\}
\end{equation*}
is dense in $[-1,1]^K$. So we finish the proof of Lemma~\ref{lem:dense:sin}.
\end{proof}


\subsection{Proof of Proposition~\ref{prop:approx:ReLU}}
\label{sec:proof:prop:approx:ReLU}

Given any \(\varepsilon\in(0,1)\), we will construct 
$\phi_\alpha\in \nn{\varrho}{2}{1}{\R}{\R}$
and show that
\[
\|\phi_\alpha-\ReLU\|_{\sup([-B,B])}< \varepsilon
\]
for all sufficiently small \(\alpha>0\).
Since $\varrho\in \calS$, there exists  $x_0\in\R$ such that
\begin{equation*}
    \varrho(x) = 
\begin{cases} 
\sin(x) & \text{if } x \ge x_0, \\ 
\sin(x_0) & \text{if } x < x_0.
\end{cases}
\end{equation*}
Clearly, we have
	\begin{equation*}
		\lim_{t\to 0^-}\frac{\varrho (x_0+t)-\varrho (x_0)}{t}=\lim_{t\to 0^-}\frac{\sin(x_0)-\sin(x_0)}{t}=0
	\end{equation*}
    and 
    \begin{equation*}
L\coloneqq \lim_{t\to 0^+}\frac{\varrho (x_0+t)-\varrho (x_0)}{t}
=\lim_{t\to 0^+}\frac{\sin (x_0+t)-\sin (x_0)}{t}=\sin^\prime(x_0)=\cos(x_0).
\end{equation*}
We split the remainder of the proof into two cases: \( L \neq 0 \) and \( L = 0 \).

\mycase{1}{$L\neq 0$.}
First, we consider the case $L\neq 0$. Since
\begin{equation*}
		\lim_{t\to 0^-}\frac{\varrho (x_0+t)-\varrho (x_0)}{t}=0
	\neq
L= \lim_{t\to 0^+}\frac{\varrho (x_0+t)-\varrho (x_0)}{t},
\end{equation*}
there exists a small $\delta_\eps\in (0,1)$ such that
\begin{equation*}
	\Big|\frac{\varrho (x_0+t  )-\varrho (x_0)}{t}-0\Big|<\frac{|L|\eps}{B}\quad \tn{for any $t\in (-\delta_\eps,0)$}
\end{equation*}
and
\begin{equation*}
	\Big|\frac{\varrho (x_0+t  )-\varrho (x_0)}{t }-L\Big|<\frac{|L|\eps}{B}\quad \tn{for any $t\in (0,\delta_\eps)$.}
\end{equation*}
That is,
\begin{equation*}
	\Big|\frac{\varrho (x_0+t  )-\varrho (x_0)}{t}- L \cdot\one_{\{t>0\}} \Big|
	<\frac{|L|\eps}{B}
    \quad \tn{for any $t\in (-\delta_\eps,0)\cup(0,\delta_\eps)$.}
\end{equation*}
Recall that \( \ReLU(x) = x \cdot \one_{\{x > 0\}} \).  
Therefore, the expression \( \frac{\varrho (x_0 + t) - \varrho (x_0)}{t} \cdot \frac{x}{L} \) should provide a good approximation of \( \ReLU{} \). Based on this, we define  
\begin{equation*}
	\phi_{\alpha}(x)\coloneqq\frac{\varrho (x_0+\alpha x)-\varrho (x_0)}{L\alpha}\quad \tn{for any $x\in \R$.}
\end{equation*}
Clearly, $\phi_\alpha(0)=0$ and 
\[\phi_\alpha\in \nn{\varrho}{1}{1}{\R}{\R}\subseteq \nn{\varrho}{2}{1}{\R}{\R}.\]
Moreover, for any $x\in [-B,B]\setminus  \{0\}$ and each $\alpha\in \big(0,\tfrac{\delta_\eps}{B}\big)$, we have $\alpha x\in (-\delta_\eps,0)\cup(0,\delta_\eps)$,	
from which we deduce
\begin{equation*}
	\begin{split}
    \big|\phi_\alpha(x)-\ReLU(x)\big|
   & =\frac{|x|}{|L|}\cdot\Big|\frac{L}{x}\cdot \phi_\alpha(x)-\frac{L}{x}\cdot\ReLU(x)\Big|
\\& =\frac{|x|}{|L|}\cdot\Big|\frac{L}{x}\cdot \frac{\varrho (x_0+\alpha x)-\varrho (x_0)}{L\alpha}-\frac{L}{x}\cdot  x\cdot \one_{\{x>0\}}\Big|
\\& =\frac{|x|}{|L|}\cdot\Big|  \frac{\varrho (x_0+\alpha x)-\varrho (x_0)}{ \alpha  x}-{L}\cdot \one_{\{x>0\}}\Big|
< 
\frac{|x|}{|L|}\cdot \frac{|L|\eps}{B}\le \eps.
	\end{split}
\end{equation*}
Therefore, we can conclude that
	\begin{equation*}
	\phi_\alpha(x)\rightrightarrows \ReLU(x)\quad \tn{as}\   \alpha\to 0^+\quad \tn{for any $x\in [-B,B]$.}
\end{equation*}
That means we finish the proof for the case of $L\neq 0$.

\mycase{2}{$L = 0$.}
Next, we consider the case where \( 0 = L = \cos(x_0) \). This implies that  
\( x_0 = \frac{(2k + 1)\pi}{2} \) for some \( k \in \mathbb{Z} \).  
It is straightforward to verify that \( \varrho \in C^1(\mathbb{R}) \setminus  C^2(\mathbb{R}) \).  
Specifically, we have  
\begin{equation*}
    \lim_{t \to 0^-} \frac{\varrho^\prime(x_0 + t) - \varrho^\prime(x_0)}{t}
    =\lim_{t \to 0^-} \frac{0-0}{t}= 0.
\end{equation*}
and
\begin{equation*}
\begin{split}
        \tildeL &\coloneqq\lim_{t \to 0^+} \frac{\varrho^\prime(x_0 + t) - \varrho^\prime(x_0)}{t}
    =\lim_{t \to 0^+} \frac{\cos(x_0 + t) - 0}{t}
    =\lim_{t \to 0^+} \frac{\cos(x_0 + t) - \cos\big(\tfrac{(2k + 1)\pi}{2}\big)}{t}
   \\& =\lim_{t \to 0^+} \frac{\cos(x_0 + t) - \cos(x_0)}{t}=-\sin(x_0)=-\sin\big(\tfrac{(2k + 1)\pi}{2}\big)\in \{-1,1\}.
\end{split}
\end{equation*}
Then there exists a small $\delta_\eps\in (0,1)$ such that
\begin{equation*}
	\Big|\frac{\varrho^\prime  (x_0+t  )-\varrho^\prime  (x_0)}{t}- \tildeL \cdot\one_{\{t>0\}} \Big|
	<\frac{|\tildeL|\eps}{ 2B }
    \quad \tn{for any $t\in (-\delta_\eps,0)\cup(0,\delta_\eps)$.}
\end{equation*}
We define  
\begin{equation*}
	\tildephi_{\alpha}(x)\coloneqq\frac{\varrho^\prime  (x_0+\alpha x)-\varrho^\prime  (x_0)}{\tildeL\alpha}\quad \tn{for any $x\in \R$.}
\end{equation*}
Clearly, $\tildephi_\alpha(0)=0$. Moreover, for any $x\in [-B,B]\setminus  \{0\}$ and each $\alpha\in \big(0,\tfrac{\delta_\eps}{B}\big)$, we have $\alpha x\in (-\delta_\eps,0)\cup(0,\delta_\eps)$,	
from which we deduce
\begin{equation*}
	\begin{split}
    \big|\tildephi_\alpha(x)-\ReLU(x)\big|
   & =\frac{|x|}{|\tildeL|}\cdot\Big|\frac{\tildeL}{x}\cdot \tildephi_\alpha(x)-\frac{\tildeL}{x}\cdot\ReLU(x)\Big|
\\& =\frac{|x|}{|\tildeL|}\cdot\Big|\frac{\tildeL}{x}\cdot \frac{\varrho^\prime  (x_0+\alpha x)-\varrho^\prime  (x_0)}{\tildeL\alpha}-\frac{\tildeL}{x}\cdot  x\cdot \one_{\{x>0\}}\Big|
\\& =\frac{|x|}{|\tildeL|}\cdot\Big|  \frac{\varrho^\prime  (x_0+\alpha x)-\varrho^\prime  (x_0)}{ \alpha  x}-{\tildeL}\cdot \one_{\{x>0\}}\Big|
<
\frac{|x|}{|\tildeL|}\cdot \frac{|\tildeL|\eps}{2B}\le \frac{\eps}{2}.
	\end{split}
\end{equation*}
That is, for each $\alpha\in (0,\frac{\delta_\eps}{B})$, we have 
\begin{equation}
\label{eq:tildephieps:minus:ReLU}
	\begin{split}
    \big|\tildephi_\alpha(x)-\ReLU(x)\big|
  < \frac{\eps}{2}\quad \tn{for any $x\in [-B, B]$.}
	\end{split}
\end{equation}

For each $\eta\in (0,1)$, we define
\begin{equation*}
	h_\eta(z)\coloneqq\frac{\varrho (z+\eta)-\varrho (z) }{\eta}\quad \tn{for any $z\in\R$.}
\end{equation*}
Recall that  $\varrho\in C^1(\R)\setminus  C^{2}(\R)$.
By Lagrange's mean value theorem, for any $z\in\R$, there exists
$\xi\in (z,z+\eta)$  such that 
\begin{equation*}
	h_\eta(z)= \frac{\varrho (z+\eta)-\varrho (z) }{\eta}=\varrho^\prime(\xi),
\end{equation*}
from which we deduce
\begin{equation*}
	h_\eta(z)= \frac{\varrho (z+\eta)-\varrho (z) }{\eta}=\varrho^\prime(\xi)
	\rightrightarrows \varrho^{\prime}(z)\quad\tn{as}\  \eta\to0\quad \tn{for any $z\in [x_0-1,x_0+1]$.}
\end{equation*}
Then there exists $\eta_\alpha>0$ such that
\begin{equation*}
	\big|	h_{\eta_\alpha}(z) - \varrho^{\prime}(z)\big|<\frac{|\tildeL|\alpha^2}{2}\quad \tn{for any $z\in [x_0-1,x_0+1]$.}
\end{equation*}				
Next, we can define the desired $\phi_\alpha$ via
\begin{equation*}
	\begin{split}
		\phi_{\alpha}(x)\coloneqq
		 \frac{h_{\eta_\alpha}(x_0+\alpha x)-\varrho^{\prime}(x_0)}{\tildeL\alpha}
         &=\frac{ \frac{\varrho(x_0+\alpha x+ \eta_\alpha) -\varrho(x_0+\alpha x)}{\eta_\alpha} -\varrho^{\prime}(x_0)}{\tildeL\alpha}
		\\& =\frac{  \varrho(x_0+\alpha x+ \eta_\alpha) -\varrho(x_0+\alpha x) -{\eta_\alpha}\varrho^{\prime}(x_0)}{{\eta_\alpha}\tildeL\alpha}
	\end{split}
\end{equation*}	
for any $x\in \R$. Clearly, $\phi_\alpha\in \nn{\varrho}{2}{1}{\R}{\R}$. Moreover, for each $\alpha\in \big(0,\min\{\tfrac{\delta_\eps}{B},\eps\}\big)\subseteq \big(0,\tfrac{1}{B}\big)$ and any $x\in [-B,B]$, we have $x_0+\alpha  x \in [x_0-1,x_0+1]$, implying
\begin{equation*}
	\begin{split}
		\big|\phi_{\alpha}(x)-\tildephi_{\alpha}(x)\big|
        &=  \bigg|\frac{h_{\eta_\alpha}(x_0+\alpha x)-\varrho^{\prime}(x_0)}{\tildeL\alpha}
         -\frac{\varrho^{\prime}(x_0+\alpha x)-\varrho^{\prime}(x_0)}{\tildeL\alpha} \bigg|\\	
		&\le  \frac{1}{\alpha}\cdot\Big|	h_{\eta_\alpha}(x_0+ \alpha x)-\varrho^{\prime}(x_0+\alpha x)\Big|
		< \frac{1}{|\tildeL|\alpha}\cdot\frac{|\tildeL|\alpha^2}{2}=\frac{\alpha}{2}\le \frac{\eps}{2}.
	\end{split}
\end{equation*}
Combining this with Equation~\eqref{eq:tildephieps:minus:ReLU}, we can conclude that
\begin{equation*}
	\big|\phi_\alpha(x)-\ReLU(x)\big|\le \big|\phi_\alpha(x)-\tildephi_\alpha(x)\big|+\big|\tildephi_\alpha(x)-\ReLU(x)\big| < \frac{\eps}{2}+ \frac{\eps}{2}=\eps,
\end{equation*}
for each $\alpha\in \big(0,\min\{\tfrac{\delta_\eps}{B},\eps\}\big)$ and any $x\in [-B,B]$. That means
	\begin{equation*}
	\phi_\alpha(x)\rightrightarrows \ReLU(x)\quad \tn{as}\   \alpha\to 0^+\quad \tn{for any $x\in [-B,B]$.}
\end{equation*}
Thus, we have completed the proof for the case \( L = 0 \), thereby concluding the proof of Proposition~\ref{prop:approx:ReLU}.

\end{document}